\documentclass{article}

\usepackage[preprint]{main}

\usepackage[utf8]{inputenc} % allow utf-8 input
\usepackage[T1]{fontenc}    % use 8-bit T1 fonts
\usepackage{hyperref}       % hyperlinks
\usepackage{url}            % simple URL typesetting
\usepackage{booktabs}       % professional-quality tables
\usepackage{amsfonts}       % blackboard math symbols
\usepackage{nicefrac}       % compact symbols for 1/2, etc.
\usepackage{microtype}      % microtypography
\usepackage{xcolor}         % colors

\usepackage{makecell}
\usepackage{tabularx}
\usepackage{array}

\usepackage{graphicx}
\usepackage{subcaption}
\usepackage{wrapfig}

\usepackage{amsmath} 
\usepackage{algorithm}
\usepackage{algorithmic}

\usepackage{amsmath}
\usepackage{amssymb}
\usepackage{mathtools}
\usepackage{amsthm}

\theoremstyle{plain}
\newtheorem{theorem}{Theorem}[section]

\newtheorem{lemma}[theorem]{Lemma}
\newtheorem{corollary}[theorem]{Corollary}
\theoremstyle{definition}
\newtheorem{definition}[theorem]{Definition}
\newtheorem{assumption}[theorem]{Assumption}
\theoremstyle{remark}

\usepackage{multirow}

\newcommand{\E}{\mathbb{E}}
\newcommand{\Argmax}{\mathop{\mathrm{Argmax}}}

\DeclareMathOperator{\tr}{tr}
\DeclareMathOperator{\Dom}{Dom}
\DeclareMathOperator{\Ran}{Ran}
\DeclareMathOperator{\supp}{supp}

\title{Operator-Guided Invariance Learning for Continuous Reinforcement Learning}
\author{%
    Zuyuan Zhang\\
    The George Washington University\\
    \texttt{zuyuan.zhang@gwu.edu}\\
    \And
    Fei Xu Yu\\
    The George Washington University\\
    \texttt{fxyu@gwu.edu}\\
    \And
    Tian Lan\\
    The George Washington University\\
    \texttt{tlan@gwu.edu}
}

\begin{document}

\maketitle

\begin{abstract}
Reinforcement learning (RL) with continuous time and state/action spaces is often data-intensive and brittle under nuisance variability and shift, motivating methods that exploit value-preserving structures to stabilize and improve learning. Most existing approaches focus on special cases, such as prescribed symmetries and exact equivariance, without addressing how to discover more general structures that require nonlinear operators to transform and map between continuous state/action systems with isomorphic value functions. We propose \textbf{VPSD-RL} (Value-Preserving Structure Discovery for Reinforcement Learning). It models continuous RL as a controlled diffusion with value-preserving mappings defined through Lie-group actions and associated pullback operators. We show that a value-preserving structure exists exactly when pulling back the value function and pushing forward actions commute with the controlled generator and reward functional. Further, approximate value-preserving structures with rigorous guarantees can be found when the Hamilton--Jacobi--Bellman mismatch is small. This framework discovers exact and approximate value-preserving structures by searching for the associated Lie group operators. VPSD-RL fits differentiable drift, diffusion, and reward models; learns infinitesimal generators via determining-equation residual minimization; exponentiates them with ODE flows to obtain finite transformations; and integrates them into continuous RL through transition augmentation and transformation-consistency regularization. We show that bounded generator/reward mismatch implies quantitative stability of the optimal value function along approximate orbits, with sensitivity governed by the effective horizon, and observe improved data efficiency and robustness on continuous-control benchmarks.
\end{abstract}

\section{Introduction}

Reinforcement learning (RL) has achieved great success across a variety of sequential decision-making tasks, yet modern RL are known to be sample-inefficient in complex environments and brittle under nuisance variability and shifts
\citep{puterman2014markov,sutton1998reinforcement,bertsekas2012dynamic,kaelbling1996reinforcement,mnih2015human,schulman2015trust,schulman2017proximal,haarnoja2018soft,xiang2019continuous,fujimoto2018addressing,kober2013reinforcement,levine2020offline,zhang2026cochain,zhang2026geometry,zhang2026structuring,zhang2025learning}. 
A key challenge is not optimization or exploration alone but also \emph{limited exploitation of underlying structures}: if transformations exist to preserve (or nearly preserve) the optimal value function, value-learning can generalize across otherwise distinct transitions/trajectories, significantly improving consistency, reducing complexity, and stabilizing learning
\citep{dean1997model,ravindran2001symmetries,ravindran2004approximate,li2006towards}. This viewpoint subsumes prior practice such as known forms of invariance-inducing structures (like symmetry, rotation, and linear scaling) and suggests a more ambitious goal: to \emph{discover} general value-preserving structures and associated transformations--both exactly and approximately--from data and exploit them in reinforcement learning, while ensuring rigorous guarantees on the policy performance
\citep{tobin2017domain,laskin2020reinforcement,yarats2021mastering,cohen2016group,bronstein2021geometric}.

This is particularly challenging in time- and state-continuous reinforcement learning, which is often modeled as controlled diffusion, a mathematically principled model widely used 
in robotics, finance, and physical processes \citep{oksendal2013stochastic,karatzas2014brownian,fleming2006controlled,yong1999stochastic,borkar2008stochastic,kushner1990numerical}. In this model, optimality is governed by an \emph{operator}---the controlled generator (or the induced semigroup)---and by the discounted Hamilton--Jacobi--Bellman (HJB) equation in viscosity form \citep{crandall1992user,bardi1997optimal,barles1994solutions}. We define value-preserving transformations/mappings by Lie-group actions (due to the continuous spaces under consideration) and associated pullback operators. Consequently, we show that whether a transformation is ``value-preserving'' must be addressed through an operator-level question: Whether pulling back the value function (and pushing forward actions) \emph{commutes} with the controlled generator and the reward functional, exactly or approximately.

This operator-centric view provides a novel and tractable method to discover general value-preserving structures \citep{fleming2006controlled,bardi1997optimal} in time- and space continuous RL. Existing approaches only considered certain known forms and special cases, such as exact symmetry by equivariant RL \citep{cohen2016group,kondor2018generalization,weiler2019general,maron2018invariant,hutchinson2021lietransformer,bronstein2021geometric} and invariant features by representation learning \citep{ha2018world,hafner2019dream,laskin2020curl,stooke2021decoupling}. Data augmentation methods \citep{tobin2017domain,laskin2020reinforcement,yarats2021mastering} can improve learning robustness if the relations are already known. Classical isomorphism and homomorphism formalize equivalence relations in MDPs, yet 
focusing only on discrete-time and discrete-state problems \citep{dean1997model,ravindran2001symmetries,ravindran2004approximate,li2006towards}.

Our proposed approach, \textbf{VPSD-RL} (Value-Preserving Structure Discovery for Reinforcement Learning), defines candidate transformations through an exact or approximate commutation property with respect to the controlled generator and reward function, using
the discounted HJB operator \citep{fleming2006controlled,crandall1992user,bardi1997optimal}. Algorithmically, VPSD-RL: %(1) works for differentiable drift/diffusion/reward models; 
(1) learns \emph{infinitesimal} generators by minimizing determining-equation residuals; (2) exponentiates the learned vector fields via ODE flows to obtain finite transformations; and (3) injects these transformations into RL through transition augmentation and a transformation-consistency regularizer for value/policy networks \citep{olver1993applications,bluman1989symmetries,ibragimov2024crc,lee2003smooth}. A notable byproduct is that, when the discovered structure integrates to a (local) group action, it recovers classical Lie symmetries as a \emph{special case} rather than a design premise \citep{olver1993applications,bluman1989symmetries}.

Theoretically, we prove that exact value-preserving structure implies exact value invariance of the optimal value, while approximate generator/reward mismatch yields the stability bound
$
\|V^\star\circ g_\vartheta - V^\star\|_\infty
\le
\frac{1}{\beta}\Big(\varepsilon_r+\varepsilon_{\mathcal L}\|V^\star\|_{C^2}\Big).
$
We further establish end-to-end guarantees for the VPSD-RL pipeline: stochastic structure discovery converges to approximate stationary points at standard SGD rates, learned infinitesimal generators are statistically consistent under realizability, numerical flow integration has order-$h^p$ error, and exact/approximate transition augmentation respectively preserves the Bellman fixed point or perturbs it by a controlled amount. VPSD-RL demonstrates improved data efficiency and robustness on representative continuous-control benchmarks \citep{fujimoto2018addressing,haarnoja2018soft}.

\section{Preliminaries}
\label{sec:prelim}

\noindent {\bf Discounted MDPs on Continuous State--Action Spaces.}
\label{sec:prelim_cmdp}
Let $\mathcal{S}$ and $\mathcal{A}$ be Borel subsets of $\mathbb{R}^d$ and $\mathbb{R}^m$ (or, more generally,
Polish spaces equipped with their Borel $\sigma$-algebras).
A discounted Markov decision process (MDP) is
$\mathcal{M}=(\mathcal{S},\mathcal{A},P,r,\gamma)$ with $\gamma\in(0,1)$,
where $P(\cdot\mid s,a)$ is a Markov transition \emph{kernel} on $\mathcal{S}$ and
$r:\mathcal{S}\times\mathcal{A}\to\mathbb{R}$ is a measurable reward bounded by $|r(s,a)|\le R_{\max}$.
A (stationary Markov) policy $\pi(\cdot\mid s)$ is a stochastic kernel on $\mathcal{A}$ given $s\in\mathcal{S}$.
The value function of $\pi$ is $V^\pi(s)
=\mathbb{E}^\pi\!\left[\sum_{t\ge 0}\gamma^t r(s_t,a_t)\,\Big|\,s_0=s\right],
\qquad
a_t\sim\pi(\cdot\mid s_t),\ \ s_{t+1}\sim P(\cdot\mid s_t,a_t),$
and the optimal value is $V^\star(s)=\sup_{\pi}V^\pi(s)$. In finite-state MDPs, exact value-preserving correspondences are often represented by permutations or homomorphisms.
Because $\mathcal{S},\mathcal{A}$ are continuous,
%and we seek an intrinsic operator-level treatment,
we work with smooth local transformation families on $\mathcal{S}$ and $\mathcal{A}$.
Lie-group actions provide one important special case and one convenient parameterization of transformation families used in our analysis, while stronger invertibility or closure properties can recover isomorphism-type correspondences within the broader value-preserving structure considered here.

\noindent {\bf Controlled Diffusions, Semigroups, and Generators.}
\label{sec:prelim_diffusion}
Let $\mathcal{S}\subseteq\mathbb{R}^d$ and $\mathcal{A}\subseteq\mathbb{R}^m$ be open sets
(or smooth manifolds with local coordinates).
Consider infinite-horizon controlled diffusion in \emph{It\^{o}} form
\begin{equation}
\label{eq:sde_ito}
ds_t=b(s_t,a_t)\,dt+\Sigma(s_t,a_t)\,dW_t,
\end{equation}
where $W_t\in\mathbb{R}^q$ is standard Brownian motion,
$b:\mathcal{S}\times\mathcal{A}\to\mathbb{R}^d$ is the drift,
and $\Sigma:\mathcal{S}\times\mathcal{A}\to\mathbb{R}^{d\times q}$ is the diffusion coefficient.
We use the same bounded reward $r:\mathcal{S}\times\mathcal{A}\to\mathbb{R}$ and fix a discount rate $\beta>0$.
For a (possibly stochastic) Markov control $a_t\sim\pi(\cdot\mid s_t)$, define the discounted return
\begin{equation}
\label{eq:cts_value}
V^\pi(s)=\mathbb{E}^\pi\!\left[\int_0^\infty e^{-\beta t} r(s_t,a_t)\,dt\,\Big|\,s_0=s\right],
\qquad
V^\star(s)=\sup_{\pi}V^\pi(s).
\end{equation}

For each fixed action $a\in\mathcal{A}$, the associated infinitesimal generator $\mathcal{L}^a$
acts on $f\in C^2(\mathcal{S})$ by
\begin{equation}
\label{eq:generator}
(\mathcal{L}^{a} f)(s)
=b(s,a)\cdot\nabla f(s)
+\frac12\operatorname{tr}\!\Big(Q(s,a)\,\nabla^2 f(s)\Big),
\qquad
Q(s,a):=\Sigma(s,a)\Sigma(s,a)^\top.
\end{equation}
Let $(P_t^a)_{t\ge 0}$ be the Markov semigroup under constant control $a$:
\begin{equation}
\label{eq:semigroup}
(P_t^a f)(s)=\mathbb{E}\!\left[f(s_t)\ \big|\ s_0=s,\ a_u\equiv a\text{ for }u\in[0,t]\right].
\end{equation}
Intuitively, $P_t^a$ propagates functions forward in time under fixed dynamics,
while $\mathcal{L}^a$ captures the \emph{first-order infinitesimal} change:
formally, $\frac{d}{dt}P_t^a f=\mathcal{L}^a P_t^a f$ whenever the derivative is justified.

\noindent {\bf Lie-Group Actions and Pullback Operators}
\label{sec:prelim_group}
Let $G$ be a Lie group acting smoothly on the state and action spaces through maps
$g_\theta:\mathcal{S}\to\mathcal{S}$ and $h_\theta:\mathcal{A}\to\mathcal{A}$ for $\theta\in G$,
where each $g_\theta$ is a diffeomorphism and each $h_\theta$ is an invertible smooth map.
We define the pullback operator $U_{g_\theta}$ on functions $f:\mathcal{S}\to\mathbb{R}$ by
\begin{equation}
\label{eq:pullback}
(U_{g_\theta}f)(s)=f(g_\theta(s)).
\end{equation}

We emphasize that we adopt the convention $U_g f = f\circ g$ (rather than $f\circ g^{-1}$).
With this choice, $(U_g f)(s)$ means: \emph{first move the state by the symmetry $g$, then evaluate $f$}.
This is the operator that naturally appears when expressing equivariance of generators/semigroups
in pointwise form, e.g., comparing $(\mathcal{L}^{h(a)}(U_g f))(s)$ with $(U_g(\mathcal{L}^a f))(s)$.

\noindent {\bf Viscosity HJB Basics and Well-Posedness}
\label{sec:prelim_hjb}
Define the Hamiltonian operator on bounded continuous functions $V:\mathcal{S}\to\mathbb{R}$ by
\begin{equation}
\label{eq:hamiltonian}
(\mathcal{H}V)(s)=\sup_{a\in\mathcal{A}}\Big\{r(s,a)+(\mathcal{L}^a V)(s)\Big\},
\end{equation}
and consider the discounted Hamilton--Jacobi--Bellman (HJB) equation
\begin{equation}
\label{eq:hjb}
\beta V(s)=(\mathcal{H}V)(s).
\end{equation}

Under the dynamic programming principle, the optimal value $V^\star$ in~\eqref{eq:cts_value} is characterized
as the unique bounded viscosity solution of~\eqref{eq:hjb}.
We use viscosity solutions because $V^\star$ need not be $C^2$ even when the coefficients are smooth;
viscosity theory provides a robust notion of solution and a comparison principle that yields uniqueness.

\begin{theorem}[Existence and uniqueness of the optimal value]
\label{thm:hjb_unique}
Under the HJB well-posedness conditions stated in Assumption~\ref{ass:hjb}, the HJB equation~\eqref{eq:hjb} admits a unique bounded viscosity solution.
Moreover, this solution coincides with the optimal value function $V^\star$ defined in~\eqref{eq:cts_value}.
\end{theorem}

This well-posedness is the formal anchor for our value-preservation analysis:
once the transformed controlled generator and reward leave the HJB operator unchanged, uniqueness forces the optimal value function to inherit the corresponding value-preserving relation.
Subsequent sections use this uniqueness-to-value-preservation logic repeatedly.

\section{Value-Preserving Structure and Reduction for Controlled Diffusion}
\label{sec:lie_sym}

We characterize \emph{exact} value-preserving structure of controlled diffusions at the level of the controlled generator (equivalently, the Markov semigroup), which directly links to the discounted HJB equation and avoids discretization dependence.
Differentiating operator equivariance along one-parameter transformation families yields the \emph{determining equations}, i.e., PDE constraints on drift, diffusion, and reward.
This structure induces an intrinsic reduction via invariants, leading to value invariance along transformation orbits and structure-consistent optimal policies under mild conditions.

\subsection{Exact Value-Preserving Structure}
\label{sec:exact_equiv}

\begin{definition}[Exact value-preserving structure]
\label{def:exact_sym}
Let \(G\) be a Lie group acting smoothly on states and actions via
\((g_\vartheta,h_\vartheta)\), \(\vartheta\in G\). We say
\((g_\vartheta,h_\vartheta)\) is an exact value-preserving structure of
the controlled diffusion if, for all \(\vartheta\in G\), all \(a\in A\),
and all \(f\in C^2(S)\),
\begin{equation}
\label{eq:gen_equiv}
L_a(f\circ g_\vartheta)
=
(L_{h_\vartheta(a)}f)\circ g_\vartheta,
\qquad\text{equivalently}\qquad
L_a U_{g_\vartheta}=U_{g_\vartheta}L_{h_\vartheta(a)},
\end{equation}
and
\begin{equation}
\label{eq:reward_inv}
r(g_\vartheta(s),h_\vartheta(a))=r(s,a),
\qquad \forall (s,a)\in S\times A .
\end{equation}
\end{definition}

Equation~\eqref{eq:gen_equiv} is the generator-level value-preservation condition:
evolving after transforming $(s,a)$ is equivalent, as seen through test functions, to transforming the outcome of evolution.
In classical group-symmetry settings, this operator-commutation condition coincides with generator equivariance; here we use it as a criterion for value preservation.
Crucially, since $\mathcal{L}^a$ depends on the diffusion only through $Q=\Sigma\Sigma^\top$,
the definition is intrinsic to the induced Markov process and does not depend on a particular factorization $\Sigma$.

\begin{lemma}[Semigroup form of exact value preservation]
\label{lem:semigroup_equiv}
Assume for each \(a\in A\) that \((P^a_t)_{t\ge 0}\) is a strongly continuous Markov semigroup on a Banach space \(\mathcal F\), and \(L^a\) is its generator with domain \(\mathrm{Dom}(L^a)\supseteq C^2(S)\cap\mathcal F\). Then Eq.~\ref{eq:gen_equiv} implies, for all \(t\ge 0\),
\begin{equation}
\label{eq:semigroup_equiv}
P^a_t U_{g_\vartheta}=U_{g_\vartheta}P^{h_\vartheta(a)}_t .
% P_t^{h_\vartheta(a)}U_{g_\vartheta}=U_{g_\vartheta}P_t^a.
\end{equation}
Conversely, if this semigroup identity holds for all \(t\ge 0\) and all \(f\in\mathrm{Dom}(L^{h_\vartheta(a)})\), then Eq.~\ref{eq:gen_equiv} holds on the corresponding generator domain.
\end{lemma}

The semigroup identity~\eqref{eq:semigroup_equiv} compares the full time-$t$ evolution under constant control.
The generator form~\eqref{eq:gen_equiv} is the infinitesimal version that is algebraically convenient
for deriving determining equations and for interfacing with the HJB operator.

\subsection{Determining Equations: From Global Value Preservation to Local PDE Constraints}
\label{sec:determining}

To obtain explicit value-preservation constraints on $(b,Q,r)$, we differentiate the operator identity
in Definition~\ref{def:exact_sym} along a one-parameter subgroup.
This yields linear PDE conditions, the determining equations, on the infinitesimal generators $(X,Y)$.

Consider a one-parameter subgroup $\{\vartheta_\epsilon\}_{\epsilon\in\mathbb{R}}\subseteq G$
with $\vartheta_0=e$ (the identity), and define the infinitesimal generators
\begin{equation}
\label{eq:XY_def}
X(s)=\left.\frac{d}{d\epsilon}g_{\vartheta_\epsilon}(s)\right|_{\epsilon=0}\in T_s\mathcal{S}\cong\mathbb{R}^d,
\qquad
Y(a)=\left.\frac{d}{d\epsilon}h_{\vartheta_\epsilon}(a)\right|_{\epsilon=0}\in T_a\mathcal{A}\cong\mathbb{R}^m.
\end{equation}

\begin{theorem}[Determining equations (necessity, It\^o form under $U_g f=f\circ g$)]
\label{thm:determining_necessity}
Assume $b,Q,r$ are differentiable and $(g_\vartheta,h_\vartheta)$ is an exact value-preserving structure
in the sense of Definition~\ref{def:exact_sym}.
Then for all $(s,a)\in\mathcal{S}\times\mathcal{A}$,
\begin{align}
\label{eq:det_reward}
&X(s)\cdot\nabla_s r(s,a)+Y(a)\cdot\nabla_a r(s,a)=0,\\
\label{eq:det_drift_ito}
&\nabla_s b(s,a)\,X(s)-(\nabla_s X(s))\,b(s,a)\;+\;\nabla_a b(s,a)\,Y(a)\;-\;\frac12\,\Delta_{Q(s,a)}X(s)=0,\\
\label{eq:det_diffQ_ito}
&\nabla_s Q(s,a)[X(s)]\;-\;(\nabla_s X(s))\,Q(s,a)\;-\;Q(s,a)\,(\nabla_s X(s))^\top
\;+\;\nabla_a Q(s,a)[Y(a)]=0,
\end{align}
where $\nabla_s Q(s,a)[X]$ denotes the directional derivative of the matrix field $Q(\cdot,a)$ at $s$
along direction $X(s)$, similarly for $\nabla_a Q(s,a)[Y]$, and
$
\Delta_{Q}X:=\bigl(\mathrm{tr}(Q\,\nabla_s^2 X_1),\ldots,\mathrm{tr}(Q\,\nabla_s^2 X_d)\bigr)^\top,
\quad
\mathrm{tr}(Q\,\nabla_s^2 X_k)=\sum_{i,j=1}^d Q_{ij}\,\partial_{ij}X_k.
$
In particular, if $g_{\vartheta_\epsilon}$ is affine in $s$ (equivalently $\nabla_s^2 X\equiv 0$), then the It\^o correction
$\Delta_Q X$ vanishes and~\eqref{eq:det_drift_ito} reduces to the first-order Lie-bracket form.
\end{theorem}

Our value-preserving definition is posed at the level of the generator~\eqref{eq:gen_equiv},
which depends on the diffusion only through $Q=\Sigma\Sigma^\top$.
Therefore, the intrinsic necessary condition is~\eqref{eq:det_diffQ_ito}.
A condition written directly for $\Sigma$ would be \emph{strictly stronger} and generally not necessary,
because different $\Sigma$ factorizations can induce the same $Q$ and hence the same generator.% Our symmetry 

\begin{theorem}[Local sufficiency]
\label{thm:local_sufficiency}
Assume $b,Q,r$ are smooth.
Let $X$ and $Y$ be smooth vector fields such that
\eqref{eq:det_reward}, \eqref{eq:det_drift_ito}, and \eqref{eq:det_diffQ_ito} hold on an open set
$\Omega\subseteq\mathcal{S}\times\mathcal{A}$.
Assume the flows $g_\epsilon$ and $h_\epsilon$ generated by the ODEs
$\frac{d}{d\epsilon}g_\epsilon(s)=X(g_\epsilon(s))$ and
$\frac{d}{d\epsilon}h_\epsilon(a)=Y(h_\epsilon(a))$
exist and are diffeomorphisms for $|\epsilon|<\epsilon_0$.
Then $(g_\epsilon,h_\epsilon)$ is an exact value-preserving structure on $\Omega$ for all $|\epsilon|<\epsilon_0$.
\end{theorem}

Theorem~\ref{thm:determining_necessity} converts a \emph{global} value-preserving operator identity
into \emph{local} PDE constraints on the coefficients.
Theorem~\ref{thm:local_sufficiency} states the converse: if the PDE constraints hold and the induced flows exist,
then integrating the infinitesimal generators recovers local value preservation.
Together, they justify treating value-preserving structure discovery as solving the determining equations.

\subsection{Intrinsic Reduction by Invariants and Constraints}
\label{sec:reduction}

\paragraph{From value-preserving directions to reduced coordinates.}
Once value-preserving directions are characterized by vector fields $X_1,\dots,X_K$,
we can reduce the effective state dimension by passing to \emph{invariant coordinates} that are constant
along transformation orbits. This gives a coordinate-free description of the orbit space without explicitly
constructing the quotient $\mathcal{S}/G$.

Let $X_1,\dots,X_K$ be linearly independent infinitesimal state generators and define the distribution
\begin{equation}
\label{eq:distribution}
\mathcal{D}(s)=\mathrm{span}\{X_1(s),\dots,X_K(s)\}\subseteq T_s\mathcal{S}.
\end{equation}
A scalar function $I:\mathcal{S}\to\mathbb{R}$ is an invariant if it is constant along $\mathcal{D}$:
\begin{equation}
\label{eq:invariant_pde}
X_k(s)\cdot\nabla I(s)=0,\qquad k=1,\dots,K.
\end{equation}

\begin{theorem}[Frobenius reduction and invariant coordinates]
\label{thm:frobenius}
Assume $\mathcal{D}$ is a smooth distribution of \emph{constant rank} $K$ on a neighborhood,
and it is involutive: $[X_i,X_j](s)\in\mathcal{D}(s)$ for all $i,j$ and all $s$ in that neighborhood.
Then locally there exist $(d-K)$ functionally independent invariants $I_1,\dots,I_{d-K}$
whose joint map $I(s)=(I_1(s),\dots,I_{d-K}(s))$ is constant on the integral leaves of $\mathcal{D}$.
In particular, the local orbit space can be represented intrinsically by the invariant coordinates $z=I(s)$,
without explicitly constructing a quotient $\mathcal{S}/G$.
\end{theorem}

% \paragraph{Post-constraint reduction.}
Let the feasible set be $\mathcal{S}_{\mathrm{feas}}=\{s:\ c(s)=0,\ g(s)\le 0\}$ for smooth constraints $(c,g)$.
A sufficient local condition for the value-preserving directions to preserve feasibility is tangency at feasible points:
\begin{equation}
\label{eq:tangent_constraints}
\nabla c(s)\,X_k(s)=0,\qquad
\nabla g_\ell(s)\,X_k(s)\le 0\ \text{on active constraints }(g_\ell(s)=0),\qquad \forall k,\ell.
\end{equation}
Under~\eqref{eq:tangent_constraints}, the flow of each $X_k$ stays in $\mathcal{S}_{\mathrm{feas}}$ for small time,
so one can solve~\eqref{eq:invariant_pde} restricted to $\mathcal{S}_{\mathrm{feas}}$ to obtain invariant coordinates
for the post-constraint state space.

\subsection{Value Invariance and Structure-Consistent Policies}
\label{sec:value_inv}

\paragraph{From value preservation to control structure.}
We now connect the generator-level value-preserving condition to consequences for optimal control.
The key step is that exact value preservation leaves the HJB operator unchanged:
reward compatibility and the generator-level condition imply
$\mathcal{H}(U_{g_\vartheta}V)=U_{g_\vartheta}(\mathcal{H}V)$,
where invertibility of $h_\vartheta$ ensures the supremum over actions is preserved.
Then, by uniqueness of bounded viscosity solutions (Theorem~\ref{thm:hjb_unique}),
the optimal value must be invariant along the transformation orbit.

\begin{theorem}[Exact value-preserving structure implies value invariance]
\label{thm:value_invariance}
Assume Assumption~\ref{ass:hjb} and let $(g_\vartheta,h_\vartheta)$ satisfy the exact structure conditions of Definition~\ref{def:exact_sym}.
Then the optimal value satisfies
\begin{equation}
\label{eq:value_invariance}
V^\star(g_\vartheta(s))=V^\star(s),\qquad \forall \vartheta\in G,\ \forall s\in\mathcal{S}.
\end{equation}
Moreover, if $I$ is a complete invariant coordinate map as in Theorem~\ref{thm:frobenius},
then locally there exists $\bar V^\star$ such that $V^\star(s)=\bar V^\star(I(s))$.
\end{theorem}

Value invariance already implies that the control problem effectively lives on the reduced coordinates $z=I(s)$.
To make this operational at the policy level, one often seeks a structure-consistent optimal selector.

\begin{corollary}[Orbit-consistent optimal actions (measurable selector form)]
\label{cor:policy_structure}
In addition to Assumption~\ref{ass:hjb}, assume that for each $s$ the maximizer set
$
\Argmax(s):=\arg\max_{a\in\mathcal{A}}\{\,r(s,a)+(\mathcal{L}^a V^\star)(s)\,\}
$
is nonempty and admits a measurable selector $a^\star(s)\in\Argmax(s)$.
Then for any $\vartheta\in G$,
$
a\in \Argmax(s)\quad \Longrightarrow \quad h_\vartheta(a)\in \Argmax(g_\vartheta(s)).
$
In particular, one can choose an optimal selector satisfying the structure-consistency relation
$a^\star(g_\vartheta(s))=h_\vartheta(a^\star(s))$ on any neighborhood where such a selector exists.
\end{corollary}

A concrete planar-rotation instance, together with its visualization, is provided in Appendix~\ref{app:running_example}.

\section{Approximate Value-Preserving Structure and Stability of the Optimal Value}
\label{sec:approx_sym}

Exact value-preserving structure implies exact value invariance along transformation orbits (Section~\ref{sec:lie_sym}).
Here we define \emph{approximate} value-preserving structure by bounding the controlled-generator and reward mismatch over a test-function class, and show that it yields a corresponding \emph{approximate} value-invariance bound.
The error scales linearly with the mismatch and is amplified by $1/\beta$, reflecting increased sensitivity under weaker discounting.

\subsection{Approximate Value-Preserving Structure as Bounded Operator Mismatch}
\label{sec:approx_def}

Exact value-preserving equalities may be too strict in data-driven or partially observed systems.
We therefore quantify value preservation through bounded mismatch of the transformed controlled generator and reward.

\begin{definition}[$(\varepsilon_{\mathcal{L}},\varepsilon_r)$-approximate value-preserving structure]
\label{def:eps_sym}
Fix a test-function class $\mathcal{F}\subseteq C^2(\mathcal{S})$ and the (standard) $C^2$ norm
$
\|f\|_{C^2}:=\|f\|_\infty+\|\nabla f\|_\infty+\|\nabla^2 f\|_\infty.
$
We say $(g_\vartheta,h_\vartheta)$ is an $(\varepsilon_{\mathcal{L}},\varepsilon_r)$-approximate value-preserving structure
if for all $\vartheta\in G$, all $a\in\mathcal{A}$, and all $f\in\mathcal{F}$,
\begin{equation}
\label{eq:eps_gen_equiv}
\left\|
L_a(f\circ g_\vartheta)
-
(L_{h_\vartheta(a)}f)\circ g_\vartheta
\right\|_\infty
\le
\epsilon_L\|f\|_{C^2}.
\end{equation}
and for all $(s,a)\in\mathcal{S}\times\mathcal{A}$,
\begin{equation}
\label{eq:eps_reward}
\big|r(g_\vartheta(s),h_\vartheta(a))-r(s,a)\big|\le \varepsilon_r.
\end{equation}
\end{definition}

Condition~\eqref{eq:eps_gen_equiv} measures how far the transformed dynamics deviate from exact generator-level value preservation
\emph{as seen through} a function class $\mathcal{F}$.
The scaling by $\|f\|_{C^2}$ is natural because $\mathcal{L}^a f$ involves first and second derivatives of $f$
through $b\cdot\nabla f$ and $\mathrm{tr}(Q\nabla^2 f)$.
Condition~\eqref{eq:eps_reward} separately controls violations of reward compatibility.

\subsection{From Approximate Value Preservation to Approximate Value Invariance}
\label{sec:approx_value}

Under exact value preservation, reward compatibility and the generator-level value-preserving condition imply that the HJB operator
$\mathcal{H}$ commutes with pullbacks: $\mathcal{H}(U_{g_\vartheta}V)=U_{g_\vartheta}(\mathcal{H}V)$.
With approximate value preservation, we obtain an \emph{inequality} controlling the HJB commutator
$\mathcal{H}U_{g_\vartheta}-U_{g_\vartheta}\mathcal{H}$, which then translates into a bound on
$V^\star\circ g_\vartheta - V^\star$ via the discounted comparison principle.

\begin{theorem}[Approximate value-preserving structure implies approximate value invariance]
\label{thm:approx_value}
Assume Assumption~\ref{ass:hjb}.
Let $(g_\vartheta,h_\vartheta)$ be an $(\varepsilon_{\mathcal{L}},\varepsilon_r)$-approximate value-preserving structure
in the sense of Definition~\ref{def:eps_sym}.
Assume additionally that $V^\star\in C^2(\mathcal{S})\cap \mathcal{F}$.
Then for all $\vartheta\in G$,
\begin{equation}
\label{eq:approx_value}
\|V^\star\circ g_\vartheta - V^\star\|_\infty
\le \frac{1}{\beta}\Big(\varepsilon_r+\varepsilon_{\mathcal{L}}\|V^\star\|_{C^2}\Big).
\end{equation}
\end{theorem}

The bound~\eqref{eq:approx_value} is intentionally stated in a strong-regularity regime ($V^\star\in C^2$)
so that the generator mismatch in~\eqref{eq:eps_gen_equiv} can be evaluated on $V^\star$ directly.
In general, $V^\star$ may only be a viscosity solution; extending~\eqref{eq:approx_value} to that setting
typically requires replacing $\|V^\star\|_{C^2}$ by an appropriate smooth approximation argument
(or a stability theorem for viscosity solutions under perturbations of the Hamiltonian).

\section{Our VPSD-RL Algorithm}
\label{sec:learning_algo}

To develop VPSD-RL, we learn infinitesimal generators $(X,Y)$ by minimizing empirical determining-equation residuals from differentiable environment models, exponentiate them via ODE flows to obtain finite transforms $(g_\alpha,h_\alpha)$, and use these transforms for RL via (i) transition augmentation and (ii) transformation-consistency regularization.
Algorithm~\ref{alg:main} in Appendix gives the full procedure.

\noindent {\bf Learning Infinitesimal Generators from Data}
\label{sec:learn_infinitesimal}
Let $\hat b_\omega,\hat\Sigma_\omega,\hat r_\omega$ be differentiable parametric models of $(b,\Sigma,r)$.
Since the controlled generator depends on the diffusion only through $Q=\Sigma\Sigma^\top$,
we work with the induced diffusion matrix model
$
\hat Q_\omega(s,a):=\hat\Sigma_\omega(s,a)\hat\Sigma_\omega(s,a)^\top \in \mathbb{R}^{d\times d}.
$
We parameterize infinitesimal generators by $X_\theta:\mathcal{S}\to\mathbb{R}^d$ and
$Y_\phi:\mathcal{A}\to\mathbb{R}^m$.

Motivated by Theorem~\ref{thm:determining_necessity}, define residuals
\begin{align}
\label{eq:residuals}
R_r(s,a)
&:= X_\theta(s)\cdot\nabla_s \hat r_\omega(s,a)+Y_\phi(a)\cdot\nabla_a \hat r_\omega(s,a),\\
R_b(s,a)
&
:=
\nabla_s\hat b_\omega(s,a)X_\theta(s)
-
(\nabla_sX_\theta(s))\hat b_\omega(s,a)
+
\nabla_a\hat b_\omega(s,a)Y_\phi(a)
-
\frac{1}{2}\Delta_{\hat Q_\omega(s,a)}X_\theta(s).
\\
R_Q(s,a)
&:= \nabla_s \hat Q_\omega(s,a)[X_\theta(s)]
\\&-(\nabla_s X_\theta(s))\,\hat Q_\omega(s,a)-\hat Q_\omega(s,a)(\nabla_s X_\theta(s))^\top
+\nabla_a \hat Q_\omega(s,a)[Y_\phi(a)].
\end{align}
Here $\nabla_s \hat Q_\omega(s,a)[X]$ denotes the directional derivative of the matrix field $\hat Q_\omega(\cdot,a)$
at $s$ along direction $X_\theta(s)$, and similarly for $\nabla_a\hat Q_\omega(s,a)[Y]$. In the experiments, the generator class is restricted to affine or locally linear vector fields, so the It\^o correction term \(\Delta_{\hat Q_\omega}X_\theta\) vanishes. For general nonlinear generator classes, \(R_b\) should include the additional \(-\frac12\Delta_{\hat Q_\omega}X_\theta\) term from Theorem~\ref{thm:determining_necessity}.

We minimize a weighted squared residual loss over a replay distribution $\rho(s,a)$:
\begin{equation}
\label{eq:sym_loss}
\min_{\theta,\phi}\ 
\mathbb{E}_{(s,a)\sim\rho}\Big[
\|R_b(s,a)\|_2^2+\lambda_Q\|R_Q(s,a)\|_F^2+\lambda_r|R_r(s,a)|^2
\Big]
\;+\;
\lambda_{\mathrm{nrm}}\Big(\mathbb{E}_{s\sim\rho_S}\|X_\theta(s)\|_2^2-1\Big)^2,
\end{equation}
where $\rho_S$ is the state marginal of $\rho$.
The normalization penalty prevents the trivial zero-field solution and fixes the overall scale of $X_\theta$.
(When learning multiple generators, one may additionally impose orthonormality constraints
$\mathbb{E}[X_i\cdot X_j]=\delta_{ij}$ to fix a basis of the symmetry algebra.)

% \paragraph{Connection to approximate operator equivariance.}
The residuals in~\eqref{eq:residuals} enforce the \emph{infinitesimal} determining equations.
After exponentiation (Section~\ref{sec:flow_integration}), small residuals imply that the finite-step transforms
$(g_\alpha,h_\alpha)$ yield small generator mismatch for small $|\alpha|$ (up to model and numerical integration errors),
thereby providing a concrete path from minimizing~\eqref{eq:sym_loss} to reducing
$(\varepsilon_{\mathcal{L}},\varepsilon_r)$ in Definition~\ref{def:eps_sym}.

\noindent {\bf Exponentiation via Flow Integration}
\label{sec:flow_integration}
%
% \paragraph{From vector fields to finite transformations.}
Given learned infinitesimal generators $(X_\theta,Y_\phi)$, we obtain finite transformations by integrating
their flows (i.e., exponentiating in the Lie sense):
\begin{equation}
\label{eq:flow}
\frac{d}{d\alpha}g_\alpha(s)=X_\theta(g_\alpha(s)),\quad g_0(s)=s,
\qquad
\frac{d}{d\alpha}h_\alpha(a)=Y_\phi(h_\alpha(a)),\quad h_0(a)=a.
\end{equation}
In practice, we integrate~\eqref{eq:flow} numerically for small $|\alpha|$.
If $\mathcal{S}$ or $\mathcal{A}$ is constrained (e.g., box constraints),
we apply a projection/retraction after each solver step to keep the iterates feasible.

\noindent {\bf Value-Preserving Reinforcement Learning Updates.}
\label{sec:sym_aware_rl}

Once a finite transformation $(g_\alpha,h_\alpha)$ is available, we use it in RL in two complementary ways.
First, transition augmentation maps collected samples $(s,a,s',r)$ to transformed samples
$(g_\alpha(s),h_\alpha(a),g_\alpha(s'),r)$, allowing the replay buffer to share experience along discovered value-preserving orbits.
Second, transformation-consistency regularization encourages the learned value and policy networks to make compatible predictions on transformed state-action pairs.
Under exact value preservation these operations are structure-preserving data-sharing steps; under approximate value preservation their effect is controlled by the mismatch bounds in Section~\ref{sec:approx_sym}.
Detailed losses and implementation choices are deferred to Appendix~\ref{app:vp_rl_updates}.

Algorithm~\ref{alg:main} in Appendix alternates among (i) data collection with the current policy,
(ii) optional environment model fitting to obtain differentiable $\hat b_\omega,\hat Q_\omega,\hat r_\omega$,
(iii) value-preserving structure discovery by minimizing~\eqref{eq:sym_loss} using automatic differentiation, and
(iv) RL updates using augmentation and/or transformation-consistency regularization with the induced finite transforms.

% ============================================================
\section{Guarantees of VPSD-RL}
\label{sec:guarantees_summary}

We establish an end-to-end guarantees for VPSD-RL which connects (i) learning decision-preserving
infinitesimal generators from data, (ii) exponentiating them into finite transformations, and
(iii) injecting these transformations into RL.
The full formal statements (assumptions, theorems, and proof sketches) are provided in Appendix~\ref{app:guarantees_full}.

\textbf{Optimization guarantee.}
Theorem~\ref{thm:sgd_stationary} shows that stochastic optimization of the structure-discovery objective reaches approximate stationary points under standard smoothness and bounded-variance conditions:
$
\frac{1}{T}\sum_{t=0}^{T-1}
\mathbb{E}\big[\|\nabla \mathcal{L}_{\mathrm{sym}}(\theta_t,\phi_t)\|_2^2\big]
\le
\frac{2\big(\mathcal{L}_{\mathrm{sym}}(\theta_0,\phi_0)-\mathcal{L}_{\inf}\big)}{\eta T}
+\eta L\sigma^2,
$
and diminishing stepsizes give the standard $\mathcal{O}(1/\sqrt{T})$ stationarity rate.

\textbf{Statistical consistency.}
Theorem~\ref{thm:consistency} states that, under realizability and consistent environment-model estimation, learned generators converge in $\rho$-mean to solutions of the determining equations.
Consequently, the induced generator and reward mismatches satisfy $\varepsilon_{\mathcal L}\to 0$ and $\varepsilon_r\to 0$ on the data support.

\textbf{Exact augmentation.}
Theorem~\ref{thm:exact_aug_preserve} gives a discrete-time contraction analogue: if the augmentation is exactly value-preserving, Bellman backups computed from augmented samples preserve the same optimal Bellman operator, so tabular value iteration keeps the usual convergence rate
$\|V_k-V^\star\|_\infty\le \gamma^k\|V_0-V^\star\|_\infty$.

\textbf{Approximate augmentation.}
Theorem~\ref{thm:approx_vi} shows that if augmentation is only $(\varepsilon_P,\varepsilon_r)$-approximate, the perturbed Bellman operator remains a $\gamma$-contraction and its fixed point satisfies
$
\|\tilde V^\star - V^\star\|_\infty
\le
\frac{1}{1-\gamma}
\left(\varepsilon_r+\gamma\,\varepsilon_P\|V^\star\|_\infty\right).
$
The numerical flow approximation guarantee, which bounds $\|\hat g_\alpha-g_\alpha\|$ and $\|\hat h_\alpha-h_\alpha\|$ by order $h^p$, is stated in Theorem~\ref{thm:ode_error} in the appendix.

\section{Experiments}
\label{sec:experiments}

Our experiments provide supporting evidence for the value-preserving structure theory in Sections~\ref{sec:lie_sym}--\ref{sec:guarantees_summary}.
We evaluate whether (i) infinitesimal generators can be recovered from data through determining-equation residual minimization,
(ii) the induced finite transformations yield small operator mismatch and value non-invariance,
and (iii) using the discovered transformations improves sample efficiency and robustness in continuous-control RL.
We report episodic return as a function of environment steps, final performance over the last evaluation window, area under the learning curve (AUC) up to a fixed budget, and wall-clock cost.
For interpretable systems with known ground-truth generators, we additionally report residual statistics
$\mathbb{E}\|R_b\|_2^2$, $\mathbb{E}\|R_Q\|_F^2$, $\mathbb{E}|R_r|^2$,
generator alignment, and empirical value non-invariance
$\mathbb{E}_{s,\alpha}|V(s)-V(g_\alpha(s))|$.
Full environment descriptions, hyperparameters, ablations, and all extended curves are deferred to Appendix~\ref{app:full_experiments}.

\begin{figure}[t]
  \centering
  \begin{subfigure}[t]{0.24\textwidth}
    \centering
    \includegraphics[width=\linewidth]{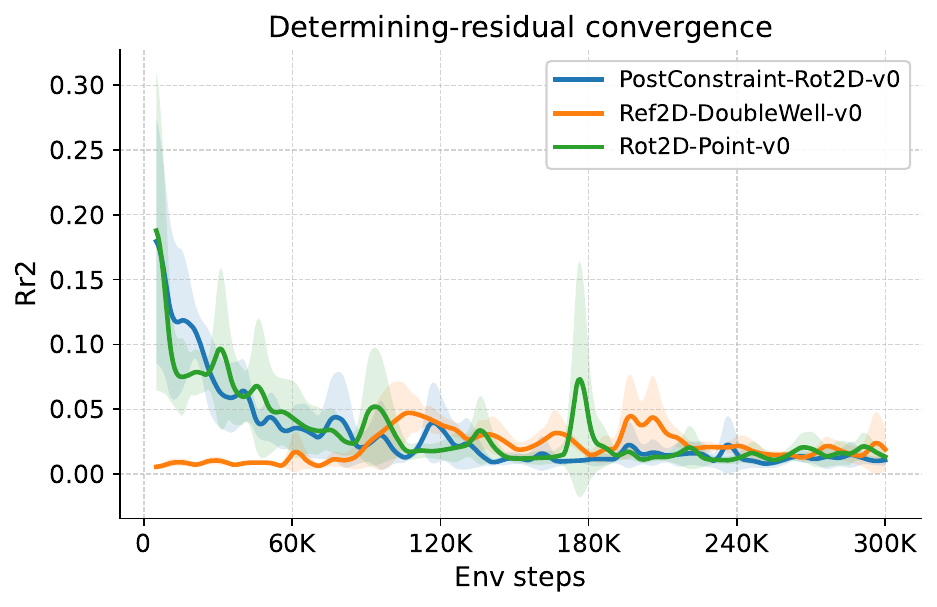}
    \caption{\texttt{Rot2D}}
    \label{fig:rr2}
  \end{subfigure}\hfill
  \begin{subfigure}[t]{0.24\textwidth}
    \centering
    \includegraphics[width=\linewidth]{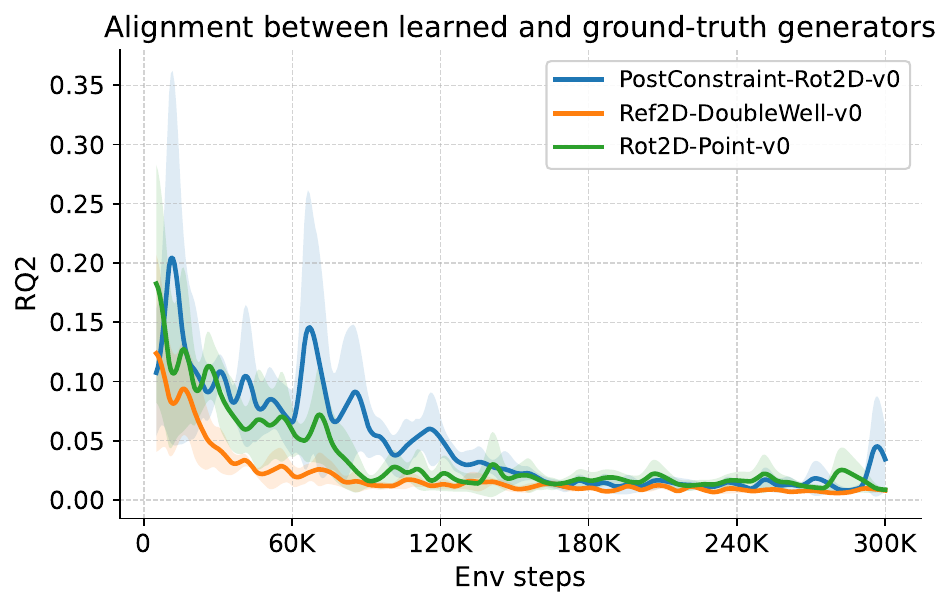}
    \caption{\texttt{Q-res.}}
    \label{fig:rq2}
  \end{subfigure}\hfill
  \begin{subfigure}[t]{0.24\textwidth}
    \centering
    \includegraphics[width=\linewidth]{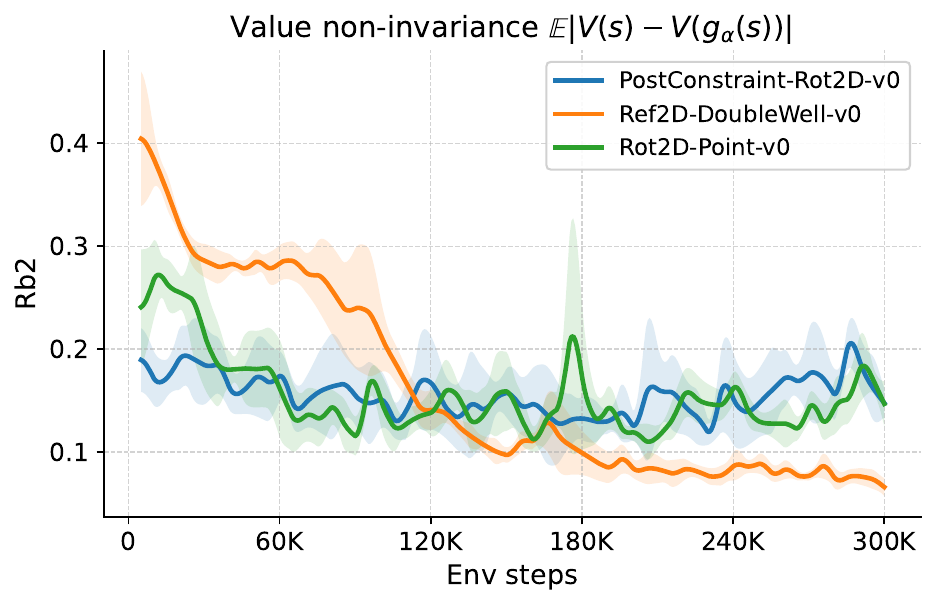}
    \caption{\texttt{b-res.}}
    \label{fig:rb2}
  \end{subfigure}
  \begin{subfigure}[t]{0.24\textwidth}
    \centering
    \includegraphics[width=\linewidth]{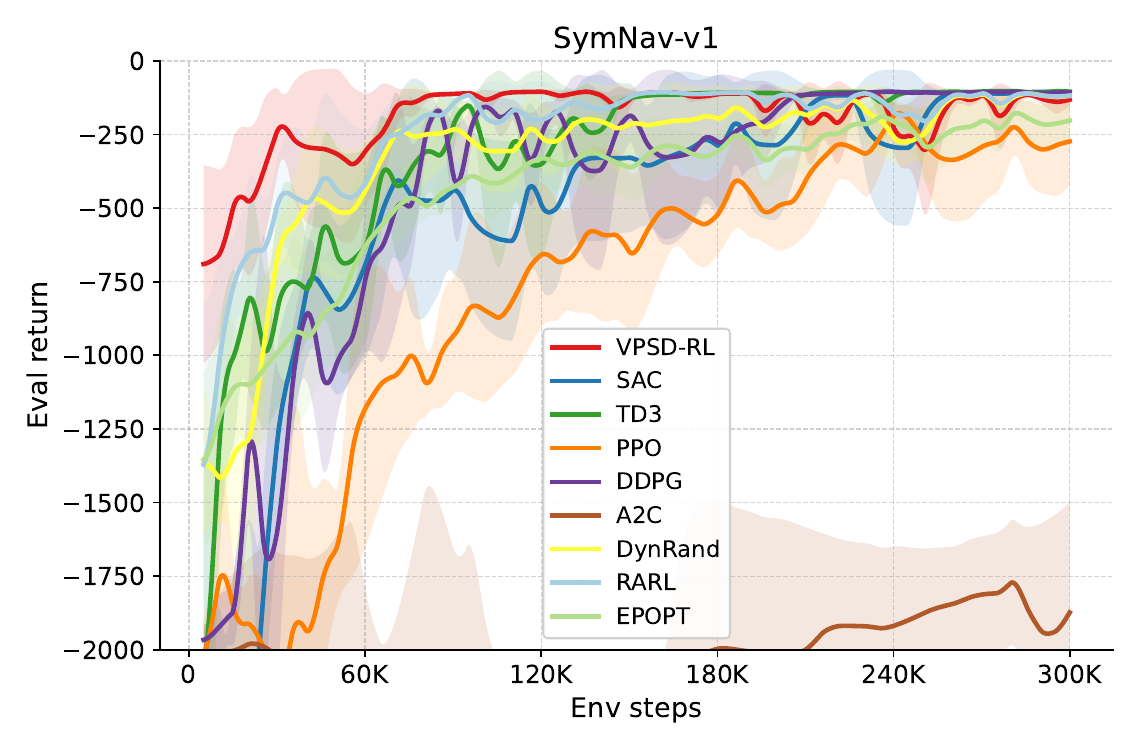}
    \caption{\texttt{SymNav-V1}}
    \label{fig:symnav_representative}
  \end{subfigure}
  \caption{
  Diagnostics on interpretable controlled diffusions and one representative \emph{SymNav-15} variant.
  The first panels show determining-residual convergence, generator alignment, and value non-invariance diagnostics for systems with known value-preserving transformations.
  The last panel shows a representative learning curve on \emph{SymNav-V1}.
  Curves report mean $\pm$ standard error over seeds.
  }
  \label{fig:synth_diag}
\end{figure}

Figure~\ref{fig:synth_diag} supports the mechanism predicted by the theory.
On interpretable controlled diffusions, minimizing the determining-equation residuals recovers transformation directions aligned with the known generators and reduces empirical value non-invariance along the learned flow.
On the representative \emph{SymNav-15} task, using the learned transformations improves sample efficiency relative to standard continuous-control baselines, suggesting that the recovered structure is useful beyond the synthetic diagnostic setting.

\begin{figure*}[t]
  \centering
  \begin{subfigure}[t]{0.32\textwidth}
    \centering
    \includegraphics[width=\linewidth]{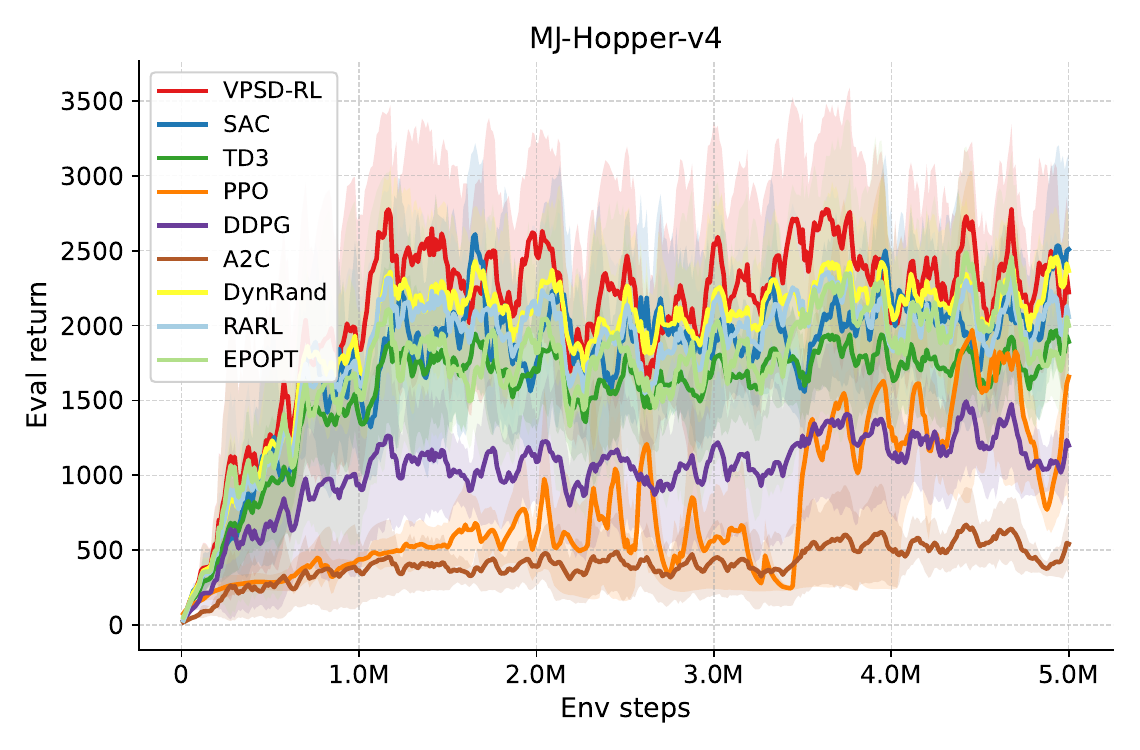}
    \caption{\texttt{Hopper-v4}}
    \label{fig:mujoco_hopper}
  \end{subfigure}\hfill
  \begin{subfigure}[t]{0.32\textwidth}
    \centering
    \includegraphics[width=\linewidth]{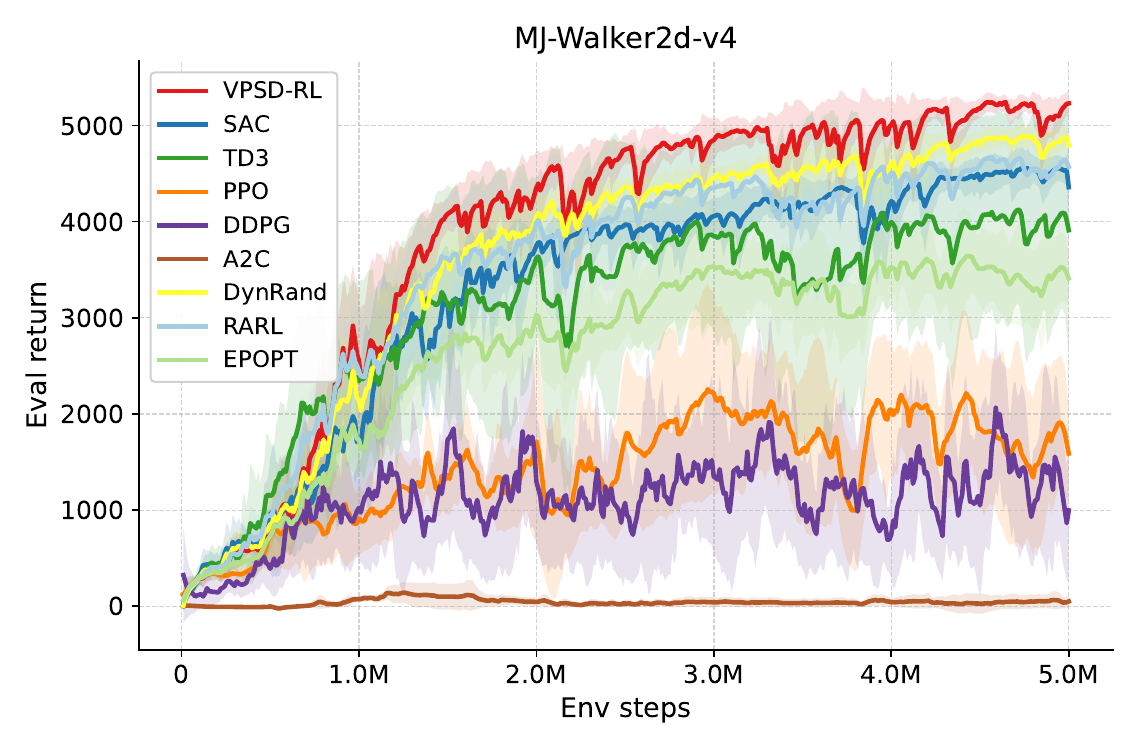}
    \caption{\texttt{Walker2d-v4}}
    \label{fig:mujoco_walker}
  \end{subfigure}\hfill
  \begin{subfigure}[t]{0.32\textwidth}
    \centering
    \includegraphics[width=\linewidth]{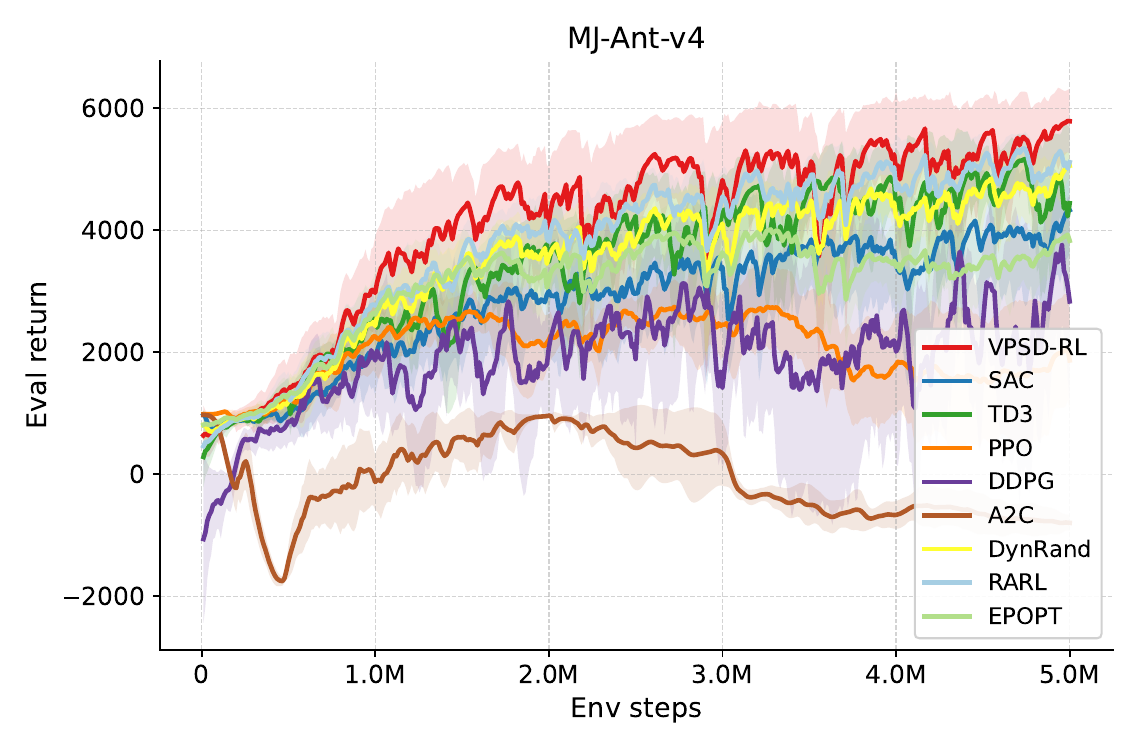}
    \caption{\texttt{Ant-v4}}
    \label{fig:mujoco_ant}
  \end{subfigure}
  \caption{
  MuJoCo locomotion performance on \texttt{Hopper-v4}, \texttt{Walker2d-v4}, and \texttt{Ant-v4}.
  Curves report mean $\pm$ standard error over seeds.
  }
  \label{fig:mujoco_curves}
\end{figure*}

Figure~\ref{fig:mujoco_curves} shows that the benefit of VPSD-RL is task-dependent but consistent with the amount of reusable structure in the task.
The gain is modest on \texttt{Hopper-v4}, whose dominant structure is a single hopping template; clearer gains appear on \texttt{Walker2d-v4}, where left-right phase correspondence creates reusable coordination; and the strongest improvement appears on \texttt{Ant-v4}, where multi-limb coordination produces richer approximately value-preserving transformation patterns.
Thus the structural summary previously reported as a separate table can be stated directly in text.

\begin{table*}[t]
\centering
\begin{minipage}[c]{0.34\textwidth}
\caption{
Compact quantitative summary on \texttt{Ant-v4}.
Numbers are mean $\pm$ std over seeds.
Final return is averaged over the last evaluation window; AUC is computed over a fixed step budget; wall-clock is seconds per $10^5$ environment steps.
}
\label{tab:ant_summary}
\end{minipage}\hfill
\begin{minipage}[c]{0.64\textwidth}
\centering
\scriptsize
\setlength{\tabcolsep}{3pt}
\begin{tabular}{lccc}
\toprule
Method & Final return $\uparrow$ & AUC@B $\uparrow$ & Wall-clock $\downarrow$ \\
\midrule
A2C                 & \texttt{-790.30}$\pm$\texttt{103.76} & \texttt{-388.66}$\pm$\texttt{127.69} & \texttt{391}$\pm$\texttt{113} \\
DDPG                & \texttt{2637.74}$\pm$\texttt{1385.82} & \texttt{924.95}$\pm$\texttt{307.18} & \texttt{437}$\pm$\texttt{8} \\
PPO                 & \texttt{1922.56}$\pm$\texttt{615.62} & \texttt{1414.10}$\pm$\texttt{205.77} & \texttt{215}$\pm$\texttt{0} \\
SAC                 & \texttt{4477.82}$\pm$\texttt{677.98} & \texttt{1227.46}$\pm$\texttt{172.73} & \texttt{397}$\pm$\texttt{159} \\
TD3                 & \texttt{3984.84}$\pm$\texttt{1111.78} & \texttt{1463.37}$\pm$\texttt{123.77} & \texttt{767}$\pm$\texttt{77} \\
\midrule
DynRand             & \texttt{5149.99}$\pm$\texttt{110.37} & \texttt{1440.34}$\pm$\texttt{152.56} & \texttt{434}$\pm$\texttt{108} \\
RARL                & \texttt{4903.49}$\pm$\texttt{615.20} & \texttt{1558.29}$\pm$\texttt{227.95} & \texttt{619}$\pm$\texttt{119} \\
EPOpt               & \texttt{3872.36}$\pm$\texttt{89.81}  & \texttt{1533.66}$\pm$\texttt{254.87} & \texttt{343}$\pm$\texttt{81} \\
\midrule
VPSD-RL (Aug+Reg)   & \texttt{5822.15}$\pm$\texttt{470.25} & \texttt{1653.21}$\pm$\texttt{336.34} & \texttt{471}$\pm$\texttt{161} \\
\bottomrule
\end{tabular}
\end{minipage}
\end{table*}

The compact comparison in Table~\ref{tab:ant_summary} shows that VPSD-RL improves final return and AUC on \texttt{Ant-v4} while adding moderate computational overhead.
Detailed aggregate results for all \emph{SymNav-15} variants, full MuJoCo curves, robustness baselines, and ablations are provided in Appendix~\ref{app:full_experiments}. These empirical diagnostics are used as reliability checks for the discovered transformations, rather than as an assumption that every continuous-control task admits a useful global symmetry.

\section{Conclusion}
\label{sec:conclusion}

We introduced \textbf{VPSD-RL}, an operator-level framework for discovering and exploiting value-preserving structure in continuous-time reinforcement learning.
Rather than assuming a prescribed symmetry, VPSD-RL defines exact and approximate value preservation through commutation of the transformed controlled generator and reward with the HJB operator, making classical Lie symmetries a special case.
We showed that exact value-preserving structure implies value invariance, while bounded generator/reward mismatch yields quantitative stability of the optimal value along approximate orbits.
We then turned these principles into a practical pipeline that learns infinitesimal generators from data, exponentiates them into finite transformations, and uses them for transition augmentation and transformation-consistency regularization.
The resulting theory and experiments suggest that value-preserving structure discovery can improve data efficiency and robustness in continuous-control RL, especially when useful transformations are present but not known in advance. These gains should be interpreted under the stated local-modeling assumptions: the learned flows are local and their reliability is assessed through residual and invariance diagnostics. Limitations and broader impacts are discussed in Appendix.

{
\small
\nocite{*}
\bibliographystyle{unsrtnat}
\bibliography{ref}
}

%%%%%%%%%%%%%%%%%%%%%%%%%%%%%%%%%%%%%%%%%%%%%%%%%%%%%%%%%%%%

\appendix

\section{Algorithm}
\begin{algorithm}[t]
\caption{Value-preserving structure discovery and reinforcement learning with learned transformations}
\label{alg:main}
\begin{algorithmic}[1]
\STATE Initialize replay buffer $\mathcal{D}$; models $(\hat b_\omega,\hat\Sigma_\omega,\hat r_\omega)$; fields $(X_\theta,Y_\phi)$; RL parameters $\psi$
\FOR{iterations $=1,2,\dots$}
  \STATE Collect rollouts with policy $\pi_\psi$; store $(s,a,s',r)$ into $\mathcal{D}$
  \STATE Update $(\hat b_\omega,\hat\Sigma_\omega,\hat r_\omega)$ on minibatches from $\mathcal{D}$ if unknown (and set $\hat Q_\omega=\hat\Sigma_\omega\hat\Sigma_\omega^\top$)
  \STATE Sample minibatch $\{(s_i,a_i)\}$ from $\mathcal{D}$
  \STATE Compute residuals $R_r,R_b,R_Q$ via automatic differentiation
  \STATE Update $(\theta,\phi)$ by (stochastic) gradient steps on \eqref{eq:sym_loss}
  \STATE Numerically integrate \eqref{eq:flow} for small $\alpha$ to obtain $(g_\alpha,h_\alpha)$
  \STATE Use $(g_\alpha,h_\alpha)$ for transition augmentation and/or transformation-consistency regularization; update $\psi$
\ENDFOR
\end{algorithmic}
\end{algorithm}

\subsection{Using learned transformations in RL}
\label{app:vp_rl_updates}

Given a transition tuple $(s,a,s',r)$ collected from the environment, transition augmentation generates
\[
(\tilde s,\tilde a,\tilde s',\tilde r)
:=
\big(g_\alpha(s),\,h_\alpha(a),\,g_\alpha(s'),\,r\big),
\]
and inserts it into the replay buffer.
Under exact value preservation, $\tilde r=r$ and the augmented transition is distributed as a valid transition from $(\tilde s,\tilde a)$.
Under approximate value preservation, the augmented transition is a controlled perturbation whose effect is governed by the generator/reward mismatch bounds in Section~\ref{sec:approx_sym} and by the discrete-time analogue in Section~\ref{app:guarantees_full}.

We also encourage the learned critic/value and policy to respect the discovered transformation.
For a value function $V_\psi$, we use the invariance penalty
\[
\mathcal{R}_{\mathrm{inv}}(\psi)
:=
\mathbb{E}_{s\sim\rho_S,\ \alpha\sim \nu}
\Big[\big(V_\psi(s)-V_\psi(g_\alpha(s))\big)^2\Big],
\]
where $\nu$ is a distribution over small flow times, such as the uniform distribution on $[-\alpha_0,\alpha_0]$.
For a stochastic policy $\pi_\psi(\cdot\mid s)$, we use a policy-transport consistency penalty based on the pushforward $(h_\alpha)_\#\pi_\psi(\cdot\mid s)$:
\[
\mathcal{R}_{\mathrm{pol}}(\psi)
:=
\mathbb{E}_{s\sim\rho_S,\ \alpha\sim\nu}
\Big[
\mathrm{KL}\big(\pi_\psi(\cdot\mid g_\alpha(s))
\ \|\ 
(h_\alpha)_\#\pi_\psi(\cdot\mid s)\big)
\Big].
\]
In implementation, this term can be estimated by sampling $a\sim\pi_\psi(\cdot\mid s)$ and evaluating the log-density of $h_\alpha(a)$ under $\pi_\psi(\cdot\mid g_\alpha(s))$.
The RL parameters are updated with the base actor--critic loss plus these consistency penalties:
\[
\min_\psi\ 
\mathcal{L}_{\mathrm{RL}}(\psi)
+
\lambda_{\mathrm{inv}}\mathcal{R}_{\mathrm{inv}}(\psi)
+
\lambda_{\mathrm{pol}}\mathcal{R}_{\mathrm{pol}}(\psi).
\]

\section{Theory Guarantees }
\label{app:guarantees_full}

% \paragraph{Roadmap of this section.}
We provide three layers of guarantees that mirror the pipeline in Section~\ref{sec:learning_algo}.
First, we give standard nonconvex optimization guarantees for stochastic structure discovery under smoothness and bounded variance.
Second, under realizability and consistent dynamics estimation, we state a consistency result showing that the learned fields
converge (in data measure) to a solution of the determining equations, implying vanishing operator mismatch on the data support.
Third, we bound numerical errors introduced by finite-step ODE solvers when exponentiating the learned vector fields.
Finally, to give an end-to-end non-asymptotic statement free of actor-training nonconvexity,
we include a discrete-time tabular analogue showing how exact/approximate value-preserving augmentation perturbs the Bellman fixed point.

\subsection{Optimization Guarantees for Structure Discovery}
\label{sec:sgd_sym}

\begin{assumption}[Smoothness and bounded variance for structure-discovery optimization]
\label{ass:sgd}
Let $\mathcal{L}_{\mathrm{sym}}(\theta,\phi)$ denote the (penalized) objective in~\eqref{eq:sym_loss}.
Assume $\mathcal{L}_{\mathrm{sym}}$ is $L$-smooth in $(\theta,\phi)$, i.e., its gradient is $L$-Lipschitz.
Assume the stochastic gradients used by the algorithm are unbiased and have uniformly bounded second moments:
$\mathbb{E}\|g_t-\nabla \mathcal{L}_{\mathrm{sym}}(\theta_t,\phi_t)\|_2^2\le \sigma^2$.
\end{assumption}

\begin{theorem}[Convergence of stochastic structure discovery to a stationary point]
\label{thm:sgd_stationary}
Under Assumption~\ref{ass:sgd}, running SGD with constant stepsize $\eta\le 1/L$ yields
\begin{equation}
\label{eq:sgd_const}
\frac{1}{T}\sum_{t=0}^{T-1}\mathbb{E}\big[\|\nabla \mathcal{L}_{\mathrm{sym}}(\theta_t,\phi_t)\|_2^2\big]
\le
\frac{2\big(\mathcal{L}_{\mathrm{sym}}(\theta_0,\phi_0)-\mathcal{L}_{\inf}\big)}{\eta T}
+\eta L\sigma^2,
\end{equation}
where $\mathcal{L}_{\inf}$ is a lower bound of $\mathcal{L}_{\mathrm{sym}}$.
With appropriately diminishing stepsizes, one obtains the standard rate
$\min_{0\le t\le T-1}\mathbb{E}\|\nabla \mathcal{L}_{\mathrm{sym}}(\theta_t,\phi_t)\|_2^2=\mathcal{O}(1/\sqrt{T})$.
\end{theorem}

% \paragraph{How to interpret Theorem~\ref{thm:sgd_stationary}.}
This result does \emph{not} claim global recovery of the true symmetry fields in general (the problem is nonconvex),
but it ensures that the optimization procedure reaches approximate stationary points at the usual SGD rates.
Consistency requires additional realizability assumptions, addressed next.

\subsection{Consistency under Realizability}
\label{sec:consistency}

\begin{assumption}[Realizability and consistent dynamics estimation]
\label{ass:realizable}
Assume there exist parameters $\omega^\star,\theta^\star,\phi^\star$ such that
$\hat b_{\omega^\star}=b$, $\hat Q_{\omega^\star}=Q$, and $\hat r_{\omega^\star}=r$ on the support of $\rho$,
and $(X_{\theta^\star},Y_{\phi^\star})$ satisfies the determining equations
\eqref{eq:det_reward}, \eqref{eq:det_drift_ito}, and \eqref{eq:det_diffQ_ito} on the support of $\rho$.
Assume the estimators $(\hat b_\omega,\hat Q_\omega,\hat r_\omega)$ are consistent as data grows,
and the structure-discovery objective~\eqref{eq:sym_loss} can be optimized to its global minimum.
\end{assumption}

\begin{theorem}[Consistency of learned infinitesimal generators (on data support)]
\label{thm:consistency}
Under Assumption~\ref{ass:realizable}, as the dataset size tends to infinity and the optimization error tends to zero,
the learned generators $(X_{\theta},Y_{\phi})$ converge in $\rho$-mean to a solution of the determining equations
\eqref{eq:det_reward}--\eqref{eq:det_diffQ_ito}.
In particular, the induced operator mismatch in Definition~\ref{def:eps_sym} satisfies
$\varepsilon_{\mathcal{L}}\to 0$ and $\varepsilon_r\to 0$ on the support of $\rho$.
\end{theorem}

% \paragraph{Identifiability remark.}
Even under realizability, infinitesimal generators are typically identifiable only up to choices of basis within the underlying transformation algebra.
The normalization term in~\eqref{eq:sym_loss} fixes the overall scale of $X_\theta$ but does not preclude equivalent generators
related by linear combinations related by linear combinations when multiple transformation directions are present.

\subsection{Numerical Flow Approximation}
\label{sec:ode_flow}

\begin{assumption}[Numerical flow integration]
\label{ass:ode}
Assume $X_\theta$ and $Y_\phi$ are Lipschitz on compact sets $\mathcal{K}_S\subseteq\mathcal{S}$ and
$\mathcal{K}_A\subseteq\mathcal{A}$ that contain the states/actions encountered by the algorithm.
Let $\hat g_\alpha$ and $\hat h_\alpha$ be computed by an order-$p$ one-step ODE solver with step size $h$
for $\alpha$ in a bounded interval.
\end{assumption}

\begin{theorem}[Flow approximation error]
\label{thm:ode_error}
Under Assumption~\ref{ass:ode}, for $\alpha$ in a bounded interval,
the numerical flows satisfy
\begin{equation}
\label{eq:ode_error}
\sup_{s\in\mathcal{K}_S}\|\hat g_\alpha(s)-g_\alpha(s)\|_2 \le C_g h^p,\qquad
\sup_{a\in\mathcal{K}_A}\|\hat h_\alpha(a)-h_\alpha(a)\|_2 \le C_h h^p,
\end{equation}
where $C_g$ and $C_h$ depend on Lipschitz constants of the fields on $\mathcal{K}_S,\mathcal{K}_A$ and on the solver stability region.
\end{theorem}

\subsection{A Discrete-Time Analogue: Value-Preserving Augmented Bellman Backups}
\label{sec:dt_analogue}

% \paragraph{Why include a discrete-time statement.}
The continuous-time theory above explains why (approximate) generator equivariance leads to (approximate) value invariance.
To provide a fully non-asymptotic end-to-end convergence statement that avoids nonconvex actor training,
we additionally state a discrete-time tabular analogue in which value iteration is a contraction and fixed-point perturbations
can be bounded explicitly.

Let $T$ be the optimal Bellman operator in a discounted discrete-time MDP:
\begin{equation}
\label{eq:bellman_opt}
(TV)(s)=\max_{a\in\mathcal{A}}\left\{r(s,a)+\gamma \int V(s')\,P(ds'\mid s,a)\right\}.
\end{equation}
$T$ is a $\gamma$-contraction in $\|\cdot\|_\infty$.

\begin{theorem}[Exact value-preserving augmentation preserves the Bellman operator]
\label{thm:exact_aug_preserve}
Assume there exists a transformation pair $(g,h)$ forming an exact value-preserving correspondence for the discrete-time MDP,
namely
$r(g(s),h(a))=r(s,a)$ and $P(\cdot\mid g(s),h(a))=g_\#P(\cdot\mid s,a)$ for all $(s,a)$.
Then the Bellman backup is invariant under augmentation by $(g,h)$ in expectation,
so tabular value iteration with exact value-preserving augmentation converges to $V^\star$ at the usual rate
$\|V_{k}-V^\star\|_\infty\le \gamma^k\|V_0-V^\star\|_\infty$.
\end{theorem}

\begin{theorem}[Approximate value-preserving augmentation yields a bounded fixed-point perturbation]
\label{thm:approx_vi}
Define total variation as $d_{\mathrm{TV}}(\mu,\nu):=\sup_{\|f\|_\infty\le 1}\big|\int f\,d(\mu-\nu)\big|$.
Assume the augmentation is $(\varepsilon_P,\varepsilon_r)$ approximate in the discrete-time sense:
\begin{equation}
\label{eq:dt_eps}
\sup_{s,a}\big|r(g(s),h(a))-r(s,a)\big|\le \varepsilon_r,\qquad
\sup_{s,a} d_{\mathrm{TV}}\!\big(P(\cdot\mid g(s),h(a)),\ g_\#P(\cdot\mid s,a)\big)\le \varepsilon_P.
\end{equation}
Let $\tilde T$ be the Bellman operator induced by the approximately augmented model.
Then for any bounded $V$,
\begin{equation}
\label{eq:operator_perturb}
\|(\tilde T V)-(T V)\|_\infty \le \varepsilon_r + \gamma\,\varepsilon_P \|V\|_\infty,
\end{equation}
and $\tilde T$ remains a $\gamma$-contraction.
Let $\tilde V^\star$ be the fixed point of $\tilde T$.
Then the fixed-point error satisfies
\begin{equation}
\label{eq:fixed_point_error}
\|\tilde V^\star - V^\star\|_\infty
\le \frac{1}{1-\gamma}\left(\varepsilon_r+\gamma\,\varepsilon_P\|V^\star\|_\infty\right).
\end{equation}
\end{theorem}

\subsection{HJB well-posedness used in the main text}
\label{app:hjb_assumption}

\begin{assumption}[HJB well-posedness]
\label{ass:hjb}
Assume:
(i) $\mathcal{A}$ is a compact metric space and $r,b,\Sigma$ are continuous in $(s,a)$ and Lipschitz in $s$
uniformly over $a$; moreover $|r|\le R_{\max}$ and $b,\Sigma$ are bounded.
(ii) The diffusion matrix $Q(s,a)=\Sigma(s,a)\Sigma(s,a)^\top$ is uniformly elliptic:
there exists $\lambda>0$ such that $\xi^\top Q(s,a)\xi\ge \lambda\|\xi\|^2$ for all $(s,a)$ and $\xi\in\mathbb{R}^d$.
(iii) Either $\mathcal{S}$ is bounded with an appropriate state-constraint or boundary condition,
or $\mathcal{S}=\mathbb{R}^d$ with a growth/boundedness condition ensuring that $V^\star$ is well defined and bounded.
(iv) A comparison principle holds for bounded viscosity sub- and supersolutions of~\eqref{eq:hjb}.
\end{assumption}

\section{Running Example: Planar Rotation as a Special Case}
\label{app:running_example}

Let
$
\hat e_3=
\begin{bmatrix}
0 & -1 & 0\\
1 & 0 & 0\\
0 & 0 & 0
\end{bmatrix},
\qquad
g_\theta=\exp(\theta \hat e_3)\in SO(3),
$
and let the induced action on vectors \(v\in\mathbb{R}^3\) be
$
\rho(g_\theta)v = R_z(\theta)v,
\qquad
R_z(\theta)=
\begin{bmatrix}
\cos\theta & -\sin\theta & 0\\
\sin\theta & \cos\theta & 0\\
0 & 0 & 1
\end{bmatrix},
\qquad \theta\in(-\pi,\pi].
$

Consider a state \(s=(x,x_d,\xi)\), where \(x\in\mathbb{R}^3\) is the current position,
\(x_d\in\mathbb{R}^3\) is a goal/location descriptor, and \(\xi\) collects scalar quantities left invariant by a transformation family.
Let the action be a control vector \(a=F\in\mathbb{R}^3\).
For the planar-orbiting case, define
$
g_\theta\cdot s := \bigl(R_z(\theta)x,\; R_z(\theta)x_d,\; \xi\bigr),
\qquad
h_\theta\cdot a := R_z(\theta)a.
$
If the controlled dynamics and reward depend only on relative geometry and transformation-invariant control quantities,
then the generator-level value-preserving condition and reward compatibility hold:
$
L_a U_{g_\theta} = U_{g_\theta} L_{h_\theta(a)},
\qquad
r(g_\theta\!\cdot s,\; h_\theta\!\cdot a)=r(s,a).
$
Hence the optimal policy can be transported consistently along the orbit,
$
\pi^\star(g_\theta\!\cdot s)\approx h_\theta\!\cdot \pi^\star(s).
$

The same perspective extends beyond globally symmetric coordinates.
For a UAV executing a figure-eight maneuver, one can align repeated phase-matched relative configurations across the two lobes of the trajectory.
These matched configurations need not form a global coordinate symmetry, but they can still induce approximate value-preserving correspondences when the local control semantics, geometry, and reward-to-go remain close.
This illustrates why the target of our method is broader value-preserving structure rather than symmetry alone.

\begin{figure}[t]
  \centering
  \includegraphics[width=\linewidth]{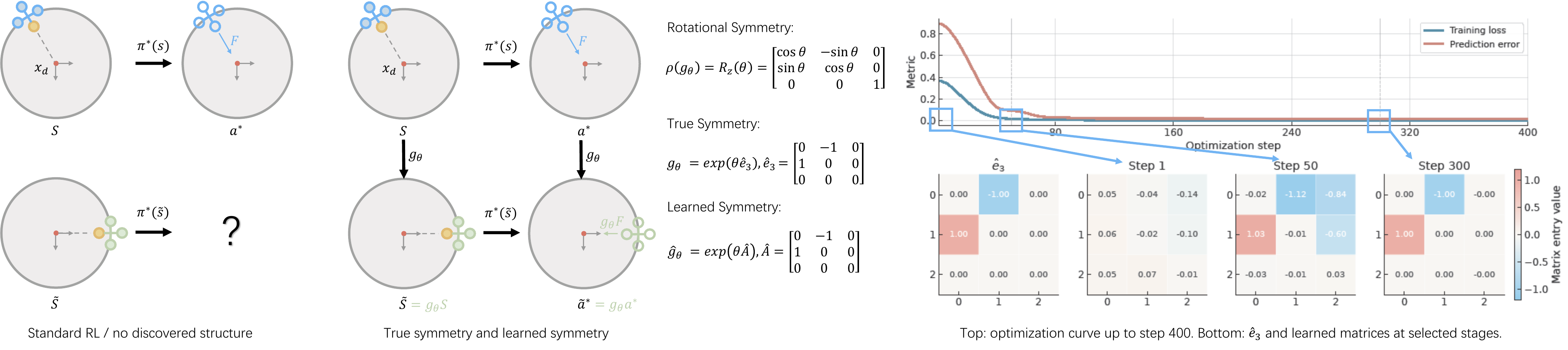}
  \caption{
  Running planar-rotation example illustrating the role of discovered value-preserving structure in continuous control.
  Left: a standard policy learner has no explicit mechanism tying a visited state \(s\) to its transformed counterpart \(\tilde s=g_\theta s\).
  Middle: in the rotational special case, the ground-truth transformation is generated by \(g_\theta=\exp(\theta \hat e_3)\), with induced action \(\rho(g_\theta)v=R_z(\theta)v\); our method instead learns an infinitesimal generator \(\hat A\) and constructs the finite transform \(\hat g_\theta=\exp(\theta \hat A)\).
  Right: numerical diagnostics show that, as optimization proceeds, the learned generator approaches the ground-truth rotational generator and the induced finite transforms become increasingly consistent with true orbit transport, supporting the value-preserving relation \(\pi^\star(g_\theta s)\approx h_\theta \pi^\star(s)\).
  }
  \label{fig:example}
\end{figure}

\section{Additional Experimental Details and Full Results}
\label{app:full_experiments}

This appendix provides full environment specifications, hyperparameters, and complete plots for all benchmark variants.
We also include ablations and additional diagnostics linking generator discovery to operator mismatch and value invariance.
Unless noted otherwise, all environments are implemented in \texttt{gymnasium} style with continuous observation and action spaces,
and at submission time, the anonymized supplementary material contains key implementation code for VPSD-RL, configuration templates. The complete cleaned repository and raw logs will be released upon acceptance.

% ------------------------------------------------------------
\subsection{Stage I: Full synthetic environments and diagnostics}
\label{app:env_synth}

\paragraph{Common state/action conventions.}
All three synthetic environments use the same state and action parameterization.
The state is
$
s=(x,y,v_x,v_y)\in\mathbb{R}^4
$
(position and velocity in the plane),
and the action is a bounded acceleration command
$
a=(a_x,a_y)\in[-1,1]^2
$.
Let $p=(x,y)$ and $v=(v_x,v_y)$.
All stochasticity enters through the velocity channels only.

Each environment is an Euler--Maruyama discretization of a controlled diffusion of the form
\begin{equation}
\label{eq:appendix_sde_template}
\mathrm{d}p_t = v_t\,\mathrm{d}t,\qquad
\mathrm{d}v_t = f(p_t,v_t,a_t)\,\mathrm{d}t + \sigma\,\mathrm{d}W_t,
\end{equation}
where $W_t$ is a 2D standard Wiener process and $\sigma>0$ is isotropic.
In discrete time with step size $\Delta t$ we use the semi-implicit update
\begin{equation}
\label{eq:appendix_em_semiimplicit}
v_{k+1} = v_k + \Delta t\, f(p_k,v_k,a_k) + \sigma\sqrt{\Delta t}\,\xi_k,\qquad
p_{k+1} = p_k + \Delta t\, v_{k+1},
\end{equation}
with $\xi_k\sim\mathcal{N}(0,I_2)$ i.i.d.

\paragraph{Rot2D-Point (\texttt{Rot2DPointEnv}).}
This is the baseline SO(2)-invariant controlled diffusion with linear damping:
\begin{equation}
\label{eq:rot2d_drift}
f(p,v,a)= a - \lambda v,
\end{equation}
where $\lambda>0$ is the damping coefficient.
The instantaneous reward is SO(2)-invariant:
\begin{equation}
\label{eq:rot2d_reward}
r(s,a) = -\|p\|_2^2 - 0.1\,\|v\|_2^2.
\end{equation}
The induced continuous symmetry action is the usual planar rotation:
for any $R\in \mathrm{SO}(2)$,
\begin{equation}
\label{eq:rot2d_group_action}
g_R(p,v,a) = (Rp, Rv, Ra),
\end{equation}
under which both the drift~\eqref{eq:rot2d_drift} and reward~\eqref{eq:rot2d_reward} are invariant.
The environment uses defaults $\Delta t=0.05$, $\sigma=0.02$, $\lambda=0.1$, horizon $T=200$ steps (see code listing).

\paragraph{Ref2D-DoubleWell (\texttt{Ref2DDoubleWellEnv}).}
This environment uses a double-well potential with reflection symmetry $x\mapsto -x$ when $\delta=0$.
Define the base potential
\begin{equation}
\label{eq:dw_U0}
U_0(x,y) = (x^2-1)^2 + y^2,
\end{equation}
and the dynamics potential (used in the drift) with a tunable symmetry-breaking term
\begin{equation}
\label{eq:dw_Udyn}
U_{\mathrm{dyn}}(x,y) = U_0(x,y) + \delta x.
\end{equation}
The drift is
\begin{equation}
\label{eq:dw_drift}
f(p,v,a)= -\nabla U_{\mathrm{dyn}}(p) - \lambda v + a,
\end{equation}
where
$
\nabla U_{\mathrm{dyn}}(x,y) = \big(4x(x^2-1)+\delta,\,2y\big).
$
The reward is kept \emph{independent} of $\delta$ (matching the implementation): 
\begin{equation}
\label{eq:dw_reward}
r(s,a) = -U_0(x,y) - 0.1\,\|v\|_2^2.
\end{equation}
When $\delta=0$, the environment is invariant under the reflection action
\begin{equation}
\label{eq:dw_reflection_action}
g(p,v,a) = ((-x,y),\,(-v_x,v_y),\,(-a_x,a_y)).
\end{equation}
When $\delta\neq 0$, the symmetry is only approximate; we use this setting to stress-test discovery under controlled mismatch.
The implementation additionally applies mild clipping to position/velocity to avoid numerical blow-ups (see code listing).
Defaults: $\Delta t=0.05$, $\sigma=0.02$, $\lambda=0.1$, horizon $T=200$ steps.

\paragraph{PostConstraint-Rot2D (\texttt{PostConstraintRot2DEnv}).}
This environment starts from the Rot2D drift~\eqref{eq:rot2d_drift} and then enforces a feasibility constraint on position:
\begin{equation}
\label{eq:annulus_constraint}
\|p\|_2 \in [r_{\min}, r_{\max}].
\end{equation}
After the unconstrained update~\eqref{eq:appendix_em_semiimplicit}, we project the position to the annulus:
\begin{equation}
\label{eq:annulus_projection}
p_{k+1} \leftarrow \Pi_{\mathcal{A}}(p_{k+1}),\quad \mathcal{A}=\{p:\ r_{\min}\le \|p\|_2 \le r_{\max}\},
\end{equation}
and when a projection occurs we remove the radial velocity component to maintain boundary consistency:
\begin{equation}
\label{eq:annulus_radial_vel_remove}
v_{k+1} \leftarrow v_{k+1} - u\,\langle v_{k+1},u\rangle,\qquad u=\frac{p_{k+1}}{\|p_{k+1}\|_2}.
\end{equation}
The reward matches Rot2D-Point:
\begin{equation}
\label{eq:post_reward}
r(s,a) = -\|p\|_2^2 - 0.1\,\|v\|_2^2.
\end{equation}
This environment preserves rotational symmetry in the interior but introduces non-smooth geometry and boundary effects,
making it a convenient testbed for constrained flows and retraction/projection handling in the discovered transformations.
Defaults: $\Delta t=0.05$, $\sigma=0.02$, $\lambda=0.1$, $r_{\min}=0.5$, $r_{\max}=2.0$, horizon $T=200$ steps.
The step function also logs whether projection occurred (\texttt{info["projected"]}).

\paragraph{Summary of synthetic parameters.}
Table~\ref{tab:synth_params} lists the default parameters used by our reference implementation (exact values in the released config files).

\begin{table}[t]
\centering
\caption{Synthetic environment default parameters in our implementation.}
\label{tab:synth_params}
\resizebox{0.85\linewidth}{!}{
\begin{tabular}{lcccccc}
\toprule
Env & $\Delta t$ & $\sigma$ & $\lambda$ & Horizon & Extra & Notes \\
\midrule
Rot2D-Point & 0.05 & 0.02 & 0.1 & 200 & -- & SO(2) invariant \\
Ref2D-DoubleWell & 0.05 & 0.02 & 0.1 & 200 & $\delta$ & reflection exact if $\delta=0$ \\
PostConstraint-Rot2D & 0.05 & 0.02 & 0.1 & 200 & $[r_{\min},r_{\max}]$ & projection + tangential velocity \\
\bottomrule
\end{tabular}
}
\end{table}

\paragraph{Symmetry discovery setup.}
For Stage I, we collect transition tuples $(s,a,r,s')$ by rolling out a behavior policy (Appendix~\ref{app:impl_details}).
The discovery module fits a generator field (or parameterization thereof) by minimizing determining-equation residuals.
We report:
(i) residual statistics $\mathbb{E}\|R_b\|_2^2$, $\mathbb{E}\|R_Q\|_F^2$, $\mathbb{E}|R_r|^2$;
(ii) generator alignment with ground truth (cosine similarity after normalization); and
(iii) value non-invariance $\mathbb{E}_{s,\alpha}|V(s)-V(g_\alpha(s))|$ vs. $|\alpha|$.

\paragraph{Additional plots.}

\begin{figure}[t]
  \centering

  \begin{subfigure}[t]{0.24\textwidth}
    \centering
    \includegraphics[width=\linewidth]{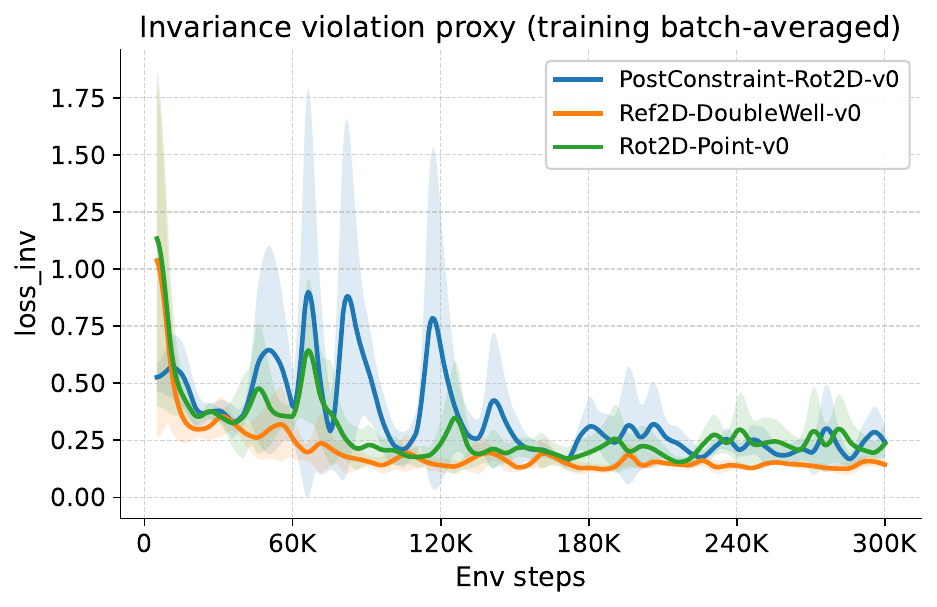}
    \caption{\texttt{}}
    \label{fig:app_loss_inv}
  \end{subfigure}\hfill
  \begin{subfigure}[t]{0.24\textwidth}
    \centering
    \includegraphics[width=\linewidth]{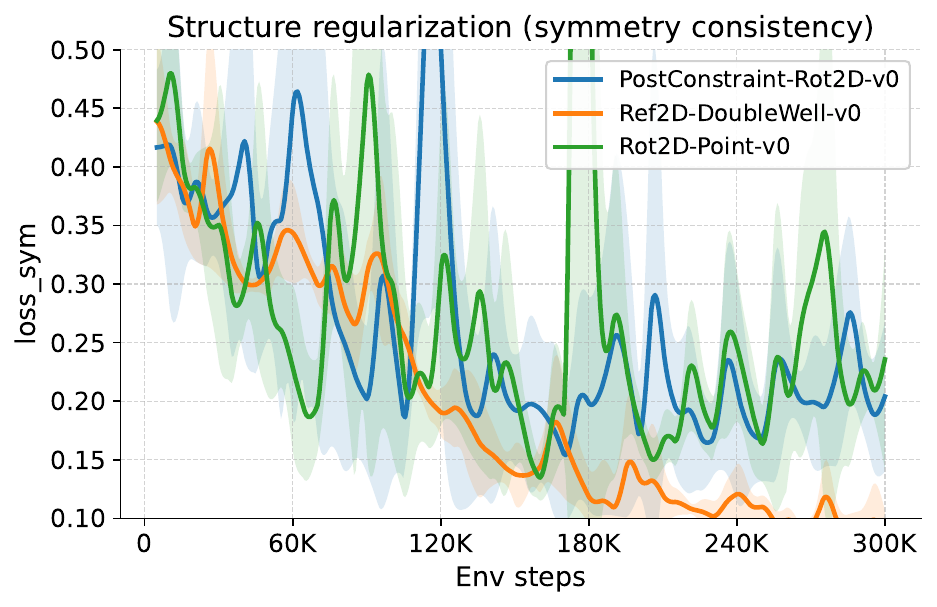}
    \caption{\texttt{}}
    \label{fig:app_loss_sym}
  \end{subfigure}\hfill
  \begin{subfigure}[t]{0.24\textwidth}
    \centering
    \includegraphics[width=\linewidth]{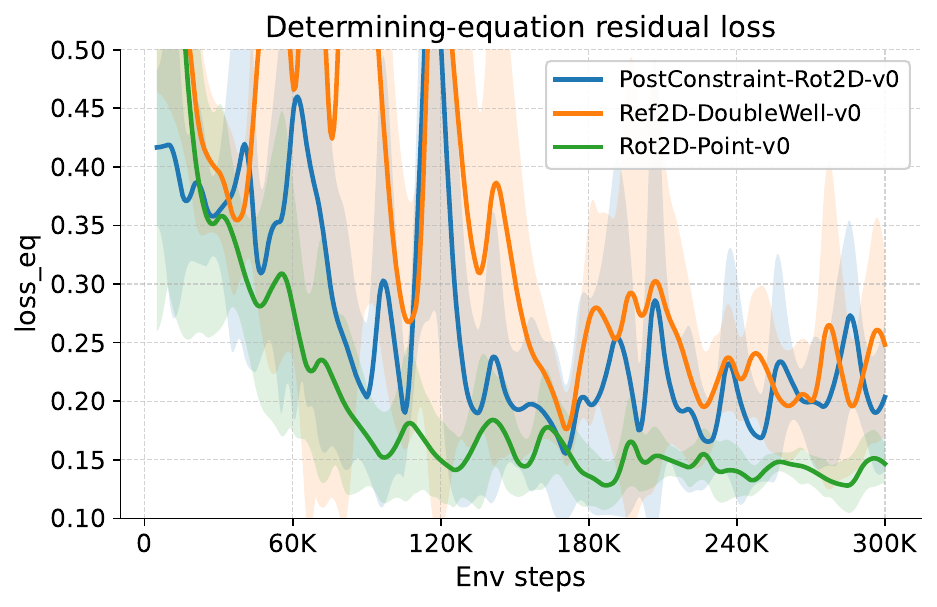}
    \caption{\texttt{}}
    \label{fig:app_loss_eq}
  \end{subfigure}
  \begin{subfigure}[t]{0.24\textwidth}
    \centering
    \includegraphics[width=\linewidth]{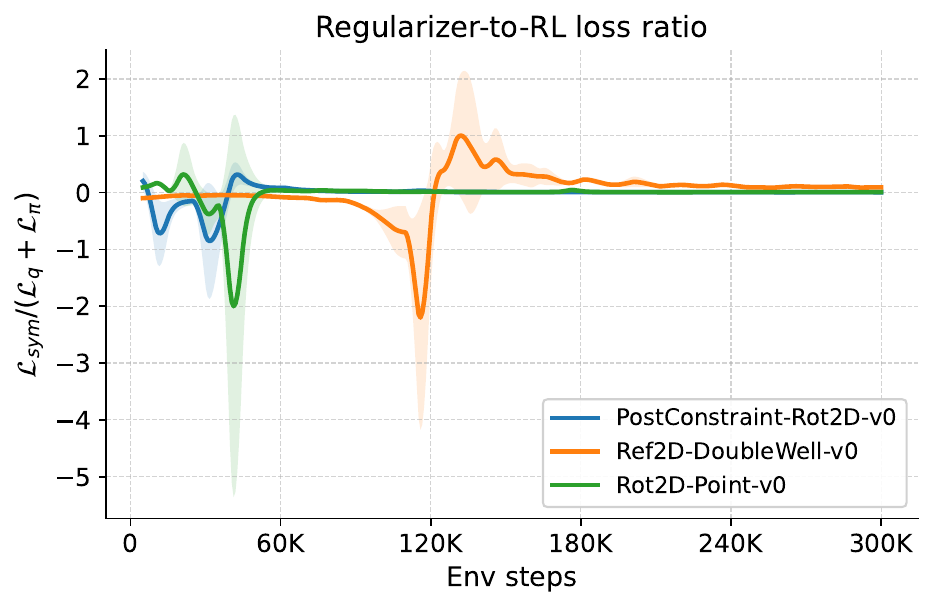}
    \caption{\texttt{}}
    \label{fig:app_loss_ratio_sym}
  \end{subfigure}

  \caption{Fig.~\ref{fig:app_loss_inv} Invariance violation proxy over training (batch-averaged $\mathcal{L}_{\mathrm{inv}}$). Fig.~\ref{fig:app_loss_sym} Symmetry/consistency regularization term over training (batch-averaged $\mathcal{L}_{\mathrm{sym}}$). Fig.~\ref{fig:app_loss_eq} Determining-equation residual term used during Stage I (batch-averaged $\mathcal{L}_{\mathrm{eq}}$).
  Fig.~\ref{fig:app_loss_ratio_sym}
  Regularizer-to-RL loss ratio across tasks, illustrating that regularization does not dominate optimization.
  }
  \label{fig:appendix_loss_diagnostics}
\end{figure}

% \begin{figure}[t]
%   \centering
%   \includegraphics[width=\linewidth]{figures/synth_all_plots.pdf}
%   \caption{Full Stage I results across all synthetic environments.}
%   \label{fig:synth_all}
% \end{figure}

% ------------------------------------------------------------
\subsection{Stage II: SymNav-15 benchmark variants}
\label{app:complex_full}

\label{app:env_symnav}

\paragraph{State, action, and observation.}
\texttt{SymNavEnv} is a 2D navigation environment with obstacles and smooth wind disturbances.
The latent physical state is again $s=(x,y,v_x,v_y)\in\mathbb{R}^4$ and the action is $a=(a_x,a_y)\in[-1,1]^2$.
The \emph{observation} augments state with goal-relative features and lidar rays:
\begin{equation}
\label{eq:symnav_obs}
o = \big[x,y,v_x,v_y,\ (g_x-x),(g_y-y),\ d_1,\ldots,d_8\big]\in\mathbb{R}^{14},
\end{equation}
where $d_i\in[0,1]$ is the normalized distance-to-collision along ray $i$ (8 rays uniformly spaced over $[0,2\pi)$),
computed against both circular obstacles and the world boundary (a disk).

\paragraph{Dynamics.}
The environment implements the same semi-implicit update~\eqref{eq:appendix_em_semiimplicit} with a variant-dependent wind field:
\begin{equation}
\label{eq:symnav_drift}
f(p,v,a) = a + w(p) - \lambda v,
\end{equation}
where $\lambda>0$ is damping and $w(p)$ is a smooth wind:
\begin{equation}
\label{eq:symnav_wind}
w(x,y) = k\big(\sin(0.5\,y),\ \cos(0.5\,x)\big),
\qquad k = 0.15\big((\mathrm{variant}\bmod 5)+1\big).
\end{equation}
Noise is isotropic in the velocity channels, as in~\eqref{eq:appendix_sde_template}.
After each step, position is clamped to remain inside the world disk of radius $R_{\mathrm{world}}$ by radial rescaling.

\paragraph{Map generation (variants).}
Each variant is identified by an integer $\mathrm{variant}\in\{1,\dots,15\}$.
The goal is placed on a circle of radius $r_g$ with angle determined by the variant index:
\begin{equation}
\label{eq:symnav_goal}
g = r_g(\cos\theta,\sin\theta),\qquad \theta = 2\pi\frac{(\mathrm{variant}\bmod 15)}{15},\qquad r_g=3.5.
\end{equation}
Obstacles are 6 circles with centers and radii deterministically generated from $\mathrm{variant}$ (matching the implementation):
for $k\in\{0,\dots,5\}$,
\begin{equation}
\label{eq:symnav_obstacles}
\theta_k = \frac{2\pi k}{6} + 0.15(\mathrm{variant}\bmod 5),\quad
\rho_k = 1.2 + 0.2\big((\mathrm{variant}+k)\bmod 3\big),
\end{equation}
\begin{equation}
\label{eq:symnav_obstacles_center_radius}
c_k = \rho_k(\cos\theta_k,\sin\theta_k),\qquad
r_k = 0.35 + 0.05\big((\mathrm{variant}+2k)\bmod 4\big).
\end{equation}
This generator yields 15 related maps with shared observation/action interfaces but diverse local geometry.

\paragraph{Reward, termination, and diagnostics.}
Let $d(p,g)=\|p-g\|_2$ be the goal distance.
The step reward is dense shaping minus collision penalty plus a terminal success bonus:
\begin{equation}
\label{eq:symnav_reward}
r = -d(p,g) - 2\cdot \mathbb{I}\{\text{collision}\} + 10\cdot \mathbb{I}\{d(p,g)<\varepsilon\},
\qquad \varepsilon=0.3.
\end{equation}
Episodes terminate on success ($d(p,g)<\varepsilon$) and truncate at horizon $T=400$.
The environment returns diagnostic flags \texttt{info["success"]}, \texttt{info["collided"]}, and \texttt{info["dist"]}.

\paragraph{Approximate transformations.}
While the exact rotational symmetry is broken by variant-specific obstacles, goal placement, and wind scaling,
the construction is designed to admit \emph{approximate} continuous transformations: rotating $(p,v,a)$ and rotating the map/goal
yields near-matching local dynamics in many regions, with mismatch localized near obstacles and due to non-invariant wind terms.
This is precisely the regime where operator mismatch bounds and value non-invariance diagnostics are most informative.

\paragraph{Variant list and full learning curves.}
We evaluate variants $\{\mathrm{V}1,\dots,\mathrm{V}15\}$ corresponding to \texttt{variant=1..15}.
Full learning curves for all variants and all methods are in Fig.~\ref{fig:complex_all_15}.
(We optionally include a compact table of map/goal/wind parameters per variant in the released appendix PDF.)

\begin{figure*}[t]
  \centering

  \begin{subfigure}[t]{0.32\textwidth}
    \centering
    \includegraphics[width=\linewidth]{figures/SymNav-v1/eval_return_mean.pdf}
    \caption{\texttt{SymNav-V1}}
    \label{fig:symnav1}
  \end{subfigure}\hfill
  \begin{subfigure}[t]{0.32\textwidth}
    \centering
    \includegraphics[width=\linewidth]{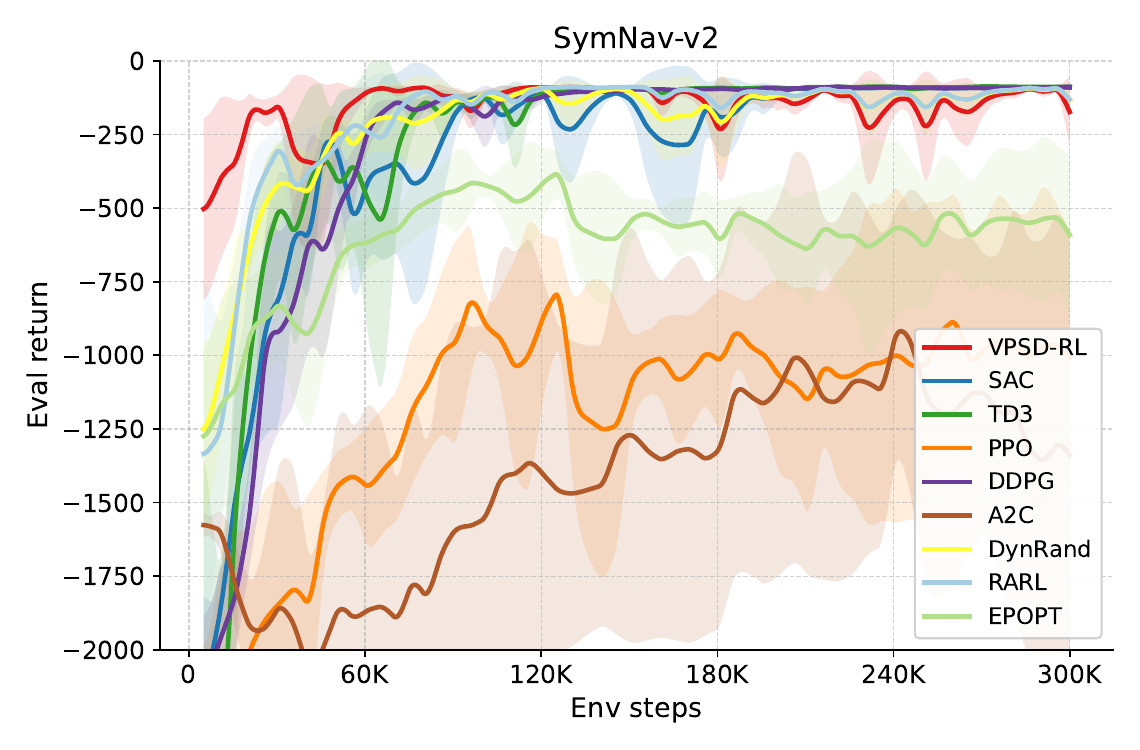}
    \caption{\texttt{SymNav-V2}}
    \label{fig:symnav2}
  \end{subfigure}\hfill
  \begin{subfigure}[t]{0.32\textwidth}
    \centering
    \includegraphics[width=\linewidth]{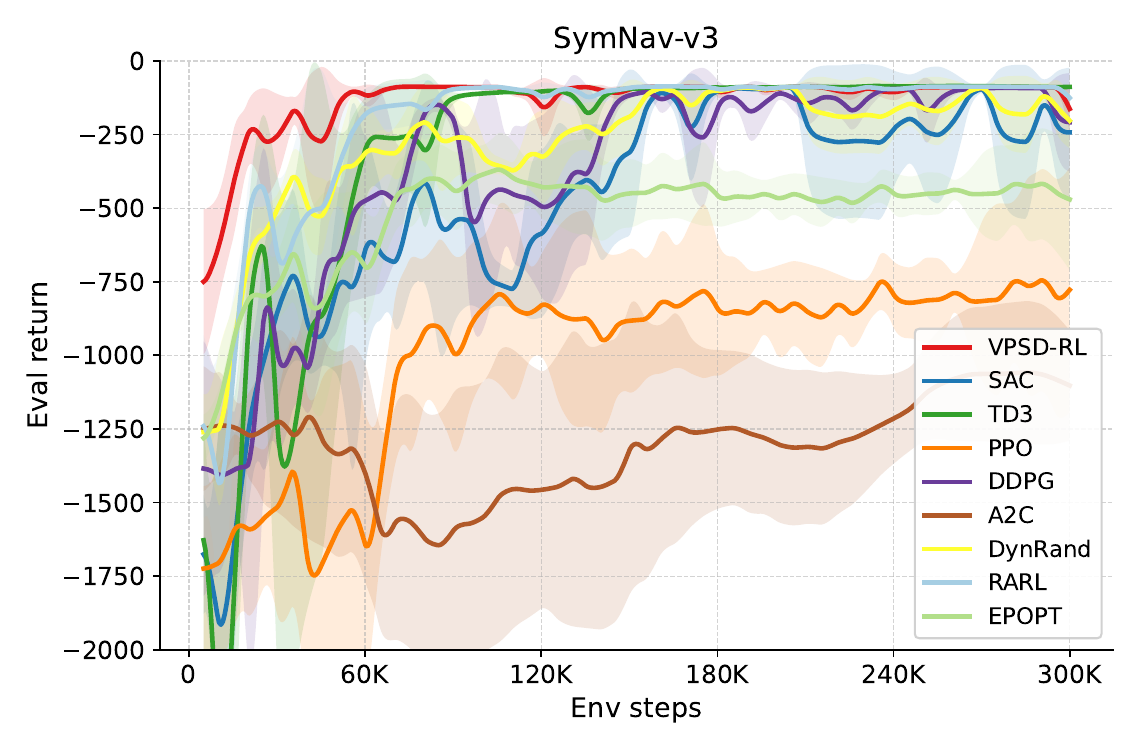}
    \caption{\texttt{SymNav-V3}}
    \label{fig:symnav3}
  \end{subfigure}

  \begin{subfigure}[t]{0.32\textwidth}
    \centering
    \includegraphics[width=\linewidth]{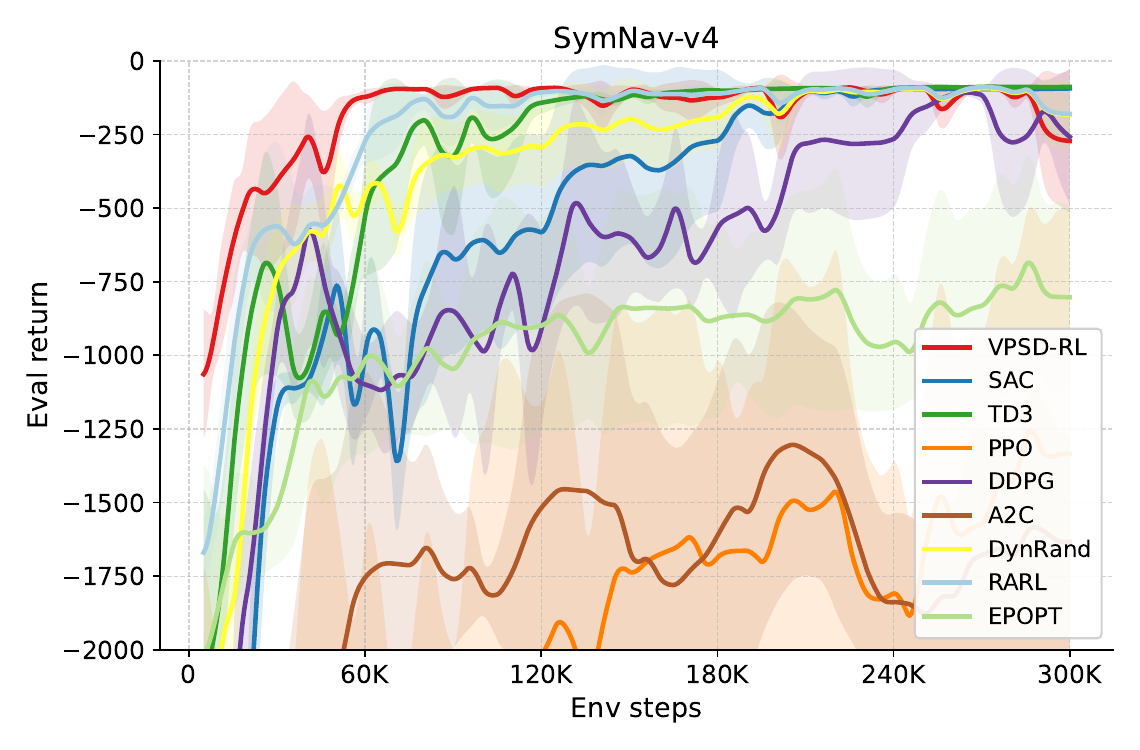}
    \caption{\texttt{SymNav-V4}}
    \label{fig:symnav4}
  \end{subfigure}\hfill
  \begin{subfigure}[t]{0.32\textwidth}
    \centering
    \includegraphics[width=\linewidth]{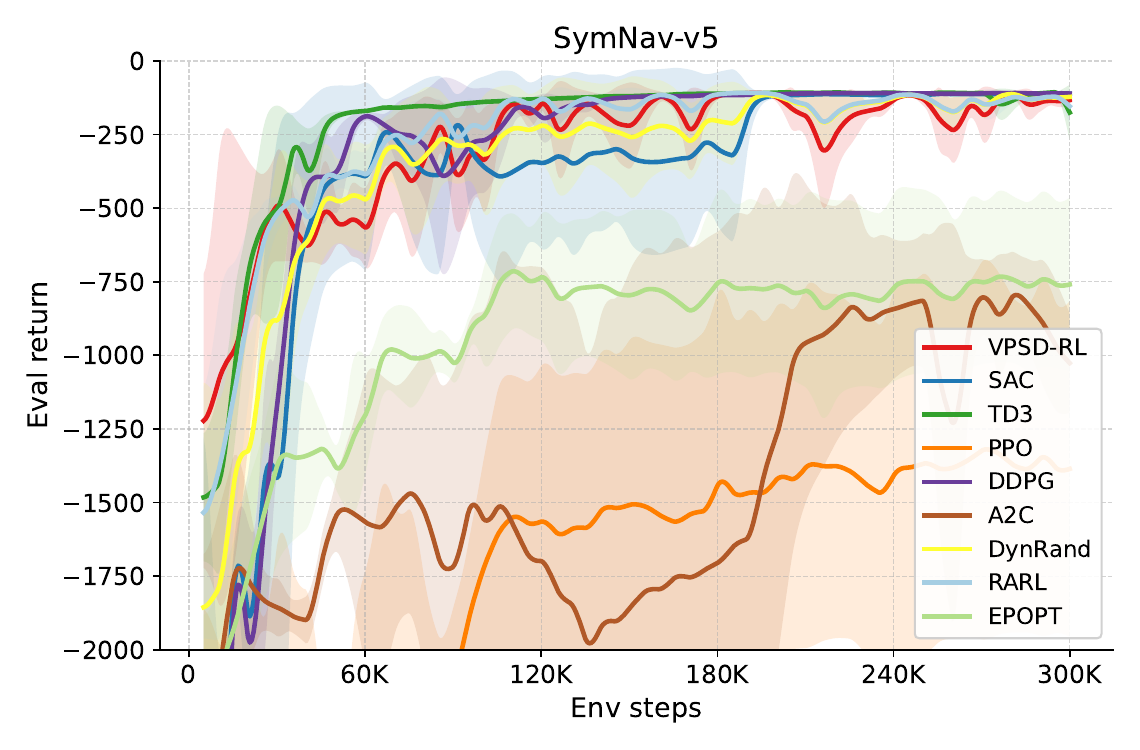}
    \caption{\texttt{SymNav-V5}}
    \label{fig:symnav5}
  \end{subfigure}\hfill
  \begin{subfigure}[t]{0.32\textwidth}
    \centering
    \includegraphics[width=\linewidth]{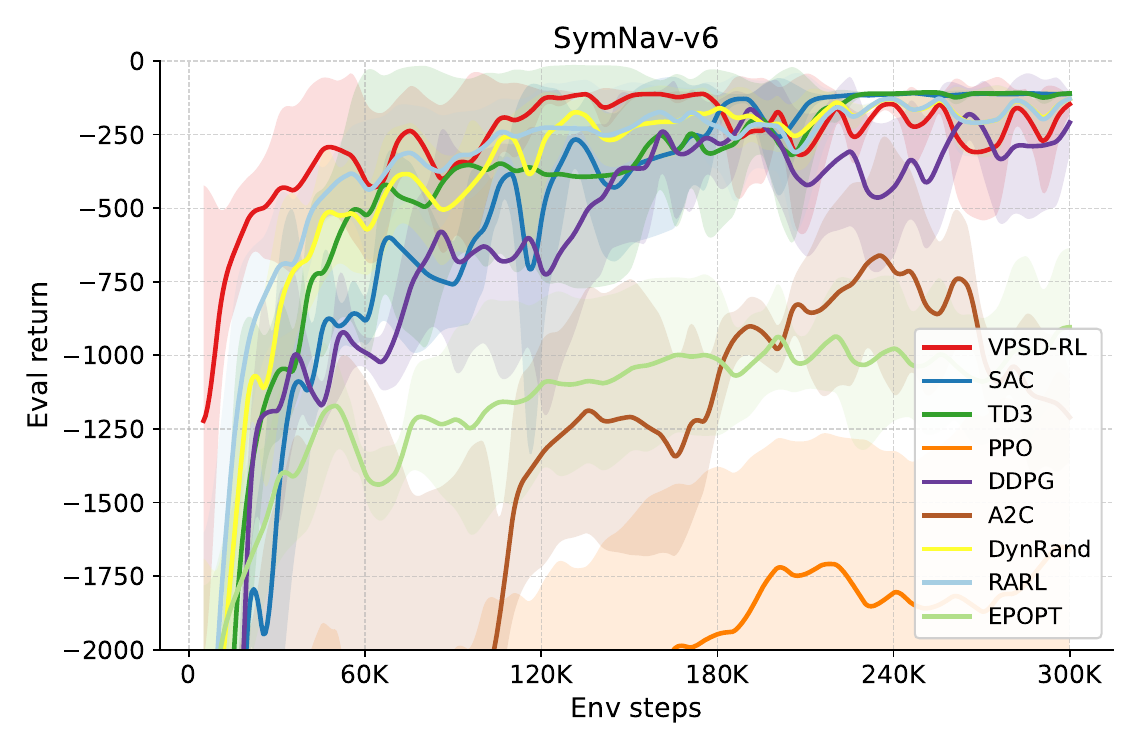}
    \caption{\texttt{SymNav-V6}}
    \label{fig:symnav6}
  \end{subfigure}

  \begin{subfigure}[t]{0.32\textwidth}
    \centering
    \includegraphics[width=\linewidth]{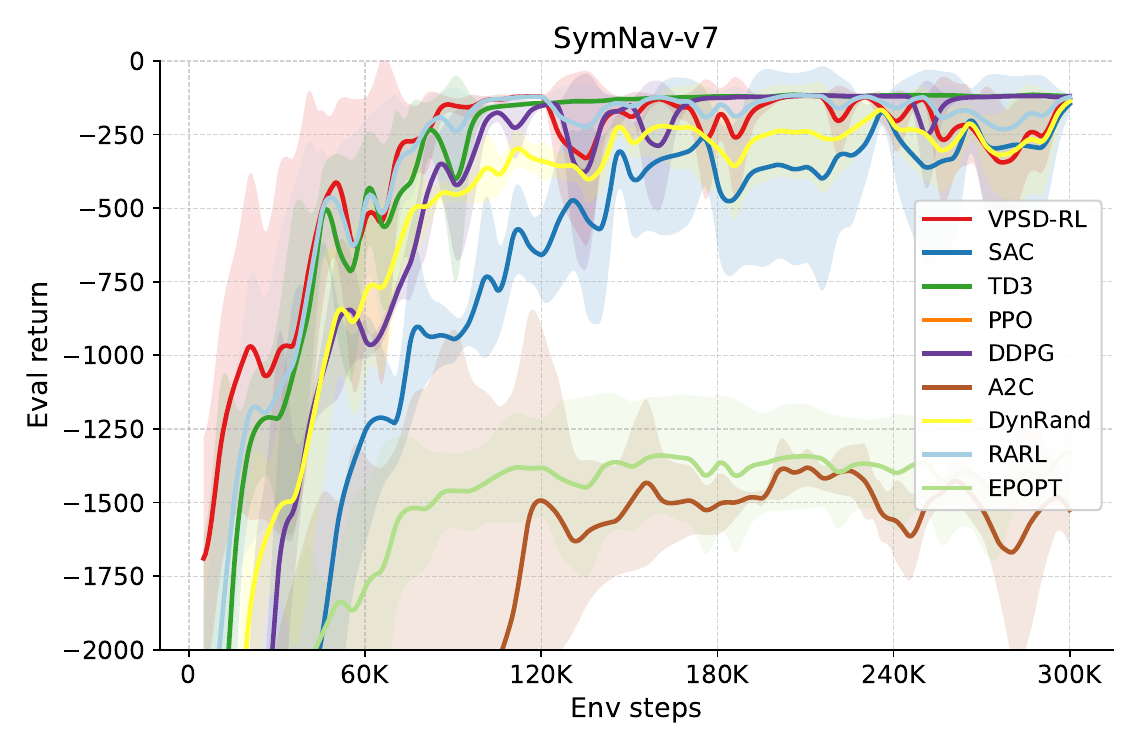}
    \caption{\texttt{SymNav-V7}}
    \label{fig:symnav7}
  \end{subfigure}\hfill
  \begin{subfigure}[t]{0.32\textwidth}
    \centering
    \includegraphics[width=\linewidth]{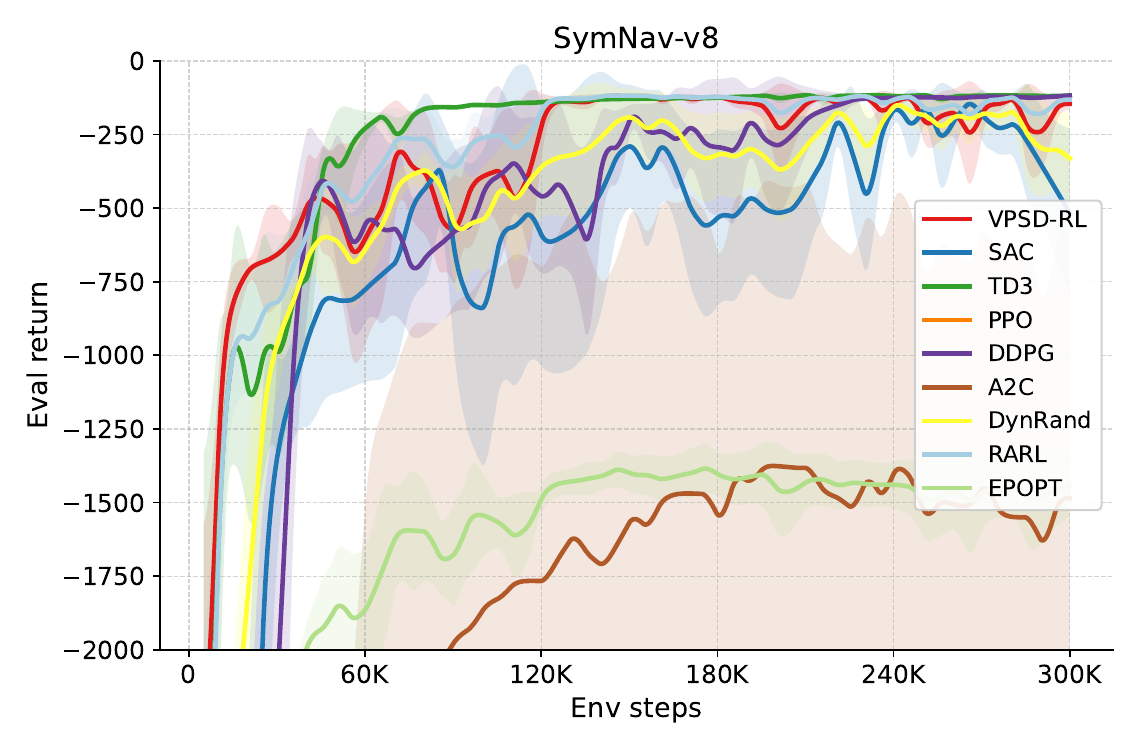}
    \caption{\texttt{SymNav-V8}}
    \label{fig:symnav8}
  \end{subfigure}\hfill
  \begin{subfigure}[t]{0.32\textwidth}
    \centering
    \includegraphics[width=\linewidth]{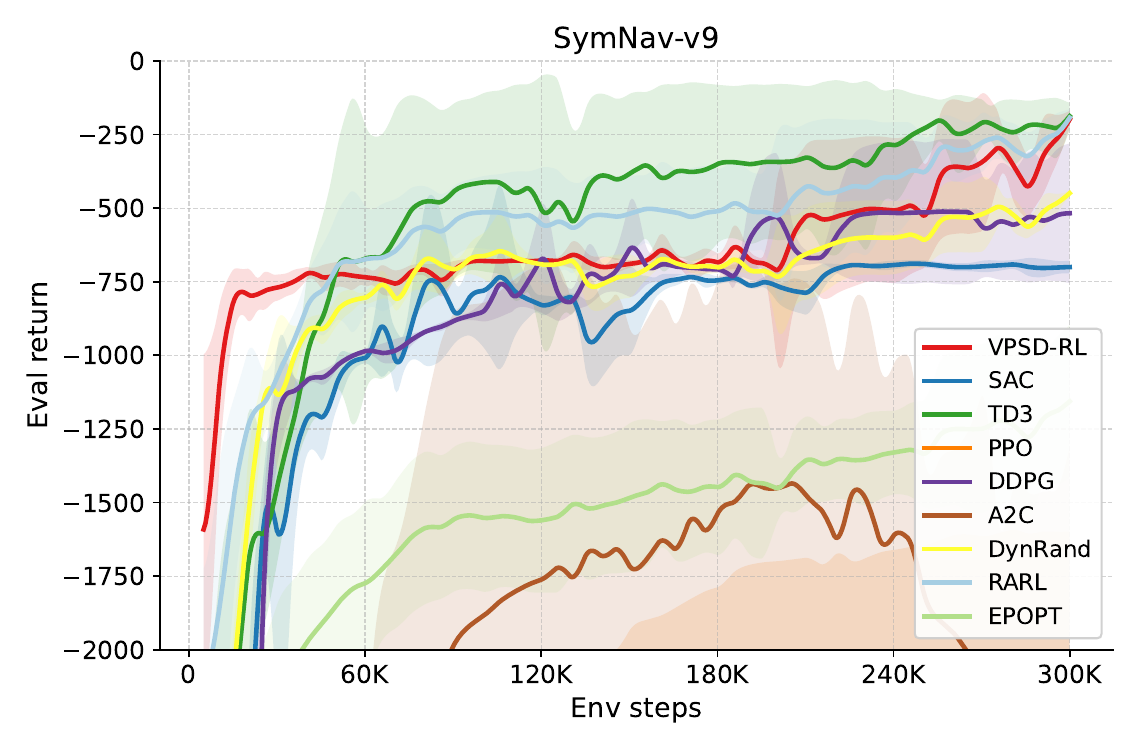}
    \caption{\texttt{SymNav-V9}}
    \label{fig:symnav9}
  \end{subfigure}

  \begin{subfigure}[t]{0.32\textwidth}
    \centering
    \includegraphics[width=\linewidth]{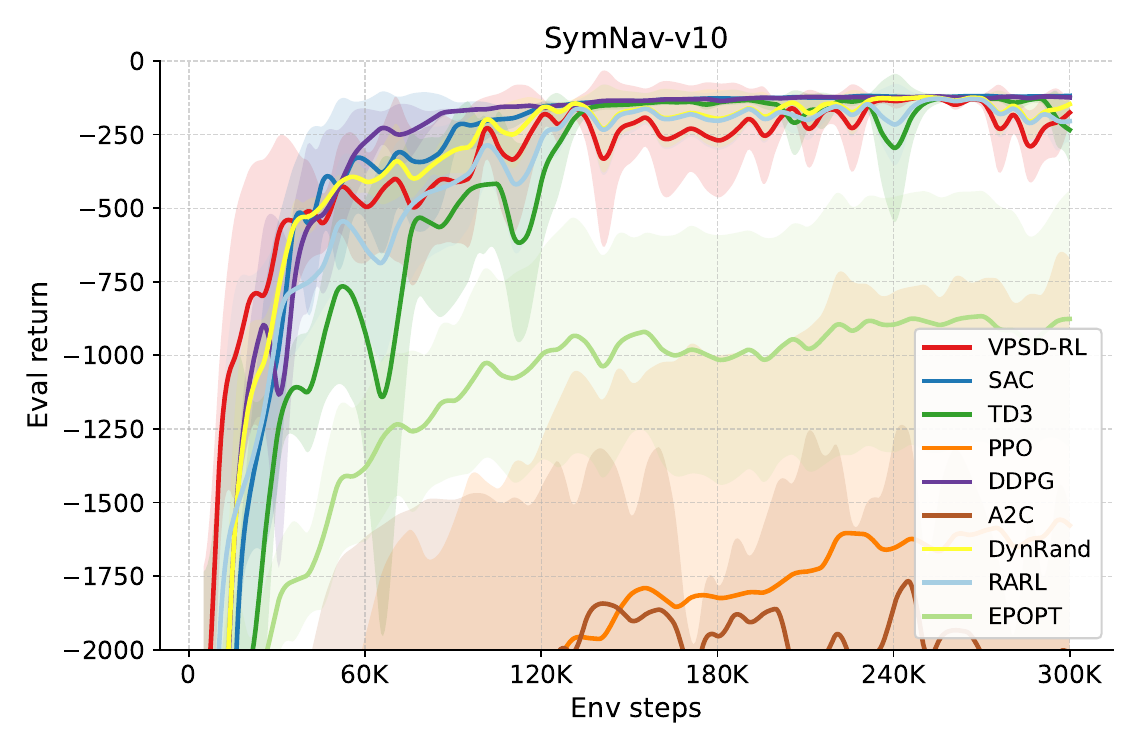}
    \caption{\texttt{SymNav-V10}}
    \label{fig:symnav10}
  \end{subfigure}\hfill
  \begin{subfigure}[t]{0.32\textwidth}
    \centering
    \includegraphics[width=\linewidth]{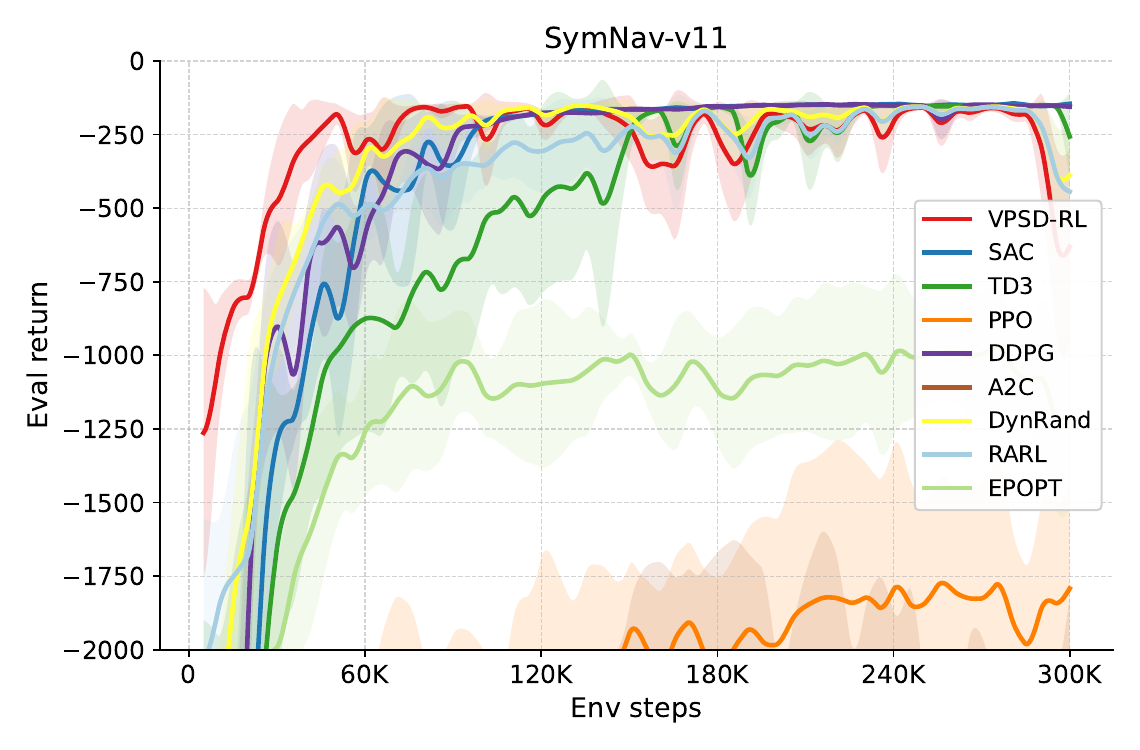}
    \caption{\texttt{SymNav-V11}}
    \label{fig:symnav11}
  \end{subfigure}\hfill
  \begin{subfigure}[t]{0.32\textwidth}
    \centering
    \includegraphics[width=\linewidth]{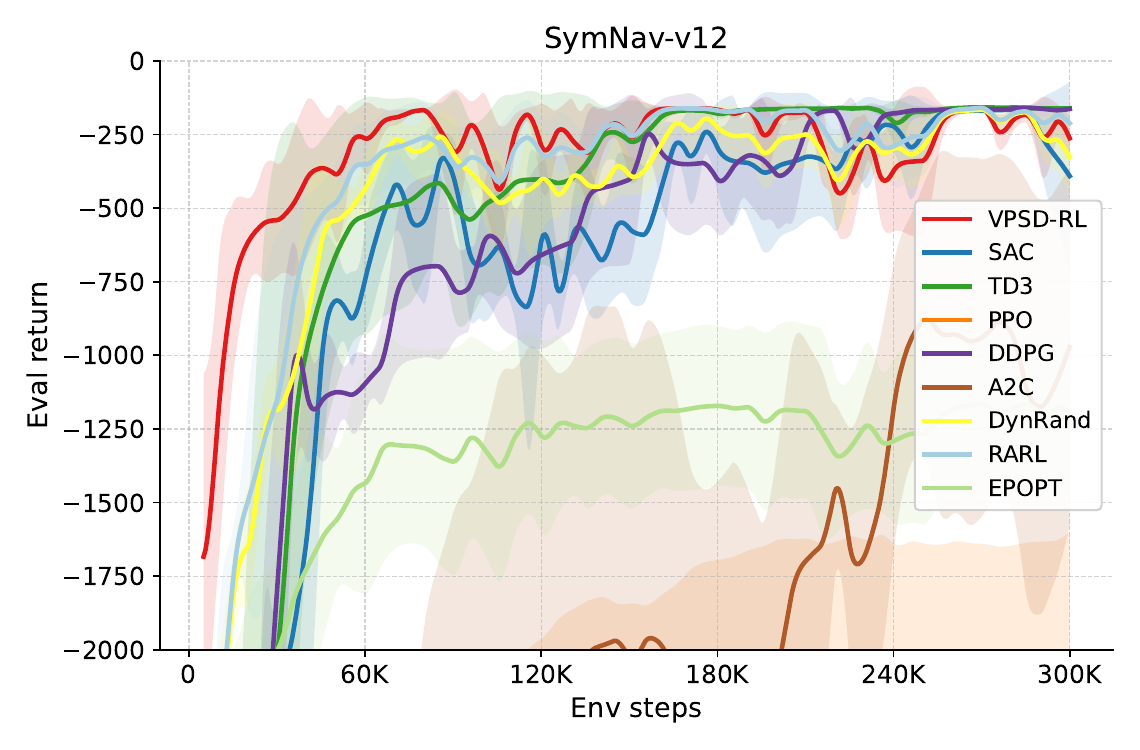}
    \caption{\texttt{SymNav-V12}}
    \label{fig:symnav12}
  \end{subfigure}

  \begin{subfigure}[t]{0.32\textwidth}
    \centering
    \includegraphics[width=\linewidth]{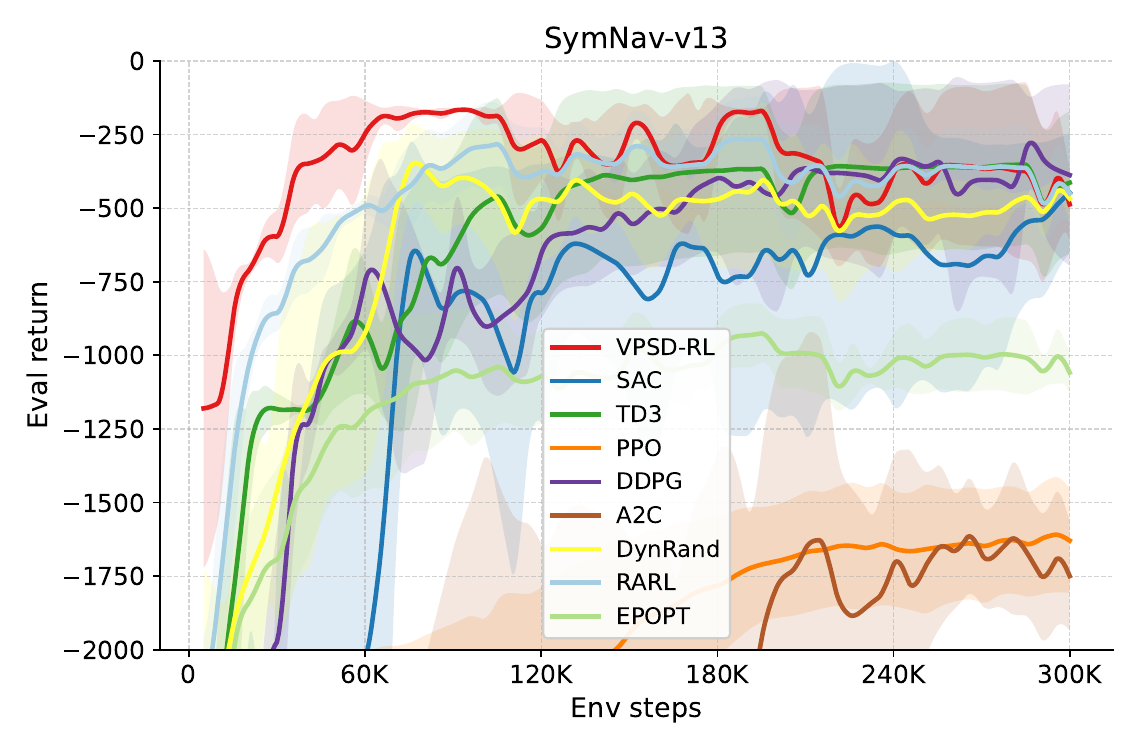}
    \caption{\texttt{SymNav-V13}}
    \label{fig:symnav13}
  \end{subfigure}\hfill
  \begin{subfigure}[t]{0.32\textwidth}
    \centering
    \includegraphics[width=\linewidth]{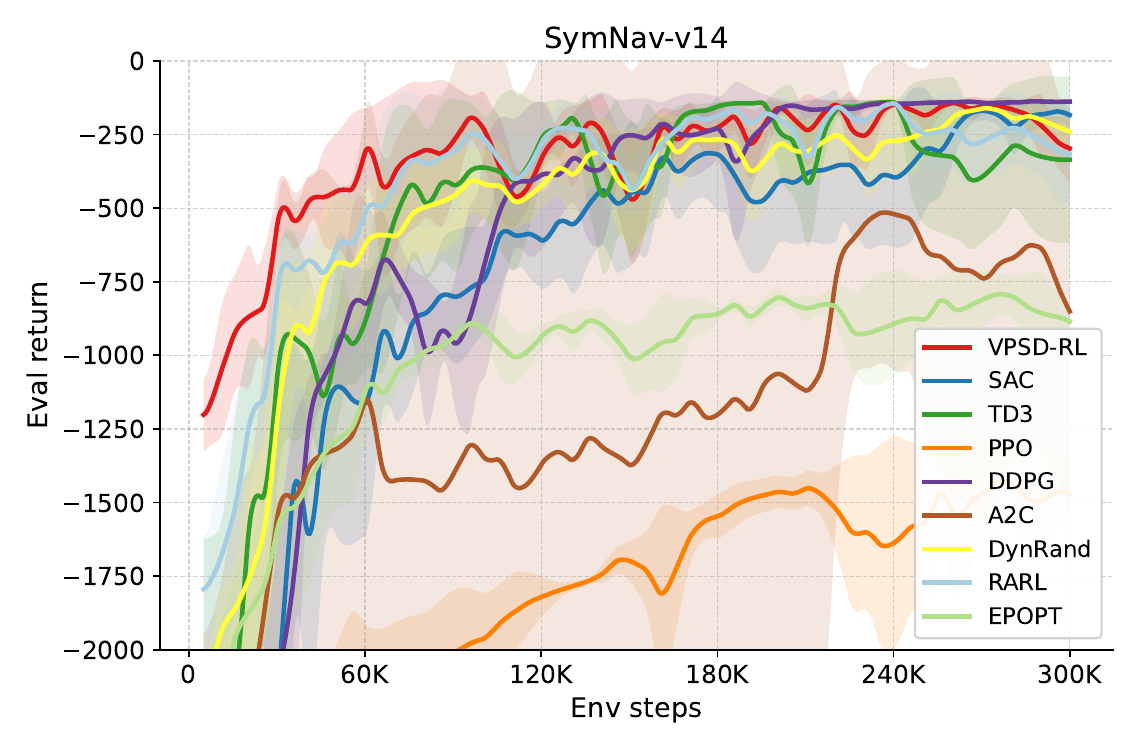}
    \caption{\texttt{SymNav-V14}}
    \label{fig:symnav14}
  \end{subfigure}\hfill
  \begin{subfigure}[t]{0.32\textwidth}
    \centering
    \includegraphics[width=\linewidth]{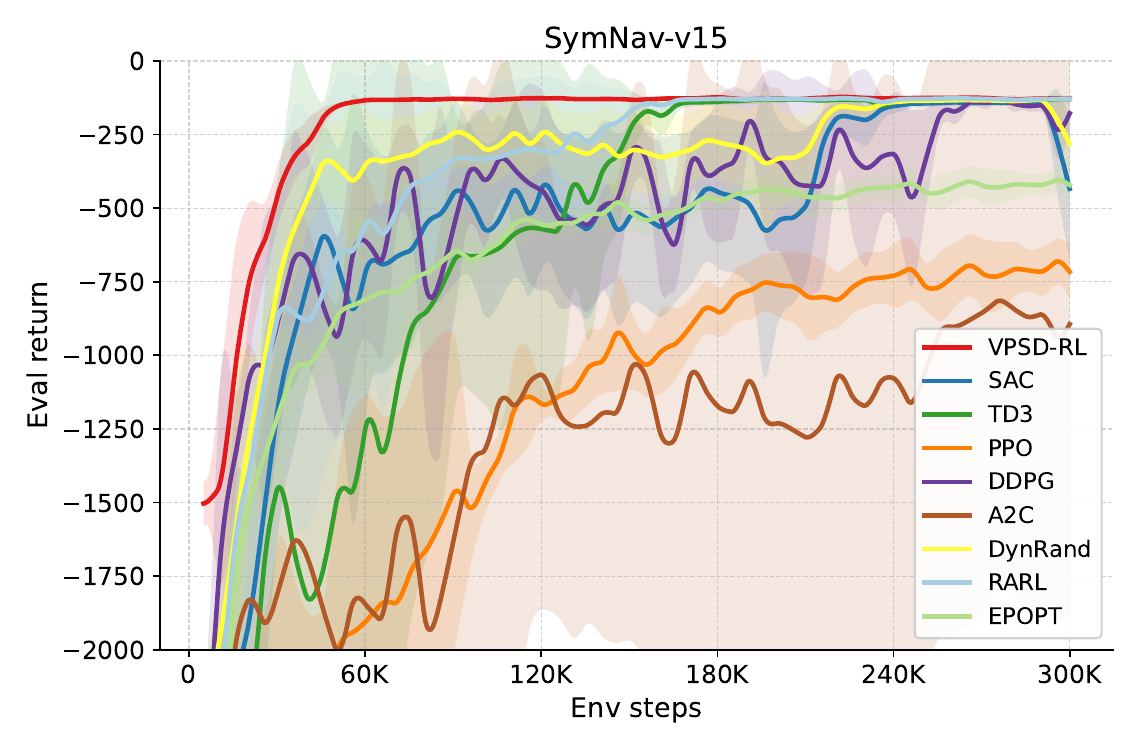}
    \caption{\texttt{SymNav-V15}}
    \label{fig:symnav15}
  \end{subfigure}

  \caption{
  Full learning curves on all SymNav-15 variants for all compared methods.
  }
  \label{fig:complex_all_15}
\end{figure*}

% \begin{figure}[t]
%   \centering
%   \includegraphics[width=\linewidth]{figures/complex_learning_curves_all_15.pdf}
%   \caption{Full learning curves on all SymNav-15 variants for all compared methods.}
%   \label{fig:complex_all_15}
% \end{figure}

\begin{table}[t]
\centering
\caption{
Aggregate results on \emph{SymNav-V1}.
Numbers are mean $\pm$ std over seeds, averaged across all variants.
AUC is computed up to a fixed step budget.
Wall-clock reports seconds per $10^5$ environment steps.
% (hardware/config in Appendix~\ref{app:compute}).
}
\label{tab:main_results_symnav}
\resizebox{0.85\linewidth}{!}{
\begin{tabular}{lccc}
\toprule
Method & Avg. normalized return $\uparrow$ & AUC $\uparrow$ & Wall-clock (sec / $10^5$ steps) $\downarrow$ \\
\midrule
A2C                 &\texttt{-1876.58}$\pm$\texttt{321.57} & \texttt{-1952.71}$\pm$\texttt{217.57} & \texttt{175}$\pm$\texttt{1} \\
DDPG                & \texttt{-106.97}$\pm$\texttt{2.36}   & \texttt{-206.47}$\pm$\texttt{39.70}	 & \texttt{540}$\pm$\texttt{7}
 \\
PPO                 & \texttt{-270.43}$\pm$\texttt{136.65}	 & \texttt{-424.68}$\pm$\texttt{140.79}	 & \texttt{217}$\pm$\texttt{35}
 \\
SAC                 & \texttt{-106.60}$\pm$\texttt{1.97}	 & \texttt{-233.54}$\pm$\texttt{43.70}	 & \texttt{410}$\pm$\texttt{12}
 \\
TD3                 & \texttt{-104.58}$\pm$\texttt{0.41}  & \texttt{-172.08}$\pm$\texttt{25.35}	 & \texttt{754}$\pm$\texttt{10}
 \\
\midrule
DynRand             & \texttt{-115.54}$\pm$\texttt{7.25}	 & \texttt{-181.49}$\pm$\texttt{18.78}	 & \texttt{581}$\pm$\texttt{12}
 \\
RARL                & \texttt{-114.53}$\pm$\texttt{6.48}	 & \texttt{-155.27}$\pm$\texttt{17.23}	 & \texttt{753}$\pm$\texttt{12}
 \\
EPOpt               & \texttt{-197.45}$\pm$\texttt{67.57}	 & \texttt{-271.25}$\pm$\texttt{78.60}	 & \texttt{485}$\pm$\texttt{27}
 \\
\midrule
VPSD-RL (Aug+Reg)   & \texttt{-124.48}$\pm$\texttt{12.65}	 & \texttt{-145.78}$\pm$\texttt{24.65}	 & \texttt{752}$\pm$\texttt{20} \\
\bottomrule
\end{tabular}
}
\end{table}

\paragraph{Training details.}
All methods use the same observation preprocessing, episode truncation, and evaluation protocol.
Per-method hyperparameters (optimizer, learning rate, batch size, replay settings for off-policy methods, etc.)
and the exact seed list are provided in Appendix~\ref{app:impl_details} and the accompanying configuration files.

% ------------------------------------------------------------
\subsection{Implementation details: model fitting and flow integration}
\label{app:impl_details}
\label{app:env_mujoco}

\paragraph{Data collection for discovery.}
Discovery operates on transition tuples $(s,a,r,s')$ collected online during RL training.
For on-policy methods, we use recent rollouts; for off-policy methods, we sample from the replay buffer.
To stabilize the determining-residual optimization, we (i) balance minibatches across the current state distribution,
and (ii) normalize state components (position/velocity scales) using running statistics.

\paragraph{Dynamics and reward modeling.}
For synthetic environments, ground-truth drift/diffusion/reward are known analytically via
Eqs.~\eqref{eq:rot2d_drift}--\eqref{eq:post_reward}, \eqref{eq:dw_drift}--\eqref{eq:dw_reward},
and are used to compute ground-truth generators and diagnostic alignment.
For complex environments (SymNav and MuJoCo), the discovery module fits local surrogate models
for drift $b(s,a)$ and reward $r(s,a)$ by regression on one-step transitions:
\[
\widehat{b}(s,a)\approx \frac{s'-s}{\Delta t},\qquad \widehat{r}(s,a)\approx r.
\]
If a diffusion model is used, we parameterize $Q(s,a)$ through a PSD form (e.g., diagonal or Cholesky factor)
to ensure positive semidefiniteness. All exact architectural choices and optimizer settings are provided in the config files.

\paragraph{Flow integration and constraint handling.}
Given a learned generator $X(\cdot)$ (and $Y(\cdot)$ when needed), we produce finite transformations by integrating the ODE
\[
\frac{\mathrm{d}}{\mathrm{d}\alpha} z(\alpha) = X(z(\alpha)),\qquad z(0)=z_0,
\]
(or its coupled version in the method section), using a fixed-step explicit solver (RK4 by default).
We integrate over a symmetric range $\alpha\in[-\alpha_{\max},\alpha_{\max}]$ with step size $h$.
For constrained environments (PostConstraint-Rot2D), we apply the same projection/retraction policy as the environment:
position is projected to the feasible annulus and the radial velocity component is removed when projection occurs
(Eqs.~\eqref{eq:annulus_projection}--\eqref{eq:annulus_radial_vel_remove}).
We reject transformations that cause numerical instability (NaNs/Infs) and fall back to identity augmentation in that case.

\section{Limitations and broader impact.}
VPSD-RL relies on differentiable surrogate drift, diffusion, and reward estimates, as well as locally valid Lie-group flows. Its performance may degrade when the learned generator is poorly identified, when the dynamics are strongly discontinuous, or when the value-preserving relation holds only on a narrow data support. The current evaluation is limited to controlled simulation benchmarks, so deployment in safety-critical robotic or physical systems requires additional validation, constraint checking, and monitoring under distribution shift. The expected positive impact is improved data efficiency and robustness in continuous-control RL; a possible negative impact is over-trusting learned transformations in high-stakes autonomous systems when the approximate-invariance assumptions fail.

\section{Proof}
\subsection{Proof of Theorem~\ref{thm:hjb_unique}}
\begin{proof}
Fix $\beta>0$ and consider the controlled diffusion
$ds_t=b(s_t,a_t)\,dt+\Sigma(s_t,a_t)\,dW_t$ with $a_t\in\mathcal A$ admissible.
For any Markov policy $\pi$, the discounted value is
$V^\pi(s)=\E^\pi\!\left[\int_0^\infty e^{-\beta t}r(s_t,a_t)\,dt\right]$ and
$V^\star(s)=\sup_\pi V^\pi(s)$.
Since $|r|\le R_{\max}$, for every $\pi$ and every $s$,
\[
|V^\pi(s)|
\le \E^\pi\!\left[\int_0^\infty e^{-\beta t}|r(s_t,a_t)|\,dt\right]
\le \int_0^\infty e^{-\beta t}R_{\max}\,dt
=\frac{R_{\max}}{\beta},
\]
hence $V^\star$ is well-defined and bounded.

Under Assumption~\ref{ass:hjb}(i)--(iii), the SDE is well-posed for admissible controls and the
strong Markov property together with control concatenation yields the discounted dynamic programming principle:
for any bounded stopping time $\tau$,
\begin{equation}
\label{eq:dpp_stop}
V^\star(s)=\sup_{\pi}\E^\pi\!\Bigl[\int_0^\tau e^{-\beta t}r(s_t,a_t)\,dt
+e^{-\beta \tau}V^\star(s_\tau)\,\Big|\,s_0=s\Bigr].
\end{equation}

We now verify that $V^\star$ is a bounded viscosity solution of the HJB equation
$\beta V=(\mathcal HV)$, where for $\varphi\in C^2(\mathcal S)$,
\[
(\mathcal L_a\varphi)(s):=b(s,a)\cdot\nabla\varphi(s)+\frac12\tr\!\bigl(Q(s,a)\nabla^2\varphi(s)\bigr),
\qquad
(\mathcal H\varphi)(s):=\sup_{a\in\mathcal A}\{r(s,a)+(\mathcal L_a\varphi)(s)\}.
\]

Let $\varphi\in C^2(\mathcal S)$ and suppose $V^\star-\varphi$ attains a local maximum at $s_0$ with
$V^\star(s_0)=\varphi(s_0)$. Then there exists $\rho>0$ such that
$V^\star(s)\le \varphi(s)$ for all $s\in\overline{B_\rho(s_0)}$.
Let $\tau_\rho:=\inf\{t\ge 0:\ s_t\notin B_\rho(s_0)\}$ and fix $t>0$; set $\tau:=t\wedge\tau_\rho$.
By \eqref{eq:dpp_stop} and the definition of supremum, for any $\varepsilon>0$ there exists an admissible policy
$\pi^\varepsilon$ such that
\[
V^\star(s_0)\le \E^{\pi^\varepsilon}\!\left[\int_0^\tau e^{-\beta u}r(s_u,a_u)\,du
+e^{-\beta \tau}V^\star(s_\tau)\right]+\varepsilon t.
\]
Since $s_\tau\in\overline{B_\rho(s_0)}$, we have $V^\star(s_\tau)\le \varphi(s_\tau)$, hence using
$V^\star(s_0)=\varphi(s_0)$,
\begin{equation}
\label{eq:sub_0}
0\le \E^{\pi^\varepsilon}\!\left[\int_0^\tau e^{-\beta u}r(s_u,a_u)\,du
+e^{-\beta \tau}\varphi(s_\tau)-\varphi(s_0)\right]+\varepsilon t.
\end{equation}
Apply It\^{o}'s formula to $e^{-\beta u}\varphi(s_u)$ under $\pi^\varepsilon$.
Using $ds_u=b(s_u,a_u)\,du+\Sigma(s_u,a_u)\,dW_u$ and $ds_u ds_u^\top=Q(s_u,a_u)\,du$, we obtain
\[
d\varphi(s_u)=\nabla\varphi(s_u)\cdot b(s_u,a_u)\,du
+\nabla\varphi(s_u)\cdot \Sigma(s_u,a_u)\,dW_u
+\frac12\tr\!\bigl(Q(s_u,a_u)\nabla^2\varphi(s_u)\bigr)\,du,
\]
and therefore
\[
d\bigl(e^{-\beta u}\varphi(s_u)\bigr)
=e^{-\beta u}\bigl((\mathcal L_{a_u}\varphi)(s_u)-\beta\varphi(s_u)\bigr)\,du
+e^{-\beta u}\nabla\varphi(s_u)\cdot\Sigma(s_u,a_u)\,dW_u.
\]
Integrating from $0$ to $\tau$ gives the identity
\[
e^{-\beta\tau}\varphi(s_\tau)-\varphi(s_0)
=\int_0^\tau e^{-\beta u}\bigl((\mathcal L_{a_u}\varphi)(s_u)-\beta\varphi(s_u)\bigr)\,du
+\int_0^\tau e^{-\beta u}\nabla\varphi(s_u)\cdot\Sigma(s_u,a_u)\,dW_u.
\]
The stochastic integral is a martingale with zero expectation (bounded coefficients and $\tau\le t<\infty$), hence
\[
\E^{\pi^\varepsilon}\!\left[e^{-\beta\tau}\varphi(s_\tau)-\varphi(s_0)\right]
=\E^{\pi^\varepsilon}\!\left[\int_0^\tau e^{-\beta u}\bigl((\mathcal L_{a_u}\varphi)(s_u)-\beta\varphi(s_u)\bigr)\,du\right].
\]
Substituting into \eqref{eq:sub_0} yields
\[
0\le \E^{\pi^\varepsilon}\!\left[\int_0^\tau e^{-\beta u}\Bigl(r(s_u,a_u)+(\mathcal L_{a_u}\varphi)(s_u)-\beta\varphi(s_u)\Bigr)\,du\right]+\varepsilon t.
\]
Define $\Psi(s):=\sup_{a\in\mathcal A}\{r(s,a)+(\mathcal L_a\varphi)(s)-\beta\varphi(s)\}
=(\mathcal H\varphi)(s)-\beta\varphi(s)$.
Since $r(s_u,a_u)+(\mathcal L_{a_u}\varphi)(s_u)-\beta\varphi(s_u)\le \Psi(s_u)$,
\[
0\le \E^{\pi^\varepsilon}\!\left[\int_0^\tau e^{-\beta u}\Psi(s_u)\,du\right]+\varepsilon t.
\]
Dividing by $t>0$ gives
\[
0\le \frac1t\,\E^{\pi^\varepsilon}\!\left[\int_0^\tau e^{-\beta u}\Psi(s_u)\,du\right]+\varepsilon.
\]
We claim $\frac1t\,\E^{\pi^\varepsilon}\!\left[\int_0^\tau e^{-\beta u}\Psi(s_u)\,du\right]\to \Psi(s_0)$ as $t\downarrow 0$.
Indeed, since $\tau\le t$ and $\Psi$ is bounded on $\overline{B_\rho(s_0)}$, using $|e^{-\beta u}-1|\le \beta u$ we have
\[
\left|\frac1t\,\E\!\left[\int_0^\tau (e^{-\beta u}-1)\Psi(s_u)\,du\right]\right|
\le \frac1t\,\E\!\left[\int_0^\tau \beta u\,\|\Psi\|_\infty\,du\right]
\le \frac{\beta\|\Psi\|_\infty}{2}\,t \to 0.
\]
Moreover, $\Psi$ is uniformly continuous on $\overline{B_\rho(s_0)}$, so for its modulus of continuity $\omega(\cdot)$,
the term $\frac1t\,\E\!\left[\int_0^\tau |\Psi(s_u)-\Psi(s_0)|\,du\right]$ vanishes as $t\downarrow 0$ because
$\sup_{0\le u\le t}\|s_u-s_0\|\to 0$ in probability under bounded coefficients (drift is $O(t)$ and the martingale part is $O(\sqrt t)$ by BDG),
and $\tau=t\wedge\tau_\rho$ satisfies $\E[\tau]/t\to 1$.
Consequently,
\[
\lim_{t\downarrow 0}\frac1t\,\E^{\pi^\varepsilon}\!\left[\int_0^\tau e^{-\beta u}\Psi(s_u)\,du\right]=\Psi(s_0).
\]
Letting $t\downarrow 0$ yields $0\le \Psi(s_0)+\varepsilon$, and then $\varepsilon\downarrow 0$ gives $\Psi(s_0)\ge 0$, i.e.
$\beta\varphi(s_0)\le (\mathcal H\varphi)(s_0)$. This proves that $V^\star$ is a viscosity subsolution.

Now suppose $V^\star-\varphi$ attains a local minimum at $s_0$ with $V^\star(s_0)=\varphi(s_0)$.
Then for some $\rho>0$, $V^\star\ge \varphi$ on $\overline{B_\rho(s_0)}$, and with $\tau=t\wedge\tau_\rho$ we have,
for any fixed $a\in\mathcal A$, by \eqref{eq:dpp_stop} (since $\sup_\pi$ dominates the constant control $a_u\equiv a$),
\[
V^\star(s_0)\ge \E\!\left[\int_0^\tau e^{-\beta u}r(s_u,a)\,du+e^{-\beta \tau}V^\star(s_\tau)\right]
\ge \E\!\left[\int_0^\tau e^{-\beta u}r(s_u,a)\,du+e^{-\beta \tau}\varphi(s_\tau)\right].
\]
Using $V^\star(s_0)=\varphi(s_0)$ and the same It\^{o} identity as above (with generator $\mathcal L_a$) gives
\[
0\ge \E\!\left[\int_0^\tau e^{-\beta u}\Bigl(r(s_u,a)+(\mathcal L_a\varphi)(s_u)-\beta\varphi(s_u)\Bigr)\,du\right].
\]
Divide by $t$ and let $t\downarrow 0$ to obtain
$0\ge r(s_0,a)+(\mathcal L_a\varphi)(s_0)-\beta\varphi(s_0)$ for every $a\in\mathcal A$.
Taking the supremum over $a$ yields $\beta\varphi(s_0)\ge (\mathcal H\varphi)(s_0)$, hence $V^\star$ is a viscosity supersolution.
Therefore $V^\star$ is a bounded viscosity solution of $\beta V=\mathcal HV$, establishing existence of a bounded viscosity solution.

Uniqueness among bounded viscosity solutions follows directly from the comparison principle in Assumption~\ref{ass:hjb}(iv):
if $u$ is a bounded viscosity subsolution and $v$ a bounded viscosity supersolution, then $u\le v$.
In particular, if $V$ is any bounded viscosity solution, then $V$ is both a subsolution and a supersolution, so
$V\le V^\star$ and $V^\star\le V$, implying $V\equiv V^\star$.
\end{proof}

\subsection{Proof of Lemma~\ref{lem:semigroup_equiv}}
\begin{proof}
We fix $a\in\mathcal A$ and write $\tilde a:=h_\vartheta(a)$, $T_t:=P_t^a$, $\tilde T_t:=P_t^{\tilde a}$,
$A:=\mathcal L^a$, $\tilde A:=\mathcal L^{\tilde a}$, and $U:=U_{g_\vartheta}$.

We will use two standard $C_0$-semigroup facts, stated and proved here for completeness.

\begin{lemma}
\label{lem:c0_uniq}
Let $(S_t)_{t\ge 0}$ be a strongly continuous semigroup on a Banach space $X$ with generator $G$.
For every $x\in\Dom(G)$, the map $t\mapsto S_t x$ is norm-differentiable and satisfies
$ \frac{d}{dt}S_t x = G S_t x = S_t G x$ for all $t\ge 0$.
Moreover, if $w:[0,\infty)\to X$ is norm-differentiable, $w(t)\in\Dom(G)$ for all $t$, and
$w'(t)=G w(t)$ with $w(0)=x\in\Dom(G)$, then necessarily $w(t)=S_t x$ for all $t\ge 0$.
\end{lemma}

\begin{proof}
Fix $x\in\Dom(G)$. By definition of $G$, the limit $Gx=\lim_{h\downarrow 0}\frac{S_hx-x}{h}$ exists in $X$.
For $t\ge 0$ and $h>0$, use the semigroup property to compute
\[
\frac{S_{t+h}x-S_tx}{h}
=\frac{S_t(S_hx-x)}{h}
=S_t\!\left(\frac{S_hx-x}{h}\right).
\]
Since $S_t$ is bounded for fixed $t$, letting $h\downarrow 0$ yields the right-derivative
$\lim_{h\downarrow 0}\frac{S_{t+h}x-S_tx}{h}=S_tGx$.
For $0<h<t$, similarly
\[
\frac{S_tx-S_{t-h}x}{h}
=\frac{S_{t-h}(S_hx-x)}{h}
=S_{t-h}\!\left(\frac{S_hx-x}{h}\right),
\]
and letting $h\downarrow 0$ gives the left-derivative equal to $S_tGx$ by strong continuity of $S_{t-h}$ at $t$.
Thus $t\mapsto S_tx$ is norm-differentiable and $\frac{d}{dt}S_tx=S_tGx$.
To identify $G S_t x$, note that for $h>0$,
\[
\frac{S_h(S_t x)-S_t x}{h}
=\frac{S_t(S_hx-x)}{h}
=S_t\!\left(\frac{S_hx-x}{h}\right)\to S_tGx \quad (h\downarrow 0),
\]
so $S_t x\in\Dom(G)$ and $G S_t x = S_t G x$, hence also $\frac{d}{dt}S_t x = G S_t x$.

For uniqueness, let $w$ satisfy the stated assumptions and fix $t>0$. Define $\phi(s):=S_{t-s}w(s)$ for $s\in[0,t]$.
Using the already proved identity $\frac{d}{ds}S_{t-s}y=-G S_{t-s}y$ for $y\in\Dom(G)$ and the product rule in $X$,
we compute
\[
\phi'(s)=-G S_{t-s}w(s)+S_{t-s}w'(s)
=-G S_{t-s}w(s)+S_{t-s}G w(s)=0,
\]
so $\phi$ is constant and $\phi(t)=\phi(0)$ gives
$w(t)=S_t w(0)=S_t x$. Since $t$ was arbitrary, $w(t)=S_t x$ for all $t\ge 0$.
\end{proof}

\begin{lemma}
\label{lem:dom_dense}
If $(S_t)_{t\ge 0}$ is a strongly continuous semigroup on a Banach space $X$ with generator $G$, then $\Dom(G)$ is dense in $X$.
\end{lemma}

\begin{proof}
Let $\lambda>0$ and define the bounded operator (the Yosida approximation)
\[
G_\lambda x := \lambda \int_0^\infty e^{-\lambda t} S_t x \,dt - x,
\qquad x\in X,
\]
where the Bochner integral exists because $t\mapsto S_t x$ is continuous and $\|S_t\|$ is locally bounded.
Define $x_\lambda := \lambda \int_0^\infty e^{-\lambda t} S_t x \,dt$.
Then $x_\lambda\to x$ in $X$ as $\lambda\to\infty$ by strong continuity at $0$ (an Abel-type limit).
Moreover, one checks directly from the generator definition that $x_\lambda\in\Dom(G)$ and
$G x_\lambda = \lambda(x_\lambda-x)$ (equivalently $x_\lambda\in\Ran(\lambda I-G)\subset\Dom(G)$).
Hence every $x\in X$ is a norm-limit of elements in $\Dom(G)$, proving density.
\end{proof}

We now prove the two directions of Lemma~\ref{lem:semigroup_equiv}.

Assume first the generator intertwining holds on the common generator core:
\[
A U f = U\tilde A f,\qquad \forall f\in \Dom(\tilde A).
\]
Fix \(f\in\Dom(\tilde A)\) and define \(u(t):=U\tilde T_t f\) for \(t\ge 0\).
By Lemma~\ref{lem:c0_uniq} applied to \((\tilde T_t)\) and \(\tilde A\), we have
\(\tilde T_t f\in\Dom(\tilde A)\) and
\[
\frac{d}{dt}\tilde T_t f=\tilde A\tilde T_t f.
\]
Since \(U\) is bounded linear,
\[
u'(t)=U\tilde A\tilde T_t f.
\]
Because \(A U=U\tilde A\) on \(\Dom(\tilde A)\), we have
\[
u'(t)=A U\tilde T_t f=A u(t),\qquad u(0)=Uf.
\]
Define \(v(t):=T_tUf\). By Lemma~\ref{lem:c0_uniq} applied to \((T_t)\) and \(A\),
\(v'(t)=Av(t)\) and \(v(0)=Uf\). Hence \(u(t)=v(t)\), namely
\[
U\tilde T_t f=T_tUf,\qquad \forall t\ge0,\ \forall f\in\Dom(\tilde A).
\]
By density of \(\Dom(\tilde A)\) and boundedness of \(U,T_t,\tilde T_t\), this extends to all
\(f\in\mathcal F\):
\[
T_tU=U\tilde T_t.
\]
Equivalently,
\[
P_t^aU_{g_\vartheta}=U_{g_\vartheta}P_t^{h_\vartheta(a)}.
\]

Conversely, assume the semigroup intertwining holds:
\[
T_tU=U\tilde T_t,\qquad t\ge0.
\]
For every \(f\in\Dom(\tilde A)\),
\[
A U f
=
\lim_{t\downarrow0}\frac{T_tUf-Uf}{t}
=
\lim_{t\downarrow0}\frac{U\tilde T_t f-Uf}{t}
=
U\tilde A f.
\]
Therefore \(AU=U\tilde A\), equivalently
\[
\mathcal L^aU_{g_\vartheta}=U_{g_\vartheta}\mathcal L^{h_\vartheta(a)},
\]
which is Eq.~\eqref{eq:gen_equiv} on the generator domain.
\end{proof}

\subsection{Proof of Theorem~\ref{thm:determining_necessity}}
\begin{proof}
Fix $(s,a)\in\mathcal{S}\times\mathcal{A}$ and abbreviate $g_\epsilon:=g_{\vartheta_\epsilon}$, $h_\epsilon:=h_{\vartheta_\epsilon}$.
Reward invariance in Definition~\ref{def:exact_sym} reads
$r(g_\epsilon(s),h_\epsilon(a))=r(s,a)$ for all $\epsilon$.
Differentiating at $\epsilon=0$ and using the chain rule gives
\[
0=\left.\frac{d}{d\epsilon}r(g_\epsilon(s),h_\epsilon(a))\right|_{\epsilon=0}
=\nabla_s r(s,a)\cdot \left.\frac{d}{d\epsilon}g_\epsilon(s)\right|_{\epsilon=0}
+\nabla_a r(s,a)\cdot \left.\frac{d}{d\epsilon}h_\epsilon(a)\right|_{\epsilon=0},
\]
which is exactly~\eqref{eq:det_reward}.

% We now turn to generator equivariance. For each fixed $\epsilon$, Definition~\ref{def:exact_sym} gives that for all $f\in C^2(\mathcal{S})$,
For each fixed \(\epsilon\), Definition~\ref{def:exact_sym} gives that for all \(f\in C^2(\mathcal{S})\),
\begin{equation}
\label{eq:gen_eq_eps}
\bigl(\mathcal{L}^{a}(f\circ g_\epsilon)\bigr)(s)
=
\bigl((\mathcal{L}^{h_\epsilon(a)} f)\circ g_\epsilon\bigr)(s)
=
\bigl(\mathcal{L}^{h_\epsilon(a)} f\bigr)(g_\epsilon(s)).
\end{equation}
Write the controlled generator in coordinates as
\[
(\mathcal{L}^a f)(x)= b(x,a)\cdot\nabla f(x) + \frac12\,\mathrm{tr}\!\bigl(Q(x,a)\,\nabla^2 f(x)\bigr),
\qquad Q(x,a)=\Sigma(x,a)\Sigma(x,a)^\top.
\]
Set $x_\epsilon:=g_\epsilon(s)$ and denote the Jacobian matrix and second derivatives of $g_\epsilon$ by
\[
J_\epsilon(s):=\nabla_s g_\epsilon(s)\in\mathbb{R}^{d\times d},
\qquad
\bigl(\nabla_s^2 g_\epsilon^k(s)\bigr)_{ij}:=\partial_{ij}\bigl(g_\epsilon^k(s)\bigr),\ \ k=1,\dots,d.
\]
For later use we expand the first and second spatial derivatives of $f\circ g_\epsilon$ at $s$.
By the chain rule, for each $i=1,\dots,d$,
\[
\partial_i(f\circ g_\epsilon)(s)
=\sum_{k=1}^d (\partial_k f)(x_\epsilon)\,\partial_i(g_\epsilon^k)(s),
\quad\text{i.e.}\quad
\nabla(f\circ g_\epsilon)(s)=J_\epsilon(s)^\top \nabla f(x_\epsilon).
\]
Differentiating once more and using product/chain rules yields, for each $i,j$,
\[
\partial_{ij}(f\circ g_\epsilon)(s)
=\sum_{k,\ell=1}^d (\partial_{k\ell}f)(x_\epsilon)\,\partial_i(g_\epsilon^k)(s)\,\partial_j(g_\epsilon^\ell)(s)
+\sum_{k=1}^d (\partial_k f)(x_\epsilon)\,\partial_{ij}(g_\epsilon^k)(s).
\]
Equivalently in matrix form,
\[
\nabla^2(f\circ g_\epsilon)(s)=J_\epsilon(s)^\top \nabla^2 f(x_\epsilon) J_\epsilon(s)
+\sum_{k=1}^d \bigl(\partial_k f(x_\epsilon)\bigr)\,\nabla_s^2 g_\epsilon^k(s).
\]
Plugging these into the left-hand side of~\eqref{eq:gen_eq_eps} gives
\begin{align*}
\bigl(\mathcal{L}^{a}(f\circ g_\epsilon)\bigr)(s)
&= b(s,a)\cdot \nabla(f\circ g_\epsilon)(s)
+\frac12\,\mathrm{tr}\!\Bigl(Q(s,a)\,\nabla^2(f\circ g_\epsilon)(s)\Bigr)\\
&= b(s,a)\cdot J_\epsilon(s)^\top \nabla f(x_\epsilon)
+\frac12\,\mathrm{tr}\!\Bigl(Q(s,a)\,J_\epsilon(s)^\top \nabla^2 f(x_\epsilon)J_\epsilon(s)\Bigr)\\
&\qquad
+\frac12\sum_{k=1}^d \bigl(\partial_k f(x_\epsilon)\bigr)\,
\mathrm{tr}\!\Bigl(Q(s,a)\,\nabla_s^2 g_\epsilon^k(s)\Bigr).
\end{align*}
Using cyclicity of trace, define
\[
\widetilde Q_\epsilon(s,a):=J_\epsilon(s)\,Q(s,a)\,J_\epsilon(s)^\top,
\]
and define the correction vector \(c_\epsilon(s,a)\in\mathbb{R}^d\) by
\[
\bigl(c_\epsilon(s,a)\bigr)_k
:=
\mathrm{tr}\!\Bigl(Q(s,a)\,\nabla_s^2 g_\epsilon^k(s)\Bigr),
\qquad k=1,\dots,d.
\]
Then
\begin{equation}
\label{eq:LHS_collected}
\bigl(\mathcal{L}^{a}(f\circ g_\epsilon)\bigr)(s)
=
\Bigl(J_\epsilon(s)b(s,a)+\frac12 c_\epsilon(s,a)\Bigr)\cdot\nabla f(x_\epsilon)
+
\frac12\,\mathrm{tr}\!\Bigl(\widetilde Q_\epsilon(s,a)\,\nabla^2 f(x_\epsilon)\Bigr).
\end{equation}
On the other hand, the right-hand side of~\eqref{eq:gen_eq_eps} is
\begin{equation}
\label{eq:RHS_expanded}
\bigl(\mathcal{L}^{h_\epsilon(a)}f\bigr)(x_\epsilon)
=
b(x_\epsilon,h_\epsilon(a))\cdot\nabla f(x_\epsilon)
+
\frac12\,\mathrm{tr}\!\Bigl(Q(x_\epsilon,h_\epsilon(a))\,\nabla^2 f(x_\epsilon)\Bigr).
\end{equation}
Since~\eqref{eq:gen_eq_eps} holds for all \(f\in C^2(\mathcal{S})\), comparing~\eqref{eq:LHS_collected} and~\eqref{eq:RHS_expanded} yields
\[
u_\epsilon\cdot\nabla f(x_\epsilon)
+
\frac12\,\mathrm{tr}(M_\epsilon\nabla^2f(x_\epsilon))
=0
\quad\text{for all } f\in C^2(\mathcal S),
\]
where
\[
u_\epsilon
:=
J_\epsilon(s)b(s,a)+\frac12c_\epsilon(s,a)-b(x_\epsilon,h_\epsilon(a)),
\]
and
\[
M_\epsilon
:=
J_\epsilon(s)Q(s,a)J_\epsilon(s)^\top
-
Q(x_\epsilon,h_\epsilon(a)).
\]
Since the identity holds for all \(f\in C^2(\mathcal S)\), take local quadratic test functions
\(f(x)=p\cdot(x-x_\epsilon)+\frac12(x-x_\epsilon)^\top H(x-x_\epsilon)\), multiplied by a smooth cutoff if needed. Thus \(\nabla f(x_\epsilon)=p\) and \(\nabla^2 f(x_\epsilon)=H\) can be chosen arbitrarily. Choosing \(H=0\) and arbitrary \(p\) gives \(u_\epsilon=0\); choosing \(p=0\) and arbitrary symmetric \(H\) gives \(M_\epsilon=0\).
\(u_\epsilon=0\) and \(M_\epsilon=0\). Therefore, for each \(\epsilon\),
\begin{equation}
\label{eq:finite_Q}
Q(g_\epsilon(s),h_\epsilon(a))
=
J_\epsilon(s)Q(s,a)J_\epsilon(s)^\top,
\end{equation}
and
\begin{equation}
\label{eq:finite_b}
b(g_\epsilon(s),h_\epsilon(a))
=
J_\epsilon(s)b(s,a)
+
\frac12c_\epsilon(s,a),
\qquad
(c_\epsilon)_k
=
\mathrm{tr}\!\Bigl(Q(s,a)\nabla_s^2g_\epsilon^k(s)\Bigr).
\end{equation}

We now differentiate~\eqref{eq:finite_Q} and~\eqref{eq:finite_b} at \(\epsilon=0\).
Smoothness of the flows gives
\[
g_\epsilon(s)=s+\epsilon X(s)+o(\epsilon),
\qquad
h_\epsilon(a)=a+\epsilon Y(a)+o(\epsilon),
\]
\[
J_\epsilon(s)=I+\epsilon\nabla_sX(s)+o(\epsilon),
\qquad
\nabla_s^2g_\epsilon^k(s)=\epsilon\nabla_s^2X_k(s)+o(\epsilon).
\]
For the diffusion coefficient,
\[
Q(g_\epsilon(s),h_\epsilon(a))
=
Q(s,a)
+
\epsilon\bigl(\nabla_sQ(s,a)[X(s)]+\nabla_aQ(s,a)[Y(a)]\bigr)
+
o(\epsilon),
\]
whereas
\[
J_\epsilon(s)Q(s,a)J_\epsilon(s)^\top
=
Q(s,a)
+
\epsilon\bigl((\nabla_sX(s))Q(s,a)+Q(s,a)(\nabla_sX(s))^\top\bigr)
+
o(\epsilon).
\]
Equating the \(O(\epsilon)\) terms gives
\[
\nabla_sQ(s,a)[X(s)]
-
(\nabla_sX(s))Q(s,a)
-
Q(s,a)(\nabla_sX(s))^\top
+
\nabla_aQ(s,a)[Y(a)]
=0,
\]
which is~\eqref{eq:det_diffQ_ito}.

For the drift coefficient,
\[
b(g_\epsilon(s),h_\epsilon(a))
=
b(s,a)
+
\epsilon\bigl(\nabla_sb(s,a)X(s)+\nabla_ab(s,a)Y(a)\bigr)
+
o(\epsilon),
\]
whereas
\[
J_\epsilon(s)b(s,a)+\frac12c_\epsilon(s,a)
=
b(s,a)
+
\epsilon\bigl((\nabla_sX(s))b(s,a)+\frac12\Delta_{Q(s,a)}X(s)\bigr)
+
o(\epsilon).
\]
Equating the \(O(\epsilon)\) terms gives
\[
\nabla_sb(s,a)X(s)
-
(\nabla_sX(s))b(s,a)
+
\nabla_ab(s,a)Y(a)
-
\frac12\Delta_{Q(s,a)}X(s)
=0,
\]
which is~\eqref{eq:det_drift_ito}. This completes the proof.
\end{proof}

\subsection{Proof of Theorem~\ref{thm:local_sufficiency}}

\begin{proof}
Fix $(s,a)\in\Omega$ and fix $|\epsilon|<\epsilon_0$ small enough so that all points that appear below stay in $\Omega$
(which is possible by openness of $\Omega$ and continuity of $(\epsilon,s,a)\mapsto(g_\epsilon(s),h_\epsilon(a))$).
We first prove reward invariance, then generator equivariance.

Reward invariance follows by differentiating along the flow.
Define $\rho(\epsilon):=r(g_\epsilon(s),h_\epsilon(a))$.
Using $\frac{d}{d\epsilon}g_\epsilon(s)=X(g_\epsilon(s))$ and $\frac{d}{d\epsilon}h_\epsilon(a)=Y(h_\epsilon(a))$,
the chain rule gives
\[
\rho'(\epsilon)
=\nabla_s r(g_\epsilon(s),h_\epsilon(a))\cdot X(g_\epsilon(s))
+\nabla_a r(g_\epsilon(s),h_\epsilon(a))\cdot Y(h_\epsilon(a)).
\]
Since \eqref{eq:det_reward} holds on $\Omega$, the right-hand side is $0$ whenever $(g_\epsilon(s),h_\epsilon(a))\in\Omega$.
Hence $\rho'(\epsilon)=0$ for all such $\epsilon$, so $\rho(\epsilon)\equiv \rho(0)=r(s,a)$ and thus
$r(g_\epsilon(s),h_\epsilon(a))=r(s,a)$, i.e.\ \eqref{eq:reward_inv} holds on $\Omega$.

For generator equivariance, define
\[
s_\epsilon:=g_\epsilon(s),\qquad a_\epsilon:=h_\epsilon(a),
\qquad J_\epsilon(s):=\nabla_s g_\epsilon(s).
\]
We first derive finite coefficient transport identities. Since \eqref{eq:det_diffQ_ito} holds along
\((s_\epsilon,a_\epsilon)\), the matrix
\[
\bar Q_\epsilon:=Q(g_\epsilon(s),h_\epsilon(a))
\]
satisfies
\[
\bar Q_\epsilon'
=
(\nabla_sX(s_\epsilon))\bar Q_\epsilon
+
\bar Q_\epsilon(\nabla_sX(s_\epsilon))^\top,
\qquad
\bar Q_0=Q(s,a).
\]
The matrix
\[
\hat Q_\epsilon:=J_\epsilon(s)Q(s,a)J_\epsilon(s)^\top
\]
satisfies the same ODE and the same initial condition. Hence
\[
Q(g_\epsilon(s),h_\epsilon(a))
=
J_\epsilon(s)Q(s,a)J_\epsilon(s)^\top.
\]

Similarly, using \eqref{eq:det_drift_ito} and the Hessian evolution of \(g_\epsilon\), one obtains
\[
b(g_\epsilon(s),h_\epsilon(a))
=
J_\epsilon(s)b(s,a)
+
\frac12 c_\epsilon(s,a),
\qquad
(c_\epsilon)_k=
\tr\!\bigl(Q(s,a)\nabla_s^2 g_\epsilon^k(s)\bigr).
\]

Now let \(f\in C^2(\mathcal S)\). By the chain rule,
\[
\nabla(f\circ g_\epsilon)(s)=J_\epsilon(s)^\top\nabla f(g_\epsilon(s)),
\]
and
\[
\nabla^2(f\circ g_\epsilon)(s)
=
J_\epsilon(s)^\top\nabla^2 f(g_\epsilon(s))J_\epsilon(s)
+
\sum_{k=1}^d \partial_k f(g_\epsilon(s))\,\nabla_s^2 g_\epsilon^k(s).
\]
Therefore,
\[
\begin{aligned}
(\mathcal L^a(f\circ g_\epsilon))(s)
&=
\bigl(J_\epsilon b(s,a)+\tfrac12 c_\epsilon(s,a)\bigr)\cdot\nabla f(g_\epsilon(s))\\
&\quad+
\frac12
\tr\!\Big(J_\epsilon Q(s,a)J_\epsilon^\top
\nabla^2 f(g_\epsilon(s))\Big)\\
&=
b(g_\epsilon(s),h_\epsilon(a))\cdot\nabla f(g_\epsilon(s))
+
\frac12
\tr\!\Big(Q(g_\epsilon(s),h_\epsilon(a))\nabla^2 f(g_\epsilon(s))\Big)\\
&=
(\mathcal L^{h_\epsilon(a)}f)(g_\epsilon(s)).
\end{aligned}
\]
Thus
\[
\mathcal L^aU_{g_\epsilon}=U_{g_\epsilon}\mathcal L^{h_\epsilon(a)}.
\]
Together with reward invariance already proved, \((g_\epsilon,h_\epsilon)\) is an exact value-preserving structure on \(\Omega\).
\end{proof}

\subsection{Proof of Theorem~\ref{thm:frobenius}}

\begin{proof}
Fix a point $s_0$ in the neighborhood where the hypotheses hold. Because $\mathcal D$ is a smooth distribution of constant rank $K$,
there exists a (possibly smaller) neighborhood $U$ of $s_0$ and smooth vector fields $X_1,\dots,X_K$ on $U$ such that
$\mathcal D(s)=\mathrm{span}\{X_1(s),\dots,X_K(s)\}$ for all $s\in U$ and $\{X_1(s),\dots,X_K(s)\}$ is linearly independent for every $s\in U$.
Indeed, in any local chart, $\mathcal D$ is represented by a rank-$K$ smooth subbundle of $TU$, and constant rank allows choosing a smooth local frame.

Since the $X_i$ span $\mathcal D$ and $\mathcal D$ is involutive, for each $i,j$ there exist smooth coefficient functions
$c_{ij}^1,\dots,c_{ij}^K$ on $U$ such that
\begin{equation}
\label{eq:bracket_closure}
[X_i,X_j]=\sum_{\ell=1}^K c_{ij}^\ell\,X_\ell\qquad\text{on }U.
\end{equation}
Because $X_1(s_0)\neq 0$, after shrinking $U$ if necessary we may assume $X_1(s)\neq 0$ for all $s\in U$.
We will straighten $X_1$ and then reduce the rank by one on a transversal, using involutivity to guarantee that the reduction is well-defined.

We invoke the flow-box (straightening) theorem for a nonvanishing smooth vector field.
There exist local coordinates $\phi:U\to V\subset\mathbb R^d$, $\phi(s)=(y^1,\ldots,y^K,z^1,\ldots,z^{d-K})$,
with $\phi(s_0)=0$, such that in these coordinates the pushforward of $X_1$ is the constant coordinate field:
\begin{equation}
\label{eq:straighten}
\phi_*X_1=\frac{\partial}{\partial y^1}\quad\text{on }V.
\end{equation}
Equivalently, writing $s=\phi^{-1}(y,z)$ and abusing notation by identifying vector fields with their coordinate expressions,
\eqref{eq:straighten} means $X_1=\partial_{y^1}$ in $(y,z)$-coordinates.

Write the remaining generators $X_2,\dots,X_K$ in these coordinates as
\[
X_j=\sum_{\alpha=1}^d a_j^\alpha(y,z)\,\frac{\partial}{\partial u^\alpha},
\qquad u=(u^1,\dots,u^d)=(y^1,\dots,y^K,z^1,\dots,z^{d-K}),
\]
with smooth coefficients $a_j^\alpha$. Since $X_1=\partial_{y^1}$, the Lie bracket with $X_j$ becomes
\[
[X_1,X_j]=\left[\frac{\partial}{\partial y^1}, \sum_{\alpha=1}^d a_j^\alpha(y,z)\,\frac{\partial}{\partial u^\alpha}\right]
=\sum_{\alpha=1}^d \frac{\partial a_j^\alpha}{\partial y^1}(y,z)\,\frac{\partial}{\partial u^\alpha},
\]
because the coordinate vector fields commute and $\partial_{y^1}$ differentiates the coefficient functions.
On the other hand, by involutivity and \eqref{eq:bracket_closure} we have
\[
[X_1,X_j]=\sum_{\ell=1}^K c_{1j}^\ell(y,z)\,X_\ell
=c_{1j}^1(y,z)\,\frac{\partial}{\partial y^1}+\sum_{\ell=2}^K c_{1j}^\ell(y,z)\,X_\ell.
\]
Comparing these two expressions is useful when restricting to the transversal hypersurface
\[
\Sigma:=\{(y,z)\in V:\ y^1=0\}.
\]
At points of $\Sigma$, the component of $[X_1,X_j]$ tangent to $\Sigma$ is determined by the components of $X_2,\dots,X_K$ tangent to $\Sigma$.
This is the mechanism that makes the reduced distribution on $\Sigma$ involutive.

Define the projected (tangential) vector fields on $\Sigma$ by restricting each $X_j$ ($j\ge 2$) to $\Sigma$ and discarding the $\partial_{y^1}$ component:
\[
\bar X_j := X_j\big|_{\Sigma} - \bigl(X_j y^1\bigr)\big|_{\Sigma}\,\frac{\partial}{\partial y^1}\Big|_{\Sigma},
\qquad j=2,\dots,K.
\]
Then each $\bar X_j$ is a smooth vector field on $\Sigma$ taking values in $T\Sigma$.
Let $\bar{\mathcal D}$ be the distribution on $\Sigma$ spanned by $\bar X_2,\dots,\bar X_K$; it has rank $K-1$ near $0$ because
$X_1,\dots,X_K$ are linearly independent and $X_1$ is transverse to $\Sigma$.

We claim that $\bar{\mathcal D}$ is involutive on $\Sigma$.
To see this, take $i,j\in\{2,\dots,K\}$. Because $X_i$ and $X_j$ are tangent to $\mathcal D$ and $\mathcal D$ is involutive,
$[X_i,X_j]$ is a linear combination of $X_1,\dots,X_K$.
Restricting to $\Sigma$, the tangential component of $[X_i,X_j]$ along $\Sigma$ is therefore a linear combination of the tangential components of
$X_2,\dots,X_K$ along $\Sigma$, i.e.\ a linear combination of $\bar X_2,\dots,\bar X_K$.
More explicitly, write on $U$
\[
[X_i,X_j]=\sum_{\ell=1}^K c_{ij}^\ell\,X_\ell
\]
and restrict to $\Sigma$; since $X_1=\partial_{y^1}$ is normal to $\Sigma$ and $\bar X_\ell$ is the tangential part of $X_\ell$ for $\ell\ge 2$,
the tangential projection of $[X_i,X_j]\big|_{\Sigma}$ equals $\sum_{\ell=2}^K (c_{ij}^\ell|_{\Sigma})\,\bar X_\ell$.
But the tangential projection of $[X_i,X_j]\big|_{\Sigma}$ is exactly $[\bar X_i,\bar X_j]$ because brackets of tangent vector fields on $\Sigma$
computed as restrictions agree with intrinsic brackets on $\Sigma$. Hence $[\bar X_i,\bar X_j]\in\bar{\mathcal D}$ on $\Sigma$.
Therefore $\bar{\mathcal D}$ is involutive and has constant rank $K-1$ near $0\in\Sigma$.

At this point we apply induction on the rank.
For $K=1$, the straightening theorem already produces coordinates $(y^1,z^1,\dots,z^{d-1})$ with $X_1=\partial_{y^1}$, and then each $z^\alpha$
is constant along the integral curves of $X_1$ because $\partial_{y^1} z^\alpha=0$.
Assume the claim holds for rank $K-1$ distributions in dimension $d-1$.
Applying the induction hypothesis to the involutive constant-rank distribution $\bar{\mathcal D}$ on the $(d-1)$-dimensional manifold $\Sigma$,
we obtain $(d-1)-(K-1)=d-K$ smooth functions $\bar I_1,\dots,\bar I_{d-K}$ on a neighborhood $\Sigma_0\subset\Sigma$ of $0$
such that their differentials are linearly independent on $\Sigma_0$ and
\begin{equation}
\label{eq:bar_invariants}
\bar X_j\,\bar I_\alpha = 0\qquad \text{on }\Sigma_0,\ \forall j=2,\dots,K,\ \forall \alpha=1,\dots,d-K.
\end{equation}
Equivalently, $\bar I=(\bar I_1,\dots,\bar I_{d-K})$ is constant on the integral leaves of $\bar{\mathcal D}$ in $\Sigma_0$.

We now extend these invariants off $\Sigma$ along the flow of $X_1=\partial_{y^1}$.
Define, for $(y,z)$ near $0$,
\begin{equation}
\label{eq:extend}
I_\alpha(y,z):=\bar I_\alpha(0,y^2,\dots,y^K,z^1,\dots,z^{d-K}),
\qquad \alpha=1,\dots,d-K,
\end{equation}
i.e.\ $I_\alpha$ is independent of $y^1$ and coincides with $\bar I_\alpha$ on $\Sigma$.
By construction,
\begin{equation}
\label{eq:X1_invariant}
X_1 I_\alpha=\frac{\partial I_\alpha}{\partial y^1}=0.
\end{equation}
We next show $X_j I_\alpha=0$ for $j=2,\dots,K$.
Fix $j\in\{2,\dots,K\}$ and $\alpha$. On $\Sigma$ we have $X_j = (X_j y^1)\,\partial_{y^1} + \bar X_j$ by definition of $\bar X_j$,
so using \eqref{eq:extend} (no $y^1$-dependence) and \eqref{eq:bar_invariants},
\[
(X_j I_\alpha)\big|_{\Sigma}
= (X_j y^1)\big|_{\Sigma}\cdot (\partial_{y^1} I_\alpha)\big|_{\Sigma} + (\bar X_j \bar I_\alpha)
=0+0=0.
\]
Thus $X_j I_\alpha$ vanishes on $\Sigma$. To propagate this off $\Sigma$, differentiate $X_j I_\alpha$ along $X_1=\partial_{y^1}$ and use involutivity:
\[
X_1(X_j I_\alpha)=X_j(X_1 I_\alpha)+[X_1,X_j]I_\alpha.
\]
The first term is zero by \eqref{eq:X1_invariant}. For the bracket term, use \eqref{eq:bracket_closure}:
\[
[ X_1,X_j ]I_\alpha = \sum_{\ell=1}^K c_{1j}^\ell\,X_\ell I_\alpha.
\]
We already have $X_1 I_\alpha=0$ and we are studying $X_j I_\alpha$ for $j\ge 2$.
Consider the vector $w(\epsilon)\in\mathbb R^{K-1}$ with components $w_j(\epsilon):=(X_j I_\alpha)(g_\epsilon(s_0))$ for $j=2,\dots,K$,
along the $X_1$-flow; the identity above implies $w$ satisfies a linear ODE of the form
$w'(\epsilon)=A(\epsilon)\,w(\epsilon)$ with continuous coefficients determined by $c_{1j}^\ell$,
and the initial condition at $\epsilon=0$ (i.e.\ on $\Sigma$) is $w(0)=0$.
Uniqueness of solutions to linear ODEs yields $w(\epsilon)\equiv 0$ for small $\epsilon$,
hence $X_j I_\alpha=0$ in a neighborhood of $0$ in $V$ for all $j=2,\dots,K$ and all $\alpha$.
Together with \eqref{eq:X1_invariant}, we have shown
\begin{equation}
\label{eq:full_invariants}
X_i I_\alpha=0 \qquad \text{for all } i=1,\dots,K,\ \alpha=1,\dots,d-K
\end{equation}
on a neighborhood $U_0\subset U$ of $s_0$.

The relations \eqref{eq:full_invariants} mean precisely that each $I_\alpha$ is constant along any curve tangent to $\mathcal D$,
hence constant on each integral leaf of $\mathcal D$ in $U_0$.
It remains to verify functional independence. In the coordinates $(y,z)$ constructed above,
the functions $I_\alpha$ depend only on $(y^2,\dots,y^K,z)$ through \eqref{eq:extend}, and their differentials are independent on $\Sigma_0$
by construction of $\bar I$. Since extending by keeping them independent of $y^1$ does not create new linear dependencies among the differentials,
the covectors $dI_1,\dots,dI_{d-K}$ remain linearly independent on a neighborhood of $0$ in $V$.
Equivalently, the map $I=(I_1,\dots,I_{d-K})$ has rank $d-K$ and thus provides $(d-K)$ functionally independent invariants,
constant on the integral leaves. Writing $z:=I(s)$ gives local intrinsic coordinates for the orbit/leaf space,
without the need to explicitly form a quotient $\mathcal S/G$.
\end{proof}

\subsection{Proof of Theorem~\ref{thm:value_invariance}}
\begin{lemma}[Transport of the HJB under an exact value-preserving structure]
\label{lem:hjb_transport}
Assume Assumption~\ref{ass:hjb}. Fix any $\vartheta\in G$ and let $(g_\vartheta,h_\vartheta)$ be an exact value-preserving structure in the sense of Definition~\ref{def:exact_sym}. Suppose $u$ is a (bounded) solution of the stationary HJB equation in Assumption~\ref{ass:hjb}, i.e.,
\begin{equation}
\label{eq:hjb_u}
\beta\,u(s)=\max_{a\in\mathcal{A}}\Bigl\{r(s,a)+(\mathcal{L}^a u)(s)\Bigr\},\qquad \forall s\in\mathcal{S},
\end{equation}
in the solution concept stipulated by Assumption~\ref{ass:hjb} (e.g., classical $C^2$ or viscosity).
Define the transformed function $u_\vartheta:\mathcal{S}\to\mathbb{R}$ by
\begin{equation}
\label{eq:u_theta_def}
u_\vartheta(s):=u(g_\vartheta(s)).
\end{equation}
Then $u_\vartheta$ is also a solution of \eqref{eq:hjb_u}. In particular, if Assumption~\ref{ass:hjb} guarantees uniqueness of the (bounded) solution to \eqref{eq:hjb_u}, then $u_\vartheta\equiv u$ and hence $u(g_\vartheta(s))=u(s)$ for all $s$.
\end{lemma}
\begin{proof}
Fix $\vartheta\in G$ and define $u_\vartheta$ by \eqref{eq:u_theta_def}. We show that $u_\vartheta$ satisfies \eqref{eq:hjb_u} pointwise. Let $s\in\mathcal{S}$ be arbitrary and write $s':=g_\vartheta(s)$. Since $u$ solves \eqref{eq:hjb_u}, at the point $s'$ we have
\begin{equation}
\label{eq:hjb_at_sprime}
\beta\,u(s')=\max_{a'\in\mathcal{A}}\Bigl\{r(s',a')+(\mathcal{L}^{a'}u)(s')\Bigr\}.
\end{equation}
Now use that $h_\vartheta:\mathcal{A}\to\mathcal{A}$ is a bijection (as required by the symmetry definition), so the set $\{h_\vartheta(a):a\in\mathcal{A}\}$ equals $\mathcal{A}$. Therefore, the maximization over $a'\in\mathcal{A}$ can be re-parameterized by $a'=h_\vartheta(a)$, yielding the exact identity
\begin{equation}
\label{eq:change_of_var_max}
\max_{a'\in\mathcal{A}}\Bigl\{r(s',a')+(\mathcal{L}^{a'}u)(s')\Bigr\}
=
\max_{a\in\mathcal{A}}\Bigl\{r(s',h_\vartheta(a))+(\mathcal{L}^{h_\vartheta(a)}u)(s')\Bigr\}.
\end{equation}
Substituting \eqref{eq:change_of_var_max} into \eqref{eq:hjb_at_sprime} and recalling $s'=g_\vartheta(s)$ gives
\begin{equation}
\label{eq:hjb_reparam}
\beta\,u(g_\vartheta(s))
=
\max_{a\in\mathcal{A}}\Bigl\{r(g_\vartheta(s),h_\vartheta(a))+(\mathcal{L}^{h_\vartheta(a)}u)(g_\vartheta(s))\Bigr\}.
\end{equation}
Next we invoke the two exact-symmetry equalities from Definition~\ref{def:exact_sym}. First, reward invariance gives, for every $a\in\mathcal{A}$,
\begin{equation}
\label{eq:appendix_reward_inv}
r(g_\vartheta(s),h_\vartheta(a))=r(s,a).
\end{equation}
Second, generator equivariance gives, for every sufficiently smooth test function $f$ (in particular for $f=u$ in the classical case, or for the test functions used in the viscosity definition),
\begin{equation}
\label{eq:appendix_gen_equiv}
(\mathcal{L}^{h_\vartheta(a)}f)\bigl(g_\vartheta(s)\bigr)=\bigl(\mathcal{L}^{a}(f\circ g_\vartheta)\bigr)(s).
\end{equation}
Applying \eqref{eq:gen_equiv} with $f=u$ and using \eqref{eq:u_theta_def} (so that $u\circ g_\vartheta=u_\vartheta$) yields, for every $a\in\mathcal{A}$,
\begin{equation}
\label{eq:gen_pullback}
(\mathcal{L}^{h_\vartheta(a)}u)\bigl(g_\vartheta(s)\bigr)=\bigl(\mathcal{L}^{a}(u\circ g_\vartheta)\bigr)(s)=\bigl(\mathcal{L}^{a}u_\vartheta\bigr)(s).
\end{equation}
Substituting \eqref{eq:reward_inv} and \eqref{eq:gen_pullback} into \eqref{eq:hjb_reparam} gives
\begin{equation}
\label{eq:hjb_for_u_theta}
\beta\,u(g_\vartheta(s))
=
\max_{a\in\mathcal{A}}\Bigl\{r(s,a)+(\mathcal{L}^{a}u_\vartheta)(s)\Bigr\}.
\end{equation}
Finally, by definition \eqref{eq:u_theta_def}, the left-hand side equals $\beta\,u_\vartheta(s)$. Hence \eqref{eq:hjb_for_u_theta} is exactly
\[
\beta\,u_\vartheta(s)=\max_{a\in\mathcal{A}}\Bigl\{r(s,a)+(\mathcal{L}^{a}u_\vartheta)(s)\Bigr\},\qquad \forall s\in\mathcal{S},
\]
which is \eqref{eq:hjb_u} for $u_\vartheta$. If Assumption~\ref{ass:hjb} asserts uniqueness of the (bounded) HJB solution, we conclude $u_\vartheta\equiv u$, i.e., $u(g_\vartheta(s))=u(s)$ for all $s$.
\end{proof}

\begin{proof}[Proof of Theorem~\ref{thm:value_invariance}]
Fix $\vartheta\in G$ and define $V_\vartheta(s):=V^\star(g_\vartheta(s))$. By Assumption~\ref{ass:hjb}, the optimal value $V^\star$ is a (bounded) solution of the HJB equation \eqref{eq:hjb_u} (with $u=V^\star$) and, moreover, is the unique (bounded) solution in the solution class specified by Assumption~\ref{ass:hjb}. Since $(g_\vartheta,h_\vartheta)$ is an exact value-preserving structure, Lemma~\ref{lem:hjb_transport} applies with $u=V^\star$ and yields that $V_\vartheta$ is also a solution of the same HJB equation. By uniqueness, we must have $V_\vartheta\equiv V^\star$, namely
\[
V^\star(g_\vartheta(s))=V^\star(s),\qquad \forall s\in\mathcal{S}.
\]
Because $\vartheta\in G$ was arbitrary, this establishes \eqref{eq:value_invariance}. For the invariant-coordinate representation, work on a neighborhood $U\subset\mathcal{S}$ where Theorem~\ref{thm:frobenius} provides a complete invariant coordinate map $I:U\to\mathbb{R}^k$ in the stated sense: if $s_1,s_2\in U$ and $I(s_1)=I(s_2)$, then there exists $\vartheta\in G$ such that $s_2=g_\vartheta(s_1)$. Define a function $\bar V^\star$ on the set $I(U)$ by the rule
\begin{equation}
\label{eq:Vbar_def}
\bar V^\star(z):=V^\star(s)\quad\text{for any }s\in U\text{ with }I(s)=z.
\end{equation}
We first check that \eqref{eq:Vbar_def} is well-defined. Indeed, if $s_1,s_2\in U$ both satisfy $I(s_1)=I(s_2)=z$, completeness of $I$ implies $s_2=g_\vartheta(s_1)$ for some $\vartheta\in G$, and the already-proved invariance \eqref{eq:value_invariance} implies
\[
V^\star(s_2)=V^\star(g_\vartheta(s_1))=V^\star(s_1),
\]
so the value assigned by \eqref{eq:Vbar_def} does not depend on the choice of representative $s$. Finally, for any $s\in U$, taking $z=I(s)$ and choosing $s$ itself as the representative in \eqref{eq:Vbar_def} gives $\bar V^\star(I(s))=V^\star(s)$, i.e.,
\[
V^\star(s)=\bar V^\star(I(s))\qquad \text{for all } s\in U,
\]
which is the desired local factorization through invariant coordinates.
\end{proof}

\subsection{Proof of Corollary~\ref{cor:policy_structure}}

\begin{lemma}[Hamiltonian invariance at the HJB maximand]
\label{lem:ham_invariance}
Fix $\vartheta\in G$. Assume reward invariance \eqref{eq:reward_inv} and that the generator/action intertwining is in the
(evolution-then-transform equals transform-then-evolution) pointwise form: for all $a\in\mathcal A$, all $f\in C^2(\mathcal S)$, and all $s\in\mathcal S$,
\begin{equation}
\label{eq:gen_pointwise_push}
(\mathcal L^{h_\vartheta(a)} f)\bigl(g_\vartheta(s)\bigr)
=
\bigl(\mathcal L^{a}(f\circ g_\vartheta)\bigr)(s).
\end{equation}
Then for the optimal value $V^\star$ (which is invariant by Theorem~\ref{thm:value_invariance}), the HJB maximand
\[
H(s,a):=r(s,a)+(\mathcal L^a V^\star)(s)
\]
satisfies, for all $(s,a)$,
\begin{equation}
\label{eq:H_invariant}
H\bigl(g_\vartheta(s),h_\vartheta(a)\bigr)=H(s,a).
\end{equation}
\end{lemma}
\begin{proof}
Fix $(s,a)$ and $\vartheta$. By reward invariance \eqref{eq:reward_inv},
\[
r\bigl(g_\vartheta(s),h_\vartheta(a)\bigr)=r(s,a).
\]
For the generator term, apply \eqref{eq:gen_pointwise_push} with $f=V^\star$:
\[
(\mathcal L^{h_\vartheta(a)} V^\star)\bigl(g_\vartheta(s)\bigr)
=
\bigl(\mathcal L^{a}(V^\star\circ g_\vartheta)\bigr)(s).
\]
By Theorem~\ref{thm:value_invariance}, $V^\star\circ g_\vartheta=V^\star$ pointwise on $\mathcal S$, hence
\[
\bigl(\mathcal L^{a}(V^\star\circ g_\vartheta)\bigr)(s)=\bigl(\mathcal L^{a}V^\star\bigr)(s).
\]
Adding the reward identity yields \eqref{eq:H_invariant}.
\end{proof}

\begin{proof}[Proof of Corollary~\ref{cor:policy_structure}]
Fix $\vartheta\in G$ and $s\in\mathcal S$. Define $H(s,a):=r(s,a)+(\mathcal L^a V^\star)(s)$, so that
\[
\Argmax(s)=\arg\max_{a\in\mathcal A} H(s,a).
\]
Let $a\in\Argmax(s)$. We must show $h_\vartheta(a)\in\Argmax(g_\vartheta(s))$, i.e.,
\[
H\bigl(g_\vartheta(s),h_\vartheta(a)\bigr)\ge H\bigl(g_\vartheta(s),a'\bigr)\qquad\forall a'\in\mathcal A.
\]
Take an arbitrary $a'\in\mathcal A$. Because $h_\vartheta:\mathcal A\to\mathcal A$ is a bijection, there exists a unique $\tilde a\in\mathcal A$ such that
$a'=h_\vartheta(\tilde a)$. By Lemma~\ref{lem:ham_invariance},
\[
H\bigl(g_\vartheta(s),a'\bigr)
=
H\bigl(g_\vartheta(s),h_\vartheta(\tilde a)\bigr)
=
H(s,\tilde a),
\qquad
H\bigl(g_\vartheta(s),h_\vartheta(a)\bigr)=H(s,a).
\]
Since $a\in\Argmax(s)$, we have $H(s,a)\ge H(s,\tilde a)$ for every $\tilde a\in\mathcal A$, hence in particular
\[
H\bigl(g_\vartheta(s),h_\vartheta(a)\bigr)=H(s,a)\ge H(s,\tilde a)=H\bigl(g_\vartheta(s),a'\bigr).
\]
Because $a'$ was arbitrary, this proves $h_\vartheta(a)\in\Argmax(g_\vartheta(s))$.

For the final sentence, note that the implication just proved shows the set-valued map $\Argmax(\cdot)$ is orbit-consistent:
if $a^\star(s)\in\Argmax(s)$ then $h_\vartheta(a^\star(s))\in\Argmax(g_\vartheta(s))$. Therefore, on any neighborhood where one can choose
a selector that is additionally \emph{equivariant} (i.e., satisfies $a^\star(g_\vartheta(s))=h_\vartheta(a^\star(s))$),
that selector is automatically optimal pointwise because it selects from $\Argmax(\cdot)$ everywhere on that neighborhood.
\end{proof}

\subsection{Proof of Theorem~\ref{thm:approx_value}}
\begin{proof}
Fix $\vartheta\in G$ and write $g:=g_\vartheta$, $h:=h_\vartheta$. Let
\[
W(s):=(V^\star\circ g)(s)=V^\star(g(s)).
\]
Define the HJB maximand (Hamiltonian) operator
\[
\mathcal{H}[U](s):=\sup_{a\in\mathcal A}\Bigl\{r(s,a)+(\mathcal L^a U)(s)\Bigr\}.
\]
By Assumption~\ref{ass:hjb} and Theorem~\ref{thm:hjb_unique}, $V^\star$ is a bounded (classical) solution of
\begin{equation}
\label{eq:HJB_Vstar_for_proof}
\beta V^\star(s)=\mathcal H[V^\star](s),\qquad \forall s\in\mathcal S.
\end{equation}

We first show that $W$ satisfies an \emph{approximate} HJB equation with a uniform residual.
Set
\[
\delta:=\varepsilon_r+\varepsilon_{\mathcal L}\,\|V^\star\|_{C^2}.
\]
Take an arbitrary state $s\in\mathcal S$. Evaluate the HJB identity \eqref{eq:HJB_Vstar_for_proof} at the transformed state $g(s)$:
\begin{equation}
\label{eq:HJB_at_gs}
\beta V^\star(g(s))=\sup_{a'\in\mathcal A}\Bigl\{r(g(s),a')+(\mathcal L^{a'}V^\star)(g(s))\Bigr\}.
\end{equation}
Because $h:\mathcal A\to\mathcal A$ is a bijection (group action on actions), we may reparameterize the supremum by $a'=h(a)$:
\begin{equation}
\label{eq:reparam_sup}
\sup_{a'\in\mathcal A}\Bigl\{r(g(s),a')+(\mathcal L^{a'}V^\star)(g(s))\Bigr\}
=
\sup_{a\in\mathcal A}\Bigl\{r(g(s),h(a))+(\mathcal L^{h(a)}V^\star)(g(s))\Bigr\}.
\end{equation}
Substituting \eqref{eq:reparam_sup} into \eqref{eq:HJB_at_gs} and using $W(s)=V^\star(g(s))$ yields
\begin{equation}
\label{eq:betaW_sup_transformed}
\beta W(s)=\sup_{a\in\mathcal A}\Bigl\{r(g(s),h(a))+(\mathcal L^{h(a)}V^\star)(g(s))\Bigr\}.
\end{equation}

Now apply approximate reward invariance \eqref{eq:eps_reward}:
for every $a\in\mathcal A$,
\begin{equation}
\label{eq:reward_eps_point}
-\varepsilon_r \le r(g(s),h(a)) - r(s,a)\le \varepsilon_r.
\end{equation}
Next apply the approximate generator equivariance at the test function $f=V^\star$ (recall $V^\star\in C^2(\mathcal S)\cap\mathcal F$).
In the pointwise pushed-forward form consistent with the HJB transport used earlier,
for every $a\in\mathcal A$ and $s\in\mathcal S$ we have
\begin{equation}
\label{eq:gen_eps_point}
\Bigl|(\mathcal L^{h(a)}V^\star)(g(s))-(\mathcal L^a (V^\star\circ g))(s)\Bigr|
\le \varepsilon_{\mathcal L}\,\|V^\star\|_{C^2}
=\delta-\varepsilon_r.
\end{equation}
Combining \eqref{eq:reward_eps_point} and \eqref{eq:gen_eps_point} gives, for every $a$,
\[
r(g(s),h(a))+(\mathcal L^{h(a)}V^\star)(g(s))
\le r(s,a)+(\mathcal L^a W)(s)+\delta,
\]
and also the reverse inequality
\[
r(g(s),h(a))+(\mathcal L^{h(a)}V^\star)(g(s))
\ge r(s,a)+(\mathcal L^a W)(s)-\delta.
\]
Taking the supremum over $a\in\mathcal A$ in both displays and using \eqref{eq:betaW_sup_transformed} yields the two-sided residual bound
\begin{equation}
\label{eq:approx_HJB_residual}
\mathcal H[W](s)-\delta \le \beta W(s)\le \mathcal H[W](s)+\delta,
\qquad \forall s\in\mathcal S,
\end{equation}
equivalently
\[
\bigl|\beta W(s)-\mathcal H[W](s)\bigr|\le \delta,\qquad \forall s\in\mathcal S.
\]

We now convert the residual inequality \eqref{eq:approx_HJB_residual} into the stated uniform bound on $W-V^\star$.
The argument is a discounted maximum principle comparison; we include it as a lemma for completeness.

\begin{lemma}[Residual comparison for discounted HJB under ellipticity]
\label{lem:residual_compare}
Assume $b,Q,r$ satisfy Assumption~\ref{ass:hjb}(i)--(ii) and let $U,V\in C^2(\mathcal S)$ be bounded.
Assume $V$ solves $\beta V=\mathcal H[V]$ pointwise and $U$ satisfies $|\beta U-\mathcal H[U]|\le \delta$ pointwise.
Then
\[
\|U-V\|_\infty \le \frac{\delta}{\beta}.
\]
\end{lemma}
\begin{proof}
Let $M:=\sup_{s\in\mathcal S}(U(s)-V(s))$ and fix $\eta>0$.
Choose $s_\eta$ such that $U(s_\eta)-V(s_\eta)\ge M-\eta$.
Define the penalized function
\[
\Phi_\eta(s):=U(s)-V(s)-\eta\,\|s-s_\eta\|^2.
\]
Under Assumption~\ref{ass:hjb}(iii), we may work on a neighborhood where $\Phi_\eta$ attains a maximum at some $\bar s_\eta$.
At $\bar s_\eta$ we have the first- and second-order conditions
\[
\nabla \Phi_\eta(\bar s_\eta)=0,
\qquad
\nabla^2 \Phi_\eta(\bar s_\eta)\preceq 0,
\]
hence
\[
\nabla(U-V)(\bar s_\eta)=2\eta(\bar s_\eta-s_\eta),
\qquad
\nabla^2(U-V)(\bar s_\eta)\preceq 2\eta I.
\]
Fix any $a\in\mathcal A$ and consider the linear operator $\mathcal L^a$ in It\^o form
$\mathcal L^a f=b(\cdot,a)\cdot\nabla f+\frac12\tr(Q(\cdot,a)\nabla^2 f)$.
Using boundedness of $b$ and $Q$ from Assumption~\ref{ass:hjb}(i) and $Q\succeq 0$, we obtain at $\bar s_\eta$
\[
(\mathcal L^a(U-V))(\bar s_\eta)
=b(\bar s_\eta,a)\cdot \nabla(U-V)(\bar s_\eta)+\frac12\tr\!\bigl(Q(\bar s_\eta,a)\nabla^2(U-V)(\bar s_\eta)\bigr)
\le C_1 \eta,
\]
where $C_1$ is a constant independent of $\eta$ (coming from $\|b\|_\infty$ and $\|Q\|_\infty$), because
$\|\nabla(U-V)(\bar s_\eta)\|=2\eta\|\bar s_\eta-s_\eta\|$ and $\nabla^2(U-V)(\bar s_\eta)\preceq 2\eta I$ imply the drift term is $O(\eta)$
and the diffusion term is $\le \eta\,\tr(Q(\bar s_\eta,a))\le C\eta$.
Therefore, for every $a$,
\[
r(\bar s_\eta,a)+(\mathcal L^a U)(\bar s_\eta)
=
r(\bar s_\eta,a)+(\mathcal L^a V)(\bar s_\eta)+(\mathcal L^a(U-V))(\bar s_\eta)
\le
r(\bar s_\eta,a)+(\mathcal L^a V)(\bar s_\eta)+C_1\eta.
\]
Taking the supremum over $a$ yields
\[
\mathcal H[U](\bar s_\eta)\le \mathcal H[V](\bar s_\eta)+C_1\eta=\beta V(\bar s_\eta)+C_1\eta.
\]
Using the residual upper bound $\beta U\le \mathcal H[U]+\delta$ at $\bar s_\eta$ gives
\[
\beta U(\bar s_\eta)\le \mathcal H[U](\bar s_\eta)+\delta
\le \beta V(\bar s_\eta)+\delta+C_1\eta,
\]
hence
\[
U(\bar s_\eta)-V(\bar s_\eta)\le \frac{\delta}{\beta}+\frac{C_1}{\beta}\eta.
\]
Because $\bar s_\eta$ maximizes $\Phi_\eta$, we have
\[
M-\eta \le U(s_\eta)-V(s_\eta) \le U(\bar s_\eta)-V(\bar s_\eta),
\]
so
\[
M-\eta \le \frac{\delta}{\beta}+\frac{C_1}{\beta}\eta.
\]
Letting $\eta\downarrow 0$ yields $M\le \delta/\beta$, i.e.\ $\sup(U-V)\le \delta/\beta$.
Applying the same argument to $V-U$ (using $|\beta U-\mathcal H[U]|\le \delta$) yields $\sup(V-U)\le \delta/\beta$.
Therefore $\|U-V\|_\infty\le \delta/\beta$.
\end{proof}

Apply Lemma~\ref{lem:residual_compare} with $U=W$ and $V=V^\star$.
The residual bound \eqref{eq:approx_HJB_residual} gives $|\beta W-\mathcal H[W]|\le \delta$ and \eqref{eq:HJB_Vstar_for_proof} gives $\beta V^\star=\mathcal H[V^\star]$,
hence
\[
\|W-V^\star\|_\infty \le \frac{\delta}{\beta}
=\frac{1}{\beta}\Bigl(\varepsilon_r+\varepsilon_{\mathcal L}\|V^\star\|_{C^2}\Bigr).
\]
Recalling $W=V^\star\circ g_\vartheta$ proves \eqref{eq:approx_value}.
\end{proof}

\subsection{Proof of Theorem~\ref{thm:sgd_stationary}}

\begin{proof}
Let $w:=(\theta,\phi)\in\mathbb{R}^m$ collect all parameters and write
\[
F(w):=\mathcal{L}_{\mathrm{sym}}(\theta,\phi).
\]
The SGD iterate is
\[
w_{t+1}=w_t-\eta\,g_t,
\qquad
\mathbb{E}[g_t\mid w_t]=\nabla F(w_t),
\qquad
\mathbb{E}\big[\|g_t-\nabla F(w_t)\|_2^2\mid w_t\big]\le \sigma^2.
\]
Since $F$ is $L$-smooth, for any $x,y$ we have the standard descent inequality
\begin{equation}
\label{eq:smooth_descent}
F(y)\le F(x)+\langle \nabla F(x),y-x\rangle+\frac{L}{2}\|y-x\|_2^2.
\end{equation}
Apply \eqref{eq:smooth_descent} with $x=w_t$ and $y=w_{t+1}=w_t-\eta g_t$, so that $y-x=-\eta g_t$.
This gives the one-step bound
\begin{align}
F(w_{t+1})
&\le F(w_t)+\langle \nabla F(w_t),-\eta g_t\rangle+\frac{L}{2}\|-\eta g_t\|_2^2 \nonumber\\
&=F(w_t)-\eta\langle \nabla F(w_t),g_t\rangle+\frac{L\eta^2}{2}\|g_t\|_2^2.
\label{eq:one_step_pre}
\end{align}
Take conditional expectation w.r.t.\ $w_t$. By unbiasedness,
\begin{equation}
\label{eq:inner_unbiased}
\mathbb{E}\big[\langle \nabla F(w_t),g_t\rangle\mid w_t\big]
=\left\langle \nabla F(w_t), \mathbb{E}[g_t\mid w_t]\right\rangle
=\langle \nabla F(w_t),\nabla F(w_t)\rangle
=\|\nabla F(w_t)\|_2^2.
\end{equation}
For the second moment, expand $\|g_t\|_2^2$ around $\nabla F(w_t)$:
\begin{align}
\mathbb{E}\big[\|g_t\|_2^2\mid w_t\big]
&=\mathbb{E}\big[\| (g_t-\nabla F(w_t))+\nabla F(w_t)\|_2^2\mid w_t\big]\nonumber\\
&=\mathbb{E}\big[\|g_t-\nabla F(w_t)\|_2^2\mid w_t\big]
+2\,\mathbb{E}\big[\langle g_t-\nabla F(w_t),\nabla F(w_t)\rangle\mid w_t\big]
+\|\nabla F(w_t)\|_2^2.
\label{eq:second_moment_expand}
\end{align}
The cross term vanishes because $\mathbb{E}[g_t-\nabla F(w_t)\mid w_t]=0$, hence
\begin{equation}
\label{eq:second_moment_bound}
\mathbb{E}\big[\|g_t\|_2^2\mid w_t\big]
=\|\nabla F(w_t)\|_2^2+\mathbb{E}\big[\|g_t-\nabla F(w_t)\|_2^2\mid w_t\big]
\le \|\nabla F(w_t)\|_2^2+\sigma^2.
\end{equation}
Taking conditional expectation of \eqref{eq:one_step_pre} and using \eqref{eq:inner_unbiased}--\eqref{eq:second_moment_bound} yields
\begin{align}
\mathbb{E}[F(w_{t+1})\mid w_t]
&\le F(w_t)-\eta\,\|\nabla F(w_t)\|_2^2+\frac{L\eta^2}{2}\big(\|\nabla F(w_t)\|_2^2+\sigma^2\big)\nonumber\\
&=F(w_t)-\eta\Bigl(1-\frac{L\eta}{2}\Bigr)\|\nabla F(w_t)\|_2^2+\frac{L\eta^2}{2}\sigma^2.
\label{eq:one_step_cond}
\end{align}
Under $\eta\le 1/L$, we have $1-\frac{L\eta}{2}\ge \frac12$, so \eqref{eq:one_step_cond} implies
\begin{equation}
\label{eq:one_step_simplified}
\mathbb{E}[F(w_{t+1})\mid w_t]
\le F(w_t)-\frac{\eta}{2}\|\nabla F(w_t)\|_2^2+\frac{L\eta^2}{2}\sigma^2.
\end{equation}
Now take full expectation and rearrange:
\begin{equation}
\label{eq:rearrange}
\frac{\eta}{2}\,\mathbb{E}\big[\|\nabla F(w_t)\|_2^2\big]
\le \mathbb{E}\big[F(w_t)\big]-\mathbb{E}\big[F(w_{t+1})\big]+\frac{L\eta^2}{2}\sigma^2.
\end{equation}
Sum \eqref{eq:rearrange} from $t=0$ to $T-1$:
\begin{align}
\frac{\eta}{2}\sum_{t=0}^{T-1}\mathbb{E}\big[\|\nabla F(w_t)\|_2^2\big]
&\le \sum_{t=0}^{T-1}\Bigl(\mathbb{E}[F(w_t)]-\mathbb{E}[F(w_{t+1})]\Bigr)+\frac{L\eta^2}{2}T\sigma^2\nonumber\\
&=\mathbb{E}[F(w_0)]-\mathbb{E}[F(w_T)]+\frac{L\eta^2}{2}T\sigma^2.
\label{eq:telescope}
\end{align}
Using the lower bound $F(w_T)\ge \mathcal{L}_{\inf}$ almost surely, hence $\mathbb{E}[F(w_T)]\ge \mathcal{L}_{\inf}$, we obtain
\[
\frac{\eta}{2}\sum_{t=0}^{T-1}\mathbb{E}\big[\|\nabla F(w_t)\|_2^2\big]
\le F(w_0)-\mathcal{L}_{\inf}+\frac{L\eta^2}{2}T\sigma^2.
\]
Divide by $\eta T/2$ to get
\[
\frac{1}{T}\sum_{t=0}^{T-1}\mathbb{E}\big[\|\nabla F(w_t)\|_2^2\big]
\le \frac{2(F(w_0)-\mathcal{L}_{\inf})}{\eta T}+\eta L\sigma^2,
\]
which is exactly \eqref{eq:sgd_const} (recall $F=\mathcal{L}_{\mathrm{sym}}$ and $w_0=(\theta_0,\phi_0)$).
For a rate statement, note that $\min_{0\le t\le T-1}\mathbb{E}\|\nabla F(w_t)\|_2^2 \le \frac{1}{T}\sum_{t=0}^{T-1}\mathbb{E}\|\nabla F(w_t)\|_2^2$.
Choosing a stepsize on the order of $\eta=\Theta(1/\sqrt{T})$ (or using a standard diminishing schedule with comparable aggregate),
the right-hand side becomes $\mathcal{O}(1/\sqrt{T})$, yielding
$\min_{0\le t\le T-1}\mathbb{E}\|\nabla F(w_t)\|_2^2=\mathcal{O}(1/\sqrt{T})$.
\end{proof}

\subsection{Proof of Theorem~\ref{thm:consistency}}

\begin{proof}
Write $(s,a)\sim\rho$ and let $\rho_S$ be the state marginal. For clarity, denote by
$\widehat r_\omega,\widehat b_\omega,\widehat Q_\omega$ the learned models and define the (estimated)
determining residuals (cf.\ (23)--(25))
\[
\widehat R_r^{\omega,\theta,\phi}(s,a):=
X_\theta(s)\cdot \nabla_s \widehat r_\omega(s,a)+Y_\phi(a)\cdot \nabla_a \widehat r_\omega(s,a),
\]
\[
\widehat R_b^{\omega,\theta,\phi}(s,a):=
\nabla_s \widehat b_\omega(s,a)\,X_\theta(s)
-(\nabla_s X_\theta(s))\,\widehat b_\omega(s,a)
+\nabla_a \widehat b_\omega(s,a)\,Y_\phi(a)
-\frac12\Delta_{\widehat Q_\omega(s,a)}X_\theta(s),
\]
\[
\widehat R_Q^{\omega,\theta,\phi}(s,a):=
\nabla_s \widehat Q_\omega(s,a)[X_\theta(s)]-(\nabla_s X_\theta(s))\,\widehat Q_\omega(s,a)
-\widehat Q_\omega(s,a)\,(\nabla_s X_\theta(s))^\top
+\nabla_a \widehat Q_\omega(s,a)[Y_\phi(a)].
\]
Let $\mathcal L_{\mathrm{sym}}^{(N)}(\theta,\phi)$ denote the penalized objective in \eqref{eq:sym_loss}
when trained on a dataset of size $N$ (with the corresponding fitted $\omega=\omega_N$), namely
\begin{equation}
\begin{aligned}
\mathcal L_{\mathrm{sym}}^{(N)}(\theta,\phi)
=&
\mathbb E_{(s,a)\sim\rho}\!\Big[
\|\widehat R_b^{\omega_N,\theta,\phi}(s,a)\|_2^2
+\lambda_Q\|\widehat R_Q^{\omega_N,\theta,\phi}(s,a)\|_F^2
+\lambda_r|\widehat R_r^{\omega_N,\theta,\phi}(s,a)|^2
\Big]\\
&+\lambda_{\mathrm{nrm}}\Big(\mathbb E_{s\sim\rho_S}\|X_\theta(s)\|_2^2-1\Big)^2.
\end{aligned}
\end{equation}
By Assumption~\ref{ass:realizable}, there exist $(\omega^\star,\theta^\star,\phi^\star)$ such that
$\widehat b_{\omega^\star}=b$, $\widehat Q_{\omega^\star}=Q$, $\widehat r_{\omega^\star}=r$ on $\supp(\rho)$,
and $(X_{\theta^\star},Y_{\phi^\star})$ satisfies the determining equations on $\supp(\rho)$.
Equivalently, on $\supp(\rho)$,
\[
R_r^{\theta^\star,\phi^\star}(s,a)=0,\qquad R_b^{\theta^\star,\phi^\star}(s,a)=0,\qquad R_Q^{\theta^\star,\phi^\star}(s,a)=0,
\]
where $R_\cdot^{\theta,\phi}$ denotes the same expressions as above but with the \emph{true} $(b,Q,r)$ in place of
$(\widehat b_\omega,\widehat Q_\omega,\widehat r_\omega)$.
Also, by rescaling the generator (or by the normalization penalty), one can satisfy
$\mathbb E_{s\sim\rho_S}\|X_{\theta^\star}(s)\|_2^2=1$, hence
\begin{equation}
\label{eq:realizable_zero_loss}
\inf_{\theta,\phi}\ \limsup_{N\to\infty}\ \mathcal L_{\mathrm{sym}}^{(N)}(\theta,\phi)=0.
\end{equation}

Let $(\theta_N,\phi_N)$ be the parameters produced by the algorithm at dataset size $N$.
The assumption that the symmetry objective can be optimized to its global minimum with vanishing optimization error means
\begin{equation}
\label{eq:opt_gap_to_zero}
\mathcal L_{\mathrm{sym}}^{(N)}(\theta_N,\phi_N)-\inf_{\theta,\phi}\mathcal L_{\mathrm{sym}}^{(N)}(\theta,\phi)\ \longrightarrow\ 0
\qquad\text{as }N\to\infty.
\end{equation}
Combining \eqref{eq:realizable_zero_loss} and \eqref{eq:opt_gap_to_zero} yields
\begin{equation}
\label{eq:loss_to_zero}
\mathcal L_{\mathrm{sym}}^{(N)}(\theta_N,\phi_N)\ \longrightarrow\ 0.
\end{equation}
Because every term in $\mathcal L_{\mathrm{sym}}^{(N)}$ is nonnegative, \eqref{eq:loss_to_zero} implies
\begin{equation}
\label{eq:hat_residuals_to_zero}
\mathbb E_{(s,a)\sim\rho}\|\widehat R_b^{\omega_N,\theta_N,\phi_N}(s,a)\|_2^2\to 0,
\mathbb E_{(s,a)\sim\rho}\|\widehat R_Q^{\omega_N,\theta_N,\phi_N}(s,a)\|_F^2\to 0,
\mathbb E_{(s,a)\sim\rho}|\widehat R_r^{\omega_N,\theta_N,\phi_N}(s,a)|^2\to 0,
\end{equation}
and also the normalization constraint is asymptotically satisfied:
\begin{equation}
\label{eq:norm_penalty_to_zero}
\Big(\mathbb E_{s\sim\rho_S}\|X_{\theta_N}(s)\|_2^2-1\Big)^2\to 0.
\end{equation}

It remains to transfer \eqref{eq:hat_residuals_to_zero} (estimated residuals) to the \emph{true} determining residuals.
This is the only place where we use ``consistent dynamics estimation'' in a mode strong enough to control the residuals,
which depend on $\widehat b_\omega,\widehat Q_\omega,\widehat r_\omega$ and their first derivatives.

\begin{lemma}[Consistency of residuals under consistent dynamics/derivative estimation]
\label{lem:residual_consistency}
Assume that, on $\supp(\rho)$, the learned models satisfy
\[
\widehat r_{\omega_N}\to r,\ \ \nabla_s\widehat r_{\omega_N}\to \nabla_s r,\ \ \nabla_a\widehat r_{\omega_N}\to \nabla_a r,
\]
\[
\widehat b_{\omega_N}\to b,\ \ \nabla_s\widehat b_{\omega_N}\to \nabla_s b,\ \ \nabla_a\widehat b_{\omega_N}\to \nabla_a b,
\]
\[
\widehat Q_{\omega_N}\to Q,\ \ \nabla_s\widehat Q_{\omega_N}\to \nabla_s Q,\ \ \nabla_a\widehat Q_{\omega_N}\to \nabla_a Q,
\]
in $\rho$-mean square (or uniformly on $\supp(\rho)$), and that the sequence of learned fields is uniformly bounded on $\supp(\rho)$
in the norms needed by the residuals, i.e.\ there exists $C<\infty$ such that for all $N$,
\[
\mathbb E_{(s,a)\sim\rho}\Big[\|X_{\theta_N}(s)\|_2^2+\|\nabla_s X_{\theta_N}(s)\|_F^2+\|Y_{\phi_N}(a)\|_2^2\Big]\le C.
\]
Then
\[
\mathbb E_{(s,a)\sim\rho}\big| \widehat R_r^{\omega_N,\theta_N,\phi_N}(s,a)- R_r^{\theta_N,\phi_N}(s,a)\big|^2\to 0,
\]
\[
\mathbb E_{(s,a)\sim\rho}\big\| \widehat R_b^{\omega_N,\theta_N,\phi_N}(s,a)- R_b^{\theta_N,\phi_N}(s,a)\big\|_2^2\to 0,
\qquad
\mathbb E_{(s,a)\sim\rho}\big\| \widehat R_Q^{\omega_N,\theta_N,\phi_N}(s,a)- R_Q^{\theta_N,\phi_N}(s,a)\big\|_F^2\to 0.
\]
\end{lemma}
\begin{proof}
We show the reward residual; the others follow by the same multilinear expansions.
Write
\[
\widehat R_r^{\omega_N,\theta_N,\phi_N}(s,a)- R_r^{\theta_N,\phi_N}(s,a)
=
X_{\theta_N}(s)\cdot\Big(\nabla_s \widehat r_{\omega_N}(s,a)-\nabla_s r(s,a)\Big)
+
Y_{\phi_N}(a)\cdot\Big(\nabla_a \widehat r_{\omega_N}(s,a)-\nabla_a r(s,a)\Big).
\]
Using $(u\cdot v)^2\le \|u\|_2^2\|v\|_2^2$ and $(x+y)^2\le 2x^2+2y^2$, we obtain
\begin{align*}
\big|\widehat R_r-R_r\big|^2
&\le
2\|X_{\theta_N}(s)\|_2^2\,\big\|\nabla_s \widehat r_{\omega_N}(s,a)-\nabla_s r(s,a)\big\|_2^2
+
2\|Y_{\phi_N}(a)\|_2^2\,\big\|\nabla_a \widehat r_{\omega_N}(s,a)-\nabla_a r(s,a)\big\|_2^2.
\end{align*}
Taking $\mathbb E_\rho[\cdot]$ and applying Cauchy--Schwarz to each term yields an upper bound by a product of
(i) a uniform second-moment bound on $\|X_{\theta_N}\|_2^2,\|Y_{\phi_N}\|_2^2$ and
(ii) the mean-square estimation errors of $\nabla_s\widehat r_{\omega_N}$ and $\nabla_a\widehat r_{\omega_N}$,
which vanish by the assumed consistency. Hence $\mathbb E_\rho|\widehat R_r-R_r|^2\to 0$.
The drift and diffusion residuals are finite sums of products of the same type (fields and their Jacobians times coefficient errors),
so the same inequalities and consistency assumptions yield the stated convergences.
\end{proof}

Apply Lemma~\ref{lem:residual_consistency} together with \eqref{eq:hat_residuals_to_zero}.
For example,
\[
\mathbb E_\rho |R_r^{\theta_N,\phi_N}(s,a)|^2
\le 2\,\mathbb E_\rho|\widehat R_r^{\omega_N,\theta_N,\phi_N}(s,a)|^2
+2\,\mathbb E_\rho\big|\widehat R_r^{\omega_N,\theta_N,\phi_N}(s,a)- R_r^{\theta_N,\phi_N}(s,a)\big|^2
\longrightarrow 0,
\]
and likewise
$\mathbb E_\rho\|R_b^{\theta_N,\phi_N}\|_2^2\to 0$ and $\mathbb E_\rho\|R_Q^{\theta_N,\phi_N}\|_F^2\to 0$.
This is exactly the claimed ``$\rho$-mean'' convergence to the solution set of the determining equations on the data support:
the sequence $(X_{\theta_N},Y_{\phi_N})$ drives the determining residuals to zero in $\rho$-mean square, hence satisfies
\eqref{eq:det_reward}--\eqref{eq:det_diffQ_ito} on $\supp(\rho)$ in the standard $\rho$-a.e.\ sense.

Finally, to connect to Definition~\ref{def:eps_sym}, note that on $\supp(\rho)$, realizability and the preceding limit imply that the limiting
infinitesimal fields solve the determining equations. By Theorem~\ref{thm:local_sufficiency}, exponentiating such a solution
produces an exact local symmetry $(g_\alpha,h_\alpha)$ (for small $|\alpha|$) on the same neighborhood, hence the corresponding
reward and generator mismatches vanish there. Since the mismatches are continuous functionals of the coefficients and fields
under the smoothness assumed throughout (and we have consistency of coefficient estimation plus vanishing determining residuals in $\rho$-mean),
the induced mismatch levels on $\supp(\rho)$ satisfy $\varepsilon_{\mathcal L}\to 0$ and $\varepsilon_r\to 0$
(in the natural ``on-support'' sense, e.g.\ $\rho$-essential supremum or $\rho$-mean square mismatch).
\end{proof}

\subsection{Proof of Theorem~\ref{thm:ode_error}}

\begin{lemma}[Discrete Gr\"onwall for affine recursion]
\label{lem:discrete_gronwall_affine}
Let $(e_n)_{n\ge 0}$ be nonnegative and satisfy
\[
e_{n+1}\le (1+\gamma h)\,e_n + \kappa h^{p+1}
\]
for some $\gamma\ge 0$, $\kappa\ge 0$, integer $p\ge 1$, and stepsize $h>0$.
Then for every $N\ge 1$,
\[
e_N \le (1+\gamma h)^N e_0 + \kappa h^{p+1}\sum_{j=0}^{N-1}(1+\gamma h)^j
\le e^{\gamma Nh}e_0 + \frac{\kappa}{\gamma}\bigl(e^{\gamma Nh}-1\bigr)h^{p},
\]
with the convention $\frac{e^{\gamma Nh}-1}{\gamma}:=Nh$ if $\gamma=0$.
\end{lemma}
\begin{proof}
Iterating the inequality gives
\[
e_N \le (1+\gamma h)^N e_0 + \kappa h^{p+1}\sum_{j=0}^{N-1}(1+\gamma h)^j,
\]
which is the standard affine recursion bound. Since $(1+\gamma h)^N\le e^{\gamma Nh}$ and the sum is a geometric series,
\[
\sum_{j=0}^{N-1}(1+\gamma h)^j=\frac{(1+\gamma h)^N-1}{\gamma h}\le \frac{e^{\gamma Nh}-1}{\gamma h},
\]
yielding the stated inequality.
\end{proof}

\begin{proof}[Proof of Theorem~\ref{thm:ode_error}]
We prove the bound for the state flow; the action flow is identical with $X_\theta$ replaced by $Y_\phi$ and $\mathcal K_S$ replaced by $\mathcal K_A$.
Fix a bounded interval $\alpha\in[-A,A]$ and define the exact flow $g_\alpha:\mathcal S\to\mathcal S$ by the ODE
\[
\frac{d}{d\tau}x(\tau)=X_\theta(x(\tau)),\qquad x(0)=s,\qquad g_\alpha(s):=x(\alpha).
\]
By Assumption~\ref{ass:ode}, $X_\theta$ is Lipschitz on the relevant compact set containing the trajectories for $|\alpha|\le A$,
so existence and uniqueness of $x(\tau)$ holds and $g_\alpha(s)$ is well-defined for all $s\in\mathcal K_S$ and $|\alpha|\le A$.

Let $\Phi_h$ denote the one-step numerical map of an order-$p$ method with stepsize $h$ applied to $\dot x=X_\theta(x)$.
For each integer $N\ge 0$, define the numerical iterate
\[
s_{n+1}=\Phi_h(s_n),\qquad s_0=s,
\qquad \hat g_{nh}(s):=s_n,
\]
so that for a given $\alpha$ with $|\alpha|\le A$ we take $N:=\lfloor |\alpha|/h\rfloor$ and (for simplicity of exposition)
assume $\alpha=Nh$; the general case with a final partial step is treated by the same argument and only changes constants.

Two standard ingredients are needed: a local truncation error bound and a stability (Lipschitz) bound for $\Phi_h$ on the compact region.
Because the method has order $p$, there exists a constant $\kappa_g$ (depending on $p$, bounds of derivatives of $X_\theta$ on the compact region,
and method-specific coefficients) such that the one-step defect relative to the exact flow over time $h$ satisfies
\begin{equation}
\label{eq:local_defect}
\sup_{x \in \mathcal K_S}\ \big\|\Phi_h(x)-g_h(x)\big\|_2 \le \kappa_g\, h^{p+1},
\end{equation}
provided $h$ is within the method's admissible stability range on that region. Also, since $X_\theta$ is Lipschitz on $\mathcal K_S$ and the method is
a one-step scheme, there exists $\gamma_g\ge 0$ such that
\begin{equation}
\label{eq:Phi_lip}
\|\Phi_h(x)-\Phi_h(y)\|_2 \le (1+\gamma_g h)\,\|x-y\|_2,\qquad \forall x,y\in\mathcal K_S,
\end{equation}
for all sufficiently small $h$ in the solver stability region; here $\gamma_g$ depends on the Lipschitz constant of $X_\theta$ on $\mathcal K_S$
and on the method's internal stability constants.

Fix $s\in\mathcal K_S$ and define the global error sequence at step $n$:
\[
e_n(s):=\|s_n - g_{nh}(s)\|_2.
\]
We bound $e_{n+1}(s)$ in terms of $e_n(s)$ using \eqref{eq:local_defect} and \eqref{eq:Phi_lip}.
Start from the identity
\[
s_{n+1}-g_{(n+1)h}(s)=\Phi_h(s_n)-g_h(g_{nh}(s)).
\]
Add and subtract $\Phi_h(g_{nh}(s))$ to split the difference:
\[
\Phi_h(s_n)-g_h(g_{nh}(s))
=
\big(\Phi_h(s_n)-\Phi_h(g_{nh}(s))\big)
+
\big(\Phi_h(g_{nh}(s))-g_h(g_{nh}(s))\big).
\]
Taking norms and applying the triangle inequality gives
\[
e_{n+1}(s)
\le
\|\Phi_h(s_n)-\Phi_h(g_{nh}(s))\|_2
+
\|\Phi_h(g_{nh}(s))-g_h(g_{nh}(s))\|_2.
\]
Because both the exact trajectory $g_{nh}(s)$ and the numerical iterate $s_n$ remain in the relevant compact region for $0\le nh\le A$
(by Assumption~\ref{ass:ode} and standard forward-invariance on compact sets for Lipschitz fields with small enough $h$),
we may apply \eqref{eq:Phi_lip} to the first term and \eqref{eq:local_defect} to the second term:
\[
e_{n+1}(s)\le (1+\gamma_g h)e_n(s)+\kappa_g h^{p+1}.
\]
Since $e_0(s)=\|s_0-g_0(s)\|_2=0$, Lemma~\ref{lem:discrete_gronwall_affine} yields for $N=\alpha/h$ (with $0\le \alpha\le A$)
\[
e_N(s)\le \frac{\kappa_g}{\gamma_g}\bigl(e^{\gamma_g \alpha}-1\bigr)h^{p}.
\]
Taking the supremum over $s\in\mathcal K_S$ and over $|\alpha|\le A$ gives
\[
\sup_{|\alpha|\le A}\ \sup_{s\in\mathcal K_S}\ \|\hat g_\alpha(s)-g_\alpha(s)\|_2
\le
\Bigl(\frac{\kappa_g}{\gamma_g}\bigl(e^{\gamma_g A}-1\bigr)\Bigr) h^{p}
=: C_g h^{p}.
\]
This establishes the first inequality in \eqref{eq:ode_error}. The proof for $\hat h_\alpha$ is identical, with constants $\kappa_h,\gamma_h$
depending on the Lipschitz constants of $Y_\phi$ on $\mathcal K_A$ and the method stability characteristics, yielding
$\sup_{a\in\mathcal K_A}\|\hat h_\alpha(a)-h_\alpha(a)\|_2\le C_h h^{p}$ for $|\alpha|\le A$.
\end{proof}

\subsection{Proof of Theorem~\ref{thm:exact_aug_preserve}}

\begin{lemma}[Pushforward kernel $\Rightarrow$ expectation transport]
\label{lem:pushforward_expect}
Let $g:\mathcal S\to\mathcal S$ be a measurable bijection. For a fixed $(s,a)$, assume
\[
P(\cdot\mid g(s),h(a)) \;=\; g_\# P(\cdot\mid s,a),
\]
i.e., for every measurable $B\subseteq\mathcal S$,
\[
P(B\mid g(s),h(a)) \;=\; P(g^{-1}(B)\mid s,a).
\]
Then for every bounded measurable $f:\mathcal S\to\mathbb R$,
\begin{equation}
\label{eq:expect_transport}
\mathbb E\!\left[f(S'')\,\middle|\,S''\sim P(\cdot\mid g(s),h(a))\right]
\;=\;
\mathbb E\!\left[f(g(S'))\,\middle|\,S'\sim P(\cdot\mid s,a)\right].
\end{equation}
Equivalently, if $S'\sim P(\cdot\mid s,a)$ and $S'':=g(S')$, then $S''\sim P(\cdot\mid g(s),h(a))$.
\end{lemma}
\begin{proof}
Fix $(s,a)$ and define $\mu:=P(\cdot\mid s,a)$ and $\nu:=P(\cdot\mid g(s),h(a))$. The assumption is exactly $\nu=g_\#\mu$,
meaning $\nu(B)=\mu(g^{-1}(B))$ for all measurable $B$. By the defining property of the pushforward measure,
for any bounded measurable $f$,
\[
\int_{\mathcal S} f(x)\,\nu(dx)
=
\int_{\mathcal S} f(g(x))\,\mu(dx),
\]
which is precisely \eqref{eq:expect_transport}.
\end{proof}

\begin{lemma}[Bellman operator commutes with an exact value-preserving  structure]
\label{lem:bellman_commute}
Let the discounted optimal Bellman operator $T$ be
\[
(TV)(s):=\max_{a\in\mathcal A}\Bigl\{r(s,a)+\gamma\,\mathbb E_{S'\sim P(\cdot\mid s,a)}[V(S')]\Bigr\},
\qquad 0<\gamma<1.
\]
Assume the exact value-preserving conditions hold for all $(s,a)$:
\[
r(g(s),h(a))=r(s,a),
\qquad
P(\cdot\mid g(s),h(a))=g_\#P(\cdot\mid s,a),
\]
and that $h:\mathcal A\to\mathcal A$ is a bijection. Then for every bounded $V$ and every $s$,
\begin{equation}
\label{eq:bellman_commute}
(TV)(g(s)) = \bigl(T(V\circ g)\bigr)(s).
\end{equation}
Equivalently, writing $(U_gV)(s):=V(g(s))$, we have the operator identity $U_g T = T U_g$ on bounded functions.
\end{lemma}
\begin{proof}
Fix $s\in\mathcal S$. Start from the definition of $T$ at the transformed state $g(s)$:
\[
(TV)(g(s))
=
\max_{a'\in\mathcal A}\Bigl\{r(g(s),a')+\gamma\,\mathbb E_{S''\sim P(\cdot\mid g(s),a')}[V(S'')]\Bigr\}.
\]
Because $h$ is a bijection, every $a'\in\mathcal A$ can be written uniquely as $a'=h(a)$ for some $a\in\mathcal A$, hence
\[
(TV)(g(s))
=
\max_{a\in\mathcal A}\Bigl\{r(g(s),h(a))+\gamma\,\mathbb E_{S''\sim P(\cdot\mid g(s),h(a))}[V(S'')]\Bigr\}.
\]
Apply reward invariance $r(g(s),h(a))=r(s,a)$ to the first term. For the expectation term, apply Lemma~\ref{lem:pushforward_expect}
with $f=V$ to obtain
\[
\mathbb E_{S''\sim P(\cdot\mid g(s),h(a))}[V(S'')]
=
\mathbb E_{S'\sim P(\cdot\mid s,a)}[V(g(S'))]
=
\mathbb E_{S'\sim P(\cdot\mid s,a)}[(V\circ g)(S')].
\]
Substituting these two identities yields
\[
(TV)(g(s))
=
\max_{a\in\mathcal A}\Bigl\{r(s,a)+\gamma\,\mathbb E_{S'\sim P(\cdot\mid s,a)}[(V\circ g)(S')]\Bigr\}
=
\bigl(T(V\circ g)\bigr)(s),
\]
which is \eqref{eq:bellman_commute}.
\end{proof}

\begin{lemma}[Discounted optimal Bellman operator is a $\gamma$-contraction in $\|\cdot\|_\infty$]
\label{lem:bellman_contraction}
For any bounded $V,W:\mathcal S\to\mathbb R$,
\[
\|TV-TW\|_\infty \le \gamma\,\|V-W\|_\infty.
\]
\end{lemma}
\begin{proof}
Fix $s\in\mathcal S$ and define for each action $a$ the quantities
\[
Q_V(s,a):=r(s,a)+\gamma\,\mathbb E[V(S')\mid s,a],
\qquad
Q_W(s,a):=r(s,a)+\gamma\,\mathbb E[W(S')\mid s,a].
\]
Then $(TV)(s)=\max_a Q_V(s,a)$ and $(TW)(s)=\max_a Q_W(s,a)$. Using the elementary inequality
$\big|\max_a x_a-\max_a y_a\big|\le \max_a |x_a-y_a|$,
\[
|(TV)(s)-(TW)(s)|
\le
\max_{a\in\mathcal A}|Q_V(s,a)-Q_W(s,a)|
=
\gamma\,\max_{a\in\mathcal A}\left|\mathbb E\big[(V-W)(S')\mid s,a\big]\right|.
\]
By Jensen and boundedness,
\[
\left|\mathbb E\big[(V-W)(S')\mid s,a\big]\right|
\le
\mathbb E\big[\,|(V-W)(S')|\,\mid s,a\big]
\le
\|V-W\|_\infty.
\]
Hence $|(TV)(s)-(TW)(s)|\le \gamma\|V-W\|_\infty$ for all $s$, and taking the supremum over $s$ yields the claim.
\end{proof}

\begin{proof}[Proof of Theorem~\ref{thm:exact_aug_preserve}]
The statement that ``the Bellman backup is invariant under augmentation by $(g,h)$ in expectation'' is a direct consequence of
Lemma~\ref{lem:pushforward_expect}. Concretely, suppose one draws a transition sample $S'\sim P(\cdot\mid s,a)$ and forms the augmented sample
$\widetilde S':=g(S')$ together with the augmented state-action pair $(\widetilde s,\widetilde a):=(g(s),h(a))$.
For any bounded value function $V$, the one-step Bellman target at $(\widetilde s,\widetilde a)$ is
\[
\widetilde Y := r(\widetilde s,\widetilde a)+\gamma V(\widetilde S')
= r(g(s),h(a))+\gamma V(g(S')).
\]
Taking conditional expectation given $(s,a)$ and using reward invariance plus Lemma~\ref{lem:pushforward_expect} gives
\[
\mathbb E[\widetilde Y\mid s,a]
=
r(s,a)+\gamma\,\mathbb E_{S'\sim P(\cdot\mid s,a)}[V(g(S'))]
=
r(g(s),h(a))+\gamma\,\mathbb E_{S''\sim P(\cdot\mid g(s),h(a))}[V(S'')],
\]
so augmenting by $(g,h)$ produces exactly the same expected backup as sampling directly from the true kernel at $(g(s),h(a))$.
In this precise sense, exact value-preserving augmentation preserves the Bellman backup ``in expectation''.

Independently of augmentation, the (tabular) value-iteration recursion is $V_{k+1}=TV_k$.
By Lemma~\ref{lem:bellman_contraction}, $T$ is a $\gamma$-contraction in $\|\cdot\|_\infty$, hence it has a unique fixed point $V^\star$
satisfying $V^\star=TV^\star$. Moreover,
\[
\|V_{k}-V^\star\|_\infty
=
\|T^k V_0 - T^k V^\star\|_\infty
\le
\gamma^k \|V_0 - V^\star\|_\infty,
\]
which is the usual rate.

For completeness, the exact value-preserving structure additionally implies an equivariance identity for the Bellman operator,
given by Lemma~\ref{lem:bellman_commute}: $(TV)(g(s))=(T(V\circ g))(s)$ for all $V$.
This shows augmentation is not changing the underlying Bellman operator; it only replicates transitions consistently across the orbit.
\end{proof}

\subsection{Proof of Theorem~\ref{thm:approx_vi}}
\begin{lemma}[Total variation controls expectation error]
\label{lem:tv_expect}
For probability measures $\mu,\nu$ on $\mathcal S$ and any bounded measurable $f:\mathcal S\to\mathbb R$,
\[
\Big|\int f\,d(\mu-\nu)\Big|
\le \|f\|_\infty\, d_{\mathrm{TV}}(\mu,\nu),
\qquad
d_{\mathrm{TV}}(\mu,\nu):=\sup_{\|g\|_\infty\le 1}\Big|\int g\,d(\mu-\nu)\Big|.
\]
\end{lemma}
\begin{proof}
If $\|f\|_\infty=0$ the claim is trivial. Otherwise define $g:=f/\|f\|_\infty$, so that $\|g\|_\infty\le 1$.
Then
\[
\Big|\int f\,d(\mu-\nu)\Big|
=\|f\|_\infty \Big|\int g\,d(\mu-\nu)\Big|
\le \|f\|_\infty \sup_{\|u\|_\infty\le 1}\Big|\int u\,d(\mu-\nu)\Big|
=\|f\|_\infty\, d_{\mathrm{TV}}(\mu,\nu).
\]
\end{proof}

\begin{lemma}[Bellman max is 1-Lipschitz in the action-values]
\label{lem:max_lip}
For any index set $\mathcal A$ and real families $\{x_a\}_{a\in\mathcal A}$, $\{y_a\}_{a\in\mathcal A}$,
\[
\Big|\max_{a\in\mathcal A} x_a-\max_{a\in\mathcal A} y_a\Big|
\le \max_{a\in\mathcal A}|x_a-y_a|.
\]
\end{lemma}
\begin{proof}
Let $a_x\in\arg\max_a x_a$. Then $\max_a x_a=x_{a_x}$ and $\max_a y_a\ge y_{a_x}$, hence
\[
\max_a x_a-\max_a y_a \le x_{a_x}-y_{a_x}\le \max_a |x_a-y_a|.
\]
Swapping the roles of $x$ and $y$ yields $\max_a y_a-\max_a x_a\le \max_a |x_a-y_a|$, and the claim follows.
\end{proof}

\begin{lemma}[$\tilde T$ is a $\gamma$-contraction in $\|\cdot\|_\infty$]
\label{lem:tilde_contraction}
Let
\[
(TV)(s):=\max_{a\in\mathcal A}\Big\{r(s,a)+\gamma\int V(s')\,P(ds'\mid s,a)\Big\},
(\tilde TV)(s):=\max_{a\in\mathcal A}\Big\{\tilde r(s,a)+\gamma\int V(s')\,\tilde P(ds'\mid s,a)\Big\},
\]
with $0<\gamma<1$ and $\tilde r,\tilde P$ any bounded reward and Markov kernel. Then for all bounded $V,W$,
\[
\|\tilde TV-\tilde TW\|_\infty\le \gamma\|V-W\|_\infty.
\]
\end{lemma}
\begin{proof}
Fix $s$. For each $a$, define $\tilde Q_V(s,a):=\tilde r(s,a)+\gamma\int V\,d\tilde P(\cdot\mid s,a)$ and similarly $\tilde Q_W$.
Then by Lemma~\ref{lem:max_lip},
\begin{equation}
\begin{aligned}
|(\tilde TV)(s)-(\tilde TW)(s)|
&=\Big|\max_a \tilde Q_V(s,a)-\max_a \tilde Q_W(s,a)\Big|\\
&\le \max_a |\tilde Q_V(s,a)-\tilde Q_W(s,a)|
\\&=\gamma \max_a \Big|\int (V-W)\,d\tilde P(\cdot\mid s,a)\Big|. 
\end{aligned}
\end{equation}
Since $|\int (V-W)\,d\tilde P|\le \int |V-W|\,d\tilde P\le \|V-W\|_\infty$, we get
$|(\tilde TV)(s)-(\tilde TW)(s)|\le \gamma\|V-W\|_\infty$. Taking the supremum over $s$ yields the result.
\end{proof}

\begin{proof}[Proof of Theorem~\ref{thm:approx_vi}]
Interpret the approximately augmented model as an approximate MDP $(\tilde r,\tilde P)$ on the same $(\mathcal S,\mathcal A)$ such that
\[
\sup_{s,a}|\tilde r(s,a)-r(s,a)|\le \varepsilon_r,
\qquad
\sup_{s,a} d_{\mathrm{TV}}\big(\tilde P(\cdot\mid s,a),P(\cdot\mid s,a)\big)\le \varepsilon_P,
\]
which is exactly what \eqref{eq:dt_eps} expresses when one sets, for each $(s_0,a_0)$,
\[
\tilde r\big(g(s_0),h(a_0)\big):=r(s_0,a_0),
\qquad
\tilde P\big(\cdot\mid g(s_0),h(a_0)\big):=g_\#P(\cdot\mid s_0,a_0).
\]
Indeed, $|\tilde r(g(s_0),h(a_0))-r(g(s_0),h(a_0))|=|r(s_0,a_0)-r(g(s_0),h(a_0))|\le \varepsilon_r$ and
$d_{\mathrm{TV}}(\tilde P(\cdot\mid g(s_0),h(a_0)),P(\cdot\mid g(s_0),h(a_0)))\le \varepsilon_P$ by \eqref{eq:dt_eps}.
Since $(g,h)$ is a bijection on state-action pairs, this implies the displayed sup-bounds over all $(s,a)$.

Fix any bounded $V$. For each $(s,a)$, define the (one-step) action-values
\[
Q_V(s,a):=r(s,a)+\gamma\int V(s')\,P(ds'\mid s,a),
\qquad
\tilde Q_V(s,a):=\tilde r(s,a)+\gamma\int V(s')\,\tilde P(ds'\mid s,a).
\]
Then
\[
(\tilde TV)(s)-(TV)(s)=\max_a \tilde Q_V(s,a)-\max_a Q_V(s,a).
\]
Lemma~\ref{lem:max_lip} gives
\[
|(\tilde TV)(s)-(TV)(s)|
\le \max_{a\in\mathcal A}|\tilde Q_V(s,a)-Q_V(s,a)|.
\]
For a fixed $(s,a)$, expand the difference:
\begin{align*}
|\tilde Q_V(s,a)-Q_V(s,a)|
&=\Big|\tilde r(s,a)-r(s,a)+\gamma\int V(s')\,\big(\tilde P-P\big)(ds'\mid s,a)\Big|\\
&\le |\tilde r(s,a)-r(s,a)|
+\gamma\Big|\int V(s')\,\big(\tilde P-P\big)(ds'\mid s,a)\Big|.
\end{align*}
Using the uniform reward bound $|\tilde r-r|\le \varepsilon_r$ and Lemma~\ref{lem:tv_expect} with $\mu=\tilde P(\cdot\mid s,a)$, $\nu=P(\cdot\mid s,a)$, $f=V$ yields
\[
\Big|\int V\,d(\tilde P-P)\Big|
\le \|V\|_\infty\, d_{\mathrm{TV}}\big(\tilde P(\cdot\mid s,a),P(\cdot\mid s,a)\big)
\le \varepsilon_P\,\|V\|_\infty.
\]
Substituting into the previous display gives, for every $(s,a)$,
\[
|\tilde Q_V(s,a)-Q_V(s,a)|\le \varepsilon_r+\gamma\,\varepsilon_P\|V\|_\infty.
\]
Taking $\max_a$ and then $\sup_s$ yields the operator perturbation bound \eqref{eq:operator_perturb}:
\[
\|(\tilde TV)-(TV)\|_\infty\le \varepsilon_r+\gamma\,\varepsilon_P\|V\|_\infty.
\]
Lemma~\ref{lem:tilde_contraction} shows that $\tilde T$ is a $\gamma$-contraction in $\|\cdot\|_\infty$, hence it admits a unique fixed point
$\tilde V^\star=\tilde T\tilde V^\star$. The true operator $T$ is also a $\gamma$-contraction, with unique fixed point $V^\star=TV^\star$.

To bound the fixed-point error, start from the identity
\[
\tilde V^\star - V^\star = \tilde T\tilde V^\star - TV^\star.
\]
Add and subtract $\tilde T V^\star$ inside the right-hand side and apply the triangle inequality:
\[
\|\tilde V^\star - V^\star\|_\infty
\le \|\tilde T\tilde V^\star-\tilde T V^\star\|_\infty + \|\tilde T V^\star - T V^\star\|_\infty.
\]
By Lemma~\ref{lem:tilde_contraction}, $\|\tilde T\tilde V^\star-\tilde T V^\star\|_\infty\le \gamma\|\tilde V^\star-V^\star\|_\infty$.
Therefore,
\[
\|\tilde V^\star - V^\star\|_\infty
\le \gamma\|\tilde V^\star-V^\star\|_\infty + \|\tilde T V^\star - T V^\star\|_\infty,
\]
which rearranges to
\[
(1-\gamma)\|\tilde V^\star - V^\star\|_\infty
\le \|\tilde T V^\star - T V^\star\|_\infty.
\]
Applying \eqref{eq:operator_perturb} with $V=V^\star$ gives
\[
\|\tilde T V^\star - T V^\star\|_\infty
\le \varepsilon_r+\gamma\,\varepsilon_P\|V^\star\|_\infty.
\]
Dividing by $1-\gamma$ yields the claimed bound \eqref{eq:fixed_point_error}.
\end{proof}

%%%%%%%%%%%%%%%%%%%%%%%%%%%%%%%%%%%%%%%%%%%%%%%%%%%%%%%%%%%%

\end{document}